\documentclass{article}
\usepackage{iclr2027_conference,times}

\usepackage{amsmath,amsfonts,bm}
\usepackage{amssymb}

\def\1{\bm{1}}

\DeclareMathAlphabet{\mathsfit}{\encodingdefault}{\sfdefault}{m}{sl}
\SetMathAlphabet{\mathsfit}{bold}{\encodingdefault}{\sfdefault}{bx}{n}

\usepackage{booktabs}
\usepackage{hyperref}
\usepackage{mathrsfs}
\usepackage{url}
\usepackage{pgfplots}
\usepgfplotslibrary{groupplots}
\usetikzlibrary{decorations.markings}
\pgfplotsset{compat=1.18}

\usetikzlibrary{intersections}
\usepgfplotslibrary{groupplots,fillbetween}

\usepackage{xcolor} 

\usepackage{amsthm}
\newtheorem{theorem}{Theorem}[section]
\newtheorem{proposition}[theorem]{Proposition}

\title{Verifier Errors in RLVR: Reward Hacking, Limits of Feedback, and Selective Control}

\author{Christian Moya$^1$\thanks{Corresponding author: \texttt{cmoyacal@purdue.edu}.}, Elliott Thornley$^2$, and Guang Lin$^1$  \\
$^1$ Purdue University, $^2$ National University of Singapore
}

\iclrfinalcopy 

\begin{document}
\maketitle
\lhead{Preprint. Under review.}  

\begin{abstract}
In reinforcement learning with verifiable rewards (RLVR), imperfect verifiers can reward incorrect responses, creating opportunities for reward hacking. Using gradient flow with a fixed verifier, we characterize the conditions under which reward rises while correctness falls. We then show that the observations available during RLVR are, in general, insufficient to detect or identify accepted errors, or to guarantee their reduction without sacrificing correct responses. To address this limit, we construct a correction using additional feedback about correctness from audits. This correction achieves \emph{selective control}: at the current policy, it lowers the probability of accepted errors and raises that of correct responses, provided it outweighs the pressure toward errors from verifier reward. Experiments with log linear and neural contextual bandits and with a language model support the analysis and show that selective control under partial auditing reduces accepted errors while increasing correctness.
\end{abstract}

\section{Introduction} \label{sec:introduction}
Reinforcement learning with verifiable rewards (RLVR)~\citep{lambert2025tulu} fine-tunes pretrained language models using rewards from automated checks. These checks come from a verifier, which compares a model's response with a known solution or runs the code in the response against tests. Rewarding responses that pass the verifier has improved reasoning performance on mathematics and coding tasks~\citep{DBLP:journals/corr/abs-2501-12948}. These gains, however, depend on the verifier: higher reward reflects progress on the intended task only when the verifier reliably judges correctness.

A verifier can accept a response even when the response fails the intended task. For example, the code in a response may pass the available tests yet fail on inputs those tests omit. Rewarding such responses reinforces accepted errors: responses that satisfy the verifier but fail the intended task. As the model produces accepted errors more often, average reward can rise even while performance on the intended task declines, the signature of reward hacking~\citep{NEURIPS2022_3d719fee}.

Empirical studies document reward hacking in RLVR. Reward hacking appears, for example, in inductive reasoning, where a model must infer a general rule from labeled examples. Models trained with RLVR often skip this rule and instead list the label of each example. The verifier accepts these responses because it checks only whether a response labels the given examples correctly~\citep{helff2026llms}. Beyond eliciting such new hacking behaviors, reinforcement learning can amplify those that models acquired during earlier fine-tuning~\citep{DBLP:journals/corr/abs-2603-07084}. These findings motivate methods that reduce reward hacking while preserving progress on the intended task.

Existing methods address reward hacking by constraining optimization~\citep{laidlaw2025correlated}, improving feedback~\citep{coste2024reward,lightman2024lets}, or detecting and correcting hacking~\citep{baker2025monitoring,wang2026detecting}. Yet because these methods rely on assumptions about correctness~\citep{everitt2017reinforcement}, they can leave verifier errors unresolved~\citep{eisenstein2024helping}. Training therefore continues with these verifier errors, rewarding correct responses and accepted errors alike. This raises a broader question: how do these verifier errors shape what RLVR learns about the intended task as training proceeds?

To answer this question, we hold an imperfect verifier fixed and study how the policy learns from its rewards. We model this learning as gradient flow because it makes the analysis tractable. Within this setting, we study when reward hacking grows, what the information available during RLVR training reveals about it, and when an intervention reduces hacking while preserving correctness. Our contributions are:

\textit{(i)} We derive a condition under which the share of accepted errors among accepted responses grows at the current policy, and we express it through two mechanisms: hack bias and correctness-to-hack leakage (Proposition~\ref{prop:hacking_growth}). We also characterize when the gradient flow increases verifier reward while reducing correctness, the signature of reward hacking (Proposition~\ref{prop:reward_hacking_flow}).

\textit{(ii)} We establish limits on detection, identification, and selective control from verifier feedback alone. We show that the complete RLVR training record provides no advantage in detecting accepted errors and can leave correctness unidentifiable (Propositions~\ref{prop:verifier_information_limit} and~\ref{prop:verifier_identification_limit}). We further prove that no controller using this record alone can guarantee fewer accepted errors while preserving correct responses (Proposition~\ref{prop:verifier_control_limit}).

\textit{(iii)} We derive conditions under which a correction to RLVR training reduces reward hacking while improving correctness, and we design projected audit correction, which satisfies them using correctness labels from audits of some responses (Theorem~\ref{thm:selective_control}).

We test our theory in settings of increasing complexity. In contextual bandits, reward hacking emerges and projected audit correction reduces hacks while increasing correctness. In a language model, verifier reward rises while correctness falls, and the same correction reverses this trend.

\section{Problem Formulation} \label{sec:problem_formulation}
\vspace{-.5em}
\subsection{Reinforcement Learning with Verifier Rewards~(RLVR)} \label{subsec:rlvr}

\textbf{Policy.}  Given a prompt $x$ drawn from a prompt distribution $\mathcal{D}$, the language policy $\pi_\theta(\cdot \mid x)$ samples a response $y$ from the response set $\mathcal{Y}_x$. This policy has a trainable parameter vector $\theta  \in \mathbb{R}^d$.

\textbf{RLVR.} We analyze RLVR~\citep{lambert2025tulu}, where the verifier assigns a reward $R(x,y) \in \{0,1\}$. A reward of~$1$ indicates acceptance. RLVR training maximizes the scalar objective: $J_R(\theta) = \mathbb{E}_{x \sim \mathcal{D},\; y \sim \pi_\theta(\cdot \mid x)}[R(x,y)].$ This objective represents the probability of acceptance averaged over prompts and sampled responses, yielding a value in $[0,1]$.
\vspace{-.5em}
\subsection{Correctness and Accepted Errors} \label{subsec:correctness_and_errors}
We define a fixed correctness indicator~$c(x,y) \in \{0,1\}$ for each pair~$(x,y)$, where $c=1$ denotes a correct response. To focus our analysis on the consequences of accepting incorrect responses, we assume the verifier produces no false negatives. This assumption implies that $c(x,y) \le R(x,y)$ for all pairs, meaning the verifier accepts all correct responses. We call the event of an incorrect response being accepted, i.e., $\{R(x,y)=1,c(x,y)=0\}$, a \textit{hack}.

For any given prompt~$x$,  we can partition the set of all accepted responses, $A_x = \{y \in \mathcal{Y}_x : R(x,y) = 1\}$. This set consists of two disjoint subsets. The first is the set of correct responses, $G_x = \{y \in \mathcal{Y}_x : c(x,y) = 1\}$. The second is the set of hacks, $H_x = A_x \setminus G_x$. These sets remain fixed throughout training. The policy, however, can learn to sample from them with different frequencies.

\textbf{Acceptance and hacking.} For any given prompt $x$, we define two key probabilities. The first is the acceptance probability, $p_x(\theta) = \Pr_{\pi_\theta}(Y \in A_x \mid x)$. The second probability, defined only when $p_x(\theta) > 0,$ is the hacked share, $q_x(\theta) = \Pr_{\pi_\theta}(Y \in H_x \mid Y \in A_x, x).$ Both probabilities depend on the policy parameters~$\theta$ and change throughout RLVR training. Using the acceptance probability $p(\theta)$, we can now rewrite the RLVR objective as the following expectation over prompts: $J_R(\theta) = \mathbb{E}_{x \sim \mathcal{D}}[p_x(\theta)]$. Because this objective depends only on~$p_x(\theta)$, it makes no distinction between correct responses and hacks (see Appendix~\ref{app:acceptance_and_hacking}).
\vspace{-.5em}
\subsection{Gradient Flow Dynamics} \label{subsec:gradient_flow_dynamics}
To analyze how the policy evolves, we use gradient ascent in continuous time. This model isolates the effects of the verifier signal by removing the noise inherent in stochastic optimization. When the objective~$J_R$ is differentiable, the parameters evolve according to the gradient flow:
\begin{equation}
\label{eq:verifier_flow}
\dot\theta(t) = g_R(\theta(t)) := \nabla_\theta J_R(\theta(t)),
\end{equation}
where $g_R(\theta) \in \mathbb{R}^d$ is the exact gradient. A direct consequence is that the objective can only improve, as $\dot J_R(\theta(t)) = \|g_R(\theta(t))\|^2 \geq 0$ (see Appendix~\ref{app:rlvr_dynamics}). This guarantee, however, applies only to the overall acceptance probability. It reveals nothing about the hacked share, $q_x(\theta)$ which may increase, decrease, or remain unchanged.

\textbf{Research questions.} We now define our two central research questions within our gradient flow framework. First, under what conditions does the hacked share, $q_x(\theta)$, grow or persist during training? Answering this question requires distinguishing increases in acceptance due to correct responses from shifts in the policy's sampling toward hacks. Second, since the verifier's feedback is blind to correctness, is it possible to reduce the hacked share using verifier's information alone? If not, what additional information would allow it to do so?

\section{When Hacking Grows: A Population Analysis} \label{sec:hacking}
This section analyzes the population dynamics of hacked responses. We show that the accepted population and the hacked share can increase together, and characterize when verifier flow increases reward while reducing correctness. We then describe how hack bias and correctness-to-hack leakage drive hacking growth. Ultimately, these drivers create a latent vulnerability that on-policy training can reinforce.
\vspace{-.5em}
\subsection{Verifier Acceptance and Hacked Share Can Increase Together} \label{subsec:acceptance_hack_together}
We partition responses into three populations: correct responses $G=\{(x,y):y\in G_x\}$, hacks $H=\{(x,y):y\in H_x\}$, and rejected responses $N=\{(x,y):y\in\mathcal{Y}_x,\;R(x,y)=0\}$. Let $p_G(\theta):=\Pr_\theta(G)\in[0,1]$ and $p_H(\theta):=\Pr_\theta(H)\in[0,1]$ denote the probabilities of correct responses and hacks, respectively. To track both how often the verifier accepts and how often accepted responses are hacks, we define $p(\theta):=\Pr_\theta(G\cup H)=J_R(\theta)$ and  $q(\theta):=\Pr_\theta(H\mid G\cup H),$ where $p(\theta)\in[0,1]$ measures overall acceptance and $q(\theta)\in[0,1]$ measures the hacked share among accepted pairs. The latter requires $p(\theta)>0$.

Expected reward depends on total acceptance $p=p_G+p_H$, regardless of how it splits between correct responses and hacks. Writing $p_H(\theta) = p(\theta)q(\theta)$ and $p_G(\theta) = p(\theta)(1-q(\theta))$ and differentiating with respect to time gives
$$
\dot p_H(\theta) = q(\theta)\dot p(\theta) + p(\theta)\dot q(\theta), \qquad \dot p_G(\theta) = (1-q(\theta))\dot p(\theta) - p(\theta)\dot q(\theta),
$$
where the $\dot p$ terms reflect changes in total acceptance and the $\dot q$ terms reflect shifts between correct responses and hacks. These identities impose no trade-off between acceptance and hacked share; both can increase together.

\textbf{Reward hacking.} \citet{NEURIPS2022_3d719fee} call a proxy reward hackable if some pair of policies has higher expected proxy return but lower expected true return. In our setting, the verifier reward $J_R=p$ serves as the proxy, and correctness $J_C=p_G$ serves as the true return. We characterize when gradient flow on $J_R$, which we call verifier flow, generates such a pair.

\begin{proposition}[Reward hacking along the flow]
\label{prop:reward_hacking_flow}
Suppose $p_G$ and $p_H$ are continuously differentiable and $p>0$ along verifier flow on $[0,T]$. Then there exist $t_1<t_2$ with $J_R(\theta(t_2))>J_R(\theta(t_1))$ and $J_C(\theta(t_2))<J_C(\theta(t_1))$ if and only if $p(\theta(t))\dot q(\theta(t))>(1-q(\theta(t)))\dot p(\theta(t))$ at some $t\in(0,T)$.
\end{proposition}
\vspace{-.5em}
\begin{proof}
See Appendix~\ref{app:proof_reward_hacking_flow}.
\end{proof}
\vspace{-.5em}
\textbf{Interpretation.} The two sides of the inequality compete. The term $p\dot q$ is the correctness lost as the accepted population shifts toward hacked responses, and $(1-q)\dot p$ is the correctness gained from rising acceptance. Training produces reward hacking once the loss outweighs the gain. The condition needs to hold at only a single time to produce a pair of policies that witnesses hackability. Because both terms depend on the policy, the condition can hold at one time and fail at another.
\vspace{-.5em}
\subsection{The Dynamics of the Hacked Share} \label{subsec:dynamics_hacks}
While the hacked share $q$ remains the quantity of interest, log-odds coordinates simplify its dynamics. When both $p_H(\theta)>0$ and $p_G(\theta)>0$, we can define the log odds $z(\theta)\in\mathbb{R}$ as
$$
z(\theta) := \log\frac{p_H(\theta)}{p_G(\theta)} = \log\frac{q(\theta)}{1-q(\theta)},
$$
Because $z$ is strictly increasing in $q$, the sign of~$\dot z$ determines whether the hacked share grows, shrinks, or remains constant.

\textbf{The log odds dynamics.} Since $z=\log p_H - \log p_G$, the rate~$\dot z$ depends on the gradient of each group's log-probability. For each group $S \in \{G,H,N\}$ with positive, continuously differentiable probability, we denote this gradient by $\bar s_S(\theta):=\nabla_\theta\log\Pr_\theta(S)\in\mathbb{R}^d,$ the group's average score. Along the verifier flow~\eqref{eq:verifier_flow}, differentiating~$z$ yields
\begin{equation}
\label{eq:log_odds_dynamics}
\dot z(\theta) = \bigl(\bar s_H(\theta) - \bar s_G(\theta) \bigr)^\top g_R.
\end{equation}
The log odds grow when the reward gradient~$g_R$ aligns with the score difference between hacks and correct responses. Differentiating total acceptance $p=p_G+p_H=1-p_N$ shows that $g_R$ is a weighted combination of the three group scores: $g_R=p\bigl[(1-q)\bar s_G+q\bar s_H\bigr]=-(1-p)\bar s_N.$ This decomposition, combined with~\eqref{eq:log_odds_dynamics}, resolves $\dot z$ into two mechanisms:
\begin{equation}  \label{eq:hack_mechanisms}
\dot z = p(1-p)[
\underbrace{(\bar s_H-\bar s_G)^\top(\bar s_G-\bar s_N)}_{\text{correctness-to-hack leakage}}
+
\underbrace{q\|\bar s_H-\bar s_G\|^2}_{\text{hack bias}}].
\end{equation}
The full derivation is provided in Appendix~\ref{app:proof_of_mechanism}.
\vspace{-.5em}
\subsection{Main Result: The Mechanism of Hacking Growth} \label{subsec:hack_mechanism}
We now state our main result characterizing the mechanism of hacking growth.

\begin{proposition}[Population growth of the hacked share]
\label{prop:hacking_growth}
Along verifier gradient flow, with $p_G,p_H,p_N>0$, the population hacked share grows if and only if
$$
(\bar s_H - \bar s_G)^\top(\bar s_G - \bar s_N)
+
q\|\bar s_H - \bar s_G\|^2
>0.
$$
\end{proposition}
When the inequality holds, $\dot z > 0$. Since $\dot q = q(1-q)\,\dot z$ and $\dot p = \|g_R\|^2 \geq 0$, both the hacked share and total acceptance grow strictly, while the correct share among accepted responses decreases.

The two mechanisms of~\eqref{eq:hack_mechanisms} play different roles:

\emph{(i) Hack bias}, $q\|\bar s_H-\bar s_G\|^2$, arises because hacks contribute to the reward gradient in proportion to their share among accepted responses. It is always nonnegative and grows with the hacked share at fixed group scores.

\emph{(ii) Correctness-to-hack leakage}, $(\bar s_H-\bar s_G)^\top(\bar s_G-\bar s_N)$, arises when the direction that favors correct responses over rejected ones also favors hacks over correct responses. It can be positive or negative: positive leakage reinforces hack bias, negative leakage opposes it and can reverse growth when its magnitude exceeds the bias.

\textbf{Interpretation.} This result is local: the group scores and hacked share depend on the current parameters, so the growth condition can change during training. The group scores evolve, so continued growth is not guaranteed. Yet a feedback loop is present: a growing hacked share strengthens hack bias at fixed group scores, which favors further growth. This result exposes a latent vulnerability: hacking can reinforce the conditions that favor its own growth, because the policy generates its own training samples. Moreover, RLVR training is not designed to oppose this reinforcement. We must therefore build any hacking defense on top of RLVR. Can the information generated during RLVR training support such a defense?

\section{What RLVR Training Reveals: Limits of Verifier Feedback} \label{sec:verifier_feedback}
Section~\ref{sec:hacking} showed that rising rewards can mask a growing share of hacks. We now study whether the observations collected during RLVR training can expose them. We show that they cannot: training observations confer no detection advantage, leaving correct responses unidentifiable, and making selective control impossible.
\vspace{-.5em}
\subsection{RLVR Training Observations, Monitors, and Compatible Correctness} \label{subsec:RLVR_observations}
\textbf{RLVR training observations.} RLVR training produces prompts, sampled responses, and verifier labels. Let $L_t$ collect these observations through time $t$, along with policy parameters, probabilities, gradients, and other derived quantities. This record is deliberately generous: the limitations below do not arise from incomplete logging. We assume that $L_t$ contains no additional correctness feedback beyond the verifier. We write $\mathcal{F}_t^R := \sigma(L_t)$ for the information contained in the record.

\textbf{Monitors.} A monitor is an algorithm that uses $L_t$ to assess hacking. A detection monitor raises a binary alarm about whether hacks exist; an identification monitor predicts the correctness label $\widehat{c}_t(x,y)$ of an accepted response. We consider deterministic monitors here and defer randomized extensions to Appendix~\ref{app:verifier_feedback}.

\textbf{Compatible correctness assignments.} A fixed assignment $c$ determines correctness. To study what the training record reveals about $c$, we consider alternative assignments consistent with the verifier. Assuming no false negatives, we define 
$$
\mathcal{C}_R := \{c' : c'(x,y) \in \{0,1\},c'(x,y) \le R(x,y) \text{ for all } (x,y)\}.
$$
 This class models uncertainty about the true correctness assignment; $c$ itself remains fixed during training. All members agree on rejected responses, but they may disagree on accepted ones. To isolate the effect of this uncertainty, we compare these alternative assignments while holding the prompt distribution, verifier, initialization, and training algorithm fixed.
 \vspace{-.5em}
\subsection{Limits of Detecting Hacking} \label{subsec:limits_of_detection}
Verifier acceptance can increase while the hacked share grows. Can the training record reveal even the presence of hacks? Under $c = R$, every accepted response is correct. A compatible alternative $c_1 \in \mathcal{C}_R$ with $\Pr_\theta(R=1, c_1=0) > 0$ admits hacks. \emph{Detection} requires distinguishing these two cases using only~$L_t$.
\begin{proposition}[Limits of detection]
\label{prop:verifier_information_limit}
Under the comparison setup above, fix $\theta$, $t$, and $c_1 \in \mathcal{C}_R$ with $\Pr_\theta(R=1, c_1=0) > 0$. The training record has the same distribution under both assignments: $\operatorname{Law}_{c=R}(L_t) = \operatorname{Law}_{c=c_1}(L_t).$ Thus, no monitor can distinguish $c = c_1$, which admits hacks, from $c = R$, which has none.
\end{proposition}
\vspace{-.5em}
\begin{proof}
See Appendix~\ref{app:verifier_only_detection}.
\end{proof}
\vspace{-.5em}
\textbf{Interpretation.} Because alternative correctness assignments do not affect verifier feedback, they leave the training record's distribution unchanged. Thus, any monitor's detection rate under $c_1$ equals its false-alarm rate under $c = R$. Even knowing that hacks exist does not reveal which accepted responses are wrong.
\vspace{-.5em}
\subsection{Limits of Identifying Hacks} \label{subsec:limits_of_id}
Suppose we know that hacks exist. The remaining task is to determine which accepted responses are wrong. \emph{Identification} requires a monitor to recover the true label $c(x,y)$ of an accepted response from $L_t$. To succeed, the prediction $\widehat{c}_t(x,y)$ must distinguish compatible assignments that disagree on this response, even when we restrict $\mathcal{C}_R$ to assignments that admit hacks.

\begin{proposition}[Limits of identification] \label{prop:verifier_identification_limit}
Under the comparison setup above, fix $\theta$ and $t$. Suppose $\theta$ assigns positive probability to both correct responses and hacks under some assignment in $\mathcal{C}_R$. For every monitor and every accepted pair $(x,y)$, there is an assignment $c_1 \in \mathcal{C}_R$ with $\Pr_\theta(R=1, c_1=0) > 0$ such that $\Pr_{c_1}\!\left(\widehat{c}_t(x,y) \ne c_1(x,y)\right) \ge \tfrac{1}{2}.$ Thus, no monitor can guarantee the correct label of an accepted response, even when we know hacks exist.
\end{proposition}
\begin{proof}
See Appendix~\ref{app:verifier_only_identification}.
\end{proof}
\textbf{Interpretation.} Two compatible assignments can label the same accepted response differently while producing identical records. Both can admit hacks, so knowing that hacks exist does not resolve the disagreement. A prediction correct under one assignment is wrong under the other. Thus, the monitor incurs an error probability of at least $\frac{1}{2}$ under at least one assignment. Because reducing hacks need not require identifying every hack, this result alone does not rule out selective control.
\vspace{-.5em}
\subsection{Limits on Selective Control from Verifier Feedback Alone} \label{subsec:limits_of_control}
Even if we cannot identify every hack, can we reduce hack probability without reducing the probability of correct responses? Suppose a controller uses $L_t$ to apply a correction $u(t) \in \mathbb{R}^d$ to the verifier flow, yielding $\dot{\theta} = g_R + u(t)$. Along this corrected flow, \emph{selective control} requires $\dot{p}_H < 0$ and $\dot{p}_G \geq 0$ whenever $p_H > 0$. We target $p_H$ itself, not the hacked share $q$, because $q$ can fall even while $p_H$ grows. For a uniform guarantee, this same controller must satisfy these conditions under every compatible assignment in~$\mathcal{C}_R$.

\begin{proposition}[Limits of selective control]
\label{prop:verifier_control_limit}
Under the comparison setup above, suppose the initial policy assigns positive probability to both correct responses and hacks under some assignment in~$\mathcal{C}_R$. For corrected flows with differentiable group probabilities, no controller using only $L_t$ can guarantee $\dot{p}_H < 0$ and $\dot{p}_G \geq 0$ whenever $p_H > 0$, under every assignment $c' \in \mathcal{C}_R$. Here, each assignment $c'$ defines its own $p_H$ and $p_G$.
\end{proposition}
\vspace{-.5em}
\begin{proof}
See Appendix~\ref{app:limit_control}.
\end{proof}
\vspace{-.5em}
\textbf{Interpretation.} We seek a correction $u(t)$ such that the corrected flow reduces hack probability without reducing the probability of correct responses. However, compatible assignments can exchange the roles of correct responses and hacks. An update that reduces hack probability under one therefore reduces the probability of correct responses under the other. The controller cannot distinguish these assignments from $L_t$, and more observations from the same verifier cannot resolve this conflict. Regularization illustrates this limit: it constrains policy updates using only the policy and the verifier, so it cannot guarantee selective control either (Appendix~\ref{app:control_examples}). A uniform guarantee of selective control thus requires additional correctness information that rules out assignments requiring incompatible updates.

\section{From Additional Feedback to Selective Control} \label{sec:selective_control}
Section~\ref{sec:verifier_feedback} showed why verifier feedback alone cannot reduce hacks without also sacrificing correct responses. We now investigate whether additional information about correctness can break this trade-off. We show that it can: such information supports a correction that suppresses hacks and promotes correct responses, provided the correction is strong enough to overcome the drift toward errors induced by RLVR training.
\vspace{-.5em}
\subsection{Projected Audit Correction in RLVR Training} \label{subsec:projected_audit_correction}

\textbf{Audits as additional information.} Verifier feedback alone cannot guarantee selective control (Section~\ref{sec:verifier_feedback}), so we introduce a stronger signal: audits. An \emph{audit} reveals the true correctness label~$c(x,y)$ for a prompt~$x$ and an accepted response~$y$. The audit $Z=(x,y,c(x,y))$ rules out every correctness rule that disagrees with this label, restricting~$\mathcal{C}_R$ to $\mathcal{C}_{R,Z} = \{c'\in\mathcal{C}_R : c'(x,y)=c(x,y)\}.$ The limits in Section~\ref{sec:verifier_feedback} arose because verifier feedback could not distinguish the true rule from alternatives in~$\mathcal{C}_R$. Audits eliminate alternatives that contradict the revealed labels. To show how audits reduce hacking, we next define the audited hack population and its probability gradient under the policy. 

Let $A_\mathrm{aud}$ be a fixed set of audits, and let $H_A = A_\mathrm{aud} \cap H$ denote the audited hacks, with probability $p_{H_A}(\theta):=\Pr_\theta(H_A)$. When $p_{H_A}$ is differentiable, the negative gradient $-\nabla_\theta p_{H_A}$ targets only the audited hacks and need not align with $-\nabla_\theta p_H$ (see Appendix~\ref{app:practical_considerations}). Despite this misalignment, $-\nabla_\theta p_{H_A}$ locally decreases the audited hack probability when $\nabla_\theta p_{H_A} \neq 0$, so we use the correction $u = -\lambda \nabla_\theta p_{H_A}$, with $\lambda > 0$, in the verifier flow $\dot{\theta} = g_R + u$.

The correction opposes growth in~$p_{H_A}$: $\dot p_{H_A} = \nabla p_{H_A}^{\top}g_R - \lambda\|\nabla p_{H_A}\|^2.$ Because these probabilities share parameters, the correction also perturbs~$p$ and~$p_G$:
$$
\dot p = \|g_R\|^2 - \lambda\,g_R^\top\nabla p_{H_A},
\qquad
\dot p_G = \nabla p_G^\top g_R
  - \lambda\,\nabla p_G^\top\nabla p_{H_A}.
$$
The cross terms have no fixed sign. When $\nabla p_G^\top\nabla p_{H_A}>0$, the correction can suppress correct responses alongside hacks. We next constrain~$u$ to preserve the growth rate of~$p$.

\textbf{Projected audit correction.} Since $p=p_H+p_G$, we have $\nabla_\theta p_G = g_R - \nabla_\theta p_H$. Preserving the instantaneous growth rate of~$p$ requires $g_R^\top u=0$, so~$u$ must lie in the subspace orthogonal to~$g_R$. Under this constraint, $\nabla_\theta p_G^\top u = -\nabla_\theta p_H^\top u$: whenever~$u$ opposes the growth of~$p_H$ ($\nabla_\theta p_H^\top u<0$), it contributes equally to the growth
of~$p_G$.

We project the audit correction onto the subspace orthogonal to~$g_R$ using $P_\perp:=I-g_Rg_R^\top/\|g_R\|^2$ when $g_R\neq0$ and $P_\perp:=I$ otherwise, to define the \emph{projected audit correction} $u=-\lambda P_\perp\nabla p_{H_A}$, with $\lambda>0$, yielding the flow $\dot\theta=g_R+u$. Since $g_R^\top P_\perp=0$ and $P_\perp$ is an orthogonal projector,
$$
\dot p=\|g_R\|^2,
\qquad
\dot p_{H_A}
=\nabla p_{H_A}^\top g_R
-\lambda\|P_\perp\nabla p_{H_A}\|^2.
$$
This projection preserves the instantaneous growth rate of~$p$ and opposes the growth of~$p_{H_A}$ whenever $P_\perp\nabla p_{H_A}\neq 0$. We next show when this projected correction shrinks the share of hacks in the full population and bounds that share during training.

\subsection{Main Result: Achieving Selective Control} \label{subsec:achieving_selective_control}
\textbf{Assumptions.} We analyze RLVR training under the projected audit correction on a finite interval~$[0,T]$, under the following assumptions. \textit{(A1) Regularity:} The probabilities $p_G,p_H,p_{H_A}$ are continuously differentiable in~$\theta$, with $p_G,p_H>0$ on~$[0,T]$. \textit{(A2) Audit coverage:} The fixed audit set satisfies $\Pr_{\theta(t)}(H\setminus A_\mathrm{aud})=0$ throughout the interval, so $p_{H_A}=p_H$ and their gradients coincide during training. \textit{(A3) Uniform corrective strength:} $\|P_\perp\nabla p_H\|^2/p\ge\kappa q$ throughout the interval for some constant $\kappa>0$. We discuss these conditions in Appendix~\ref{app:practical_considerations}.

\begin{theorem}[Selective control]
\label{thm:selective_control}
Consider RLVR training under the projected audit correction with constant $\lambda>0$ on $[0,T]$.

\textit{(i) Selective control.} Under (A1)--(A2), the correction opposes the relative growth of hacks while preserving the instantaneous growth rate of~$p$:
$$
\dot z
=
\underbrace{(\bar s_H-\bar s_G)^\top g_R}
_{\text{bias + leakage}}
\;-\;
\underbrace{\frac{\lambda\|P_\perp\nabla p_H\|^2}
{p\,q(1-q)}}_{\text{correction}},
\qquad
\dot p=\|g_R\|^2.
$$
Whenever the correction term exceeds the bias and leakage, the share of hacks decreases and~$p_G$ increases. More strongly, if $\lambda\|P_\perp\nabla p_H\|^2 >\nabla p_H^\top g_R,$ then~$p_H$ itself decreases ($\dot p_H<0$, $\dot p_G>0$, and $\dot z<0$), achieving selective control.

\textit{(ii) An ISS-type bound on the share of hacks.} Assume additionally (A3), and let $D\ge0$ bound the growth pressure, $(\bar s_H-\bar s_G)^\top g_R\le D$, during training. Then, for every $t\in[0,T]$,
$$
q(t)\le
e^{-\lambda\kappa t}\,q(0)
+\frac{D}{4\lambda\kappa}
\bigl(1-e^{-\lambda\kappa t}\bigr).
$$
The bound separates a decaying contribution from the initial hacked share and a residual contribution from growth pressure. For fixed~$D$, a larger product~$\lambda\kappa$ reduces the residual bound.
\end{theorem}
\begin{proof}[Proof sketch] By (A2), $\nabla p_{H_A}=\nabla p_H$. Substitute the correction into $\dot z=\nabla z^\top(g_R+u)$ and use $g_R^\top u=0$ to obtain part~(i). For part~(ii), (A3) and $q(1-q)\le1/4$ give $\dot q\le D/4-\lambda\kappa q$; integration yields the bound. See Appendix~\ref{app:proof_selective_control} for details. 
\end{proof}
\textbf{Interpretation.} Under the stated assumptions, the projected audit correction favors correct responses over hacks while preserving the instantaneous growth rate of~$p$. When bias and leakage supply no positive growth pressure, the input-to-state stability (ISS)-type bound~\citep{khalil2002nonlinear} ensures the share of hacks decays exponentially. Persistent pressure contributes a residual bound that decreases with~$\lambda\kappa$ for fixed~$D$. The selective effect is instantaneous; the trajectory bound requires the assumptions to hold throughout the interval. 

\textbf{Remark (Practical implementation).} The control guarantees (Theorem~\ref{thm:selective_control}) require audits that provide a direction opposing hack growth and sufficient corrective strength throughout training. Partial audit coverage can suffice when the resulting correction remains aligned with the gradient of hack probability. Finite-sample gradient estimates, optimizers, and finite step sizes can prevent the implemented update from preserving acceptance progress or suppressing hacks. Appendix~\ref{app:practical_considerations} discusses these requirements and how to estimate and validate the correction.

\section{Experiments} \label{sec:experiments}
We test reward hacking, the limits of verifier feedback, and selective control through audit correction in contextual bandits and language models. The code is available at \url{https://github.com/cmoyacal/verifier-errors}.
\vspace{-.5em}
\subsection{Gaussian Contextual Bandits}
\label{subsec:exp_bandits}
\textbf{Setup.} We use Gaussian contextual bandits to study when training on verifier rewards amplifies reward hacking and whether audits enable selective control. We set the initial probabilities of correct and hacked responses, $p_G(0)$ and $p_H(0)$, to isolate how the initial composition shapes training. We also vary the policy class: a log linear policy keeps its features fixed, while a neural policy learns its own representation. Appendices~\ref{app:gaussian_bandit} and~\ref{app:neural_bandit} provide the experimental details.

\textbf{Results.} Figure~\ref{fig:bandit-growth-four-panel}~(a) shows acceptance $p$ and hacked share $q$ increasing together from $q(0)=0.3$, as Proposition~\ref{prop:hacking_growth} predicts. Panel~(b) plots correctness $J_C=p_G$ against verifier reward for three initial hacked shares. For $q(0)=0.5$, correctness declines throughout the recorded interval; for $q(0)=0.3$, it first improves and then declines; for $q(0)=0.1$, both correctness and verifier reward improve. Both declines are reward hacking in the sense of Proposition~\ref{prop:reward_hacking_flow}. Panel~(c) evaluates the same verifier flow under two correctness assignments: $c=R$, which admits no hacked responses, and $c_1=\mathbf1_G$, which does. The information available during training is identical under both, yet correctness differs, as Proposition~\ref{prop:verifier_information_limit} predicts. Panel~(d) shows selective control with a neural policy: projected audit correction (PAC) decreases the probability $p_H$ of hacked responses while increasing correctness $p_G$, consistent with Theorem~\ref{thm:selective_control}. Raw audit correction initially decreases both probabilities, whereas verifier flow and gradient regularization increase $p_H$. Appendices~\ref{app:gaussian_bandit} and~\ref{app:neural_bandit} provide additional experiments and ablations.

\begin{figure*}[t]
\centering
\begingroup
\pgfplotsset{
  growth axis/.style={
    width=0.175\textwidth,,
    height=0.19\textwidth,
    scale only axis,
    tick label style={font=\small},
    xlabel style={font=\small},
    ylabel style={
    font=\small,
    at={(axis description cs:-0.16,0.5)},
    anchor=south
    },
    title style={font=\small},
    every axis plot/.append style={line width=1.6pt},
    legend style={font=\small, draw=none, fill=none,
                  row sep=-1pt, cells={anchor=west}},
    legend cell align=left,
    enlargelimits=false,
  },
}
\def\GrowthTimePanel#1{
  \nextgroupplot[growth axis, title={#1}, xmin=0, xmax=50, ymin=0, ymax=1,
    xtick={0,25,50}, ytick={0,0.5,1}, xlabel={Training time}, ylabel={Probability},
    legend style={at={(0.02,0.02)},anchor=south west,
                  font=\small,draw=none,fill=none,row sep=-1pt}]
  \addplot[blue, solid] table[x=t,y=p] {bandit/growth_A.dat};
  \addlegendentry{$p$}
  \addplot[orange, dashed] table[x=t,y=q] {bandit/growth_A.dat};
  \addlegendentry{$q$}
  \addplot[black, dashdotted] table[x=t,y=p_G] {bandit/growth_A.dat};
  \addlegendentry{$p_G$}
}
\def\GrowthRewardPanel#1{
  \nextgroupplot[growth axis, title={#1}, xmin=0.66, xmax=1,
    ymin=0, ymax=1, xtick={0.7,0.85,1}, ytick={0,0.5,1},
    xlabel={$J_R=p$}, ylabel={$J_C=p_G$},
    legend style={at={(0.97,0.97)},anchor=north east,
              font=\small,draw=none,fill=none,row sep=-1pt}]
  \addplot[gray,postaction={decorate},
    decoration={markings,mark=at position 0.55 with {\arrow{stealth}}}]
    table[x=p,y=p_G] {bandit/trajectory_00.dat};
  \addlegendentry{$0.1$}

  \addplot[orange,postaction={decorate},
    decoration={markings,mark=at position 0.55 with {\arrow{stealth}}}]
    table[x=p,y=p_G] {bandit/trajectory_01.dat};
  \addlegendentry{$0.3$}
  \addplot[orange, forget plot, line width=1.6pt,unbounded coords=jump]
    table[x=p,y=p_G_declining] {bandit/trajectory_01.dat};
  \addplot[orange,forget plot, only marks,mark=o,mark size=1.7pt]
    table[x=p,y=p_G] {bandit/stationary_01.dat};

  \addplot[black,postaction={decorate},
    decoration={markings,mark=at position 0.55 with {\arrow{stealth}}}]
    table[x=p,y=p_G] {bandit/trajectory_02.dat};
  \addlegendentry{$0.5$}
  \addplot[black,line width=1.6pt,unbounded coords=jump]
    table[x=p,y=p_G_declining] {bandit/trajectory_02.dat};
}
\def\InformationPanel#1{%
  \nextgroupplot[
    growth axis,
    title={#1},
    xmin=0, xmax=50,
    ymin=0, ymax=1,
    xlabel={Training time},
    ylabel={Probability},
    legend style={
      at={(0.03,0.03)}, anchor=south west,
      font=\small, draw=none, fill=none, row sep=-1pt
    }
  ]
  \addplot[blue,line width=1.6pt]
    table[x=t,y=p] {bandit/information_C.dat};
  \addlegendentry{$p$}

  \addplot[black,only marks,mark=square,mark size=2pt]
    table[x=t,y=p_G_c_eq_R]
      {bandit/information_C_markers.dat};
  \addlegendentry{$p_G\ (c=R)$}

  \addplot[orange,dashed,line width=1.6pt]
    table[x=t,y=p_G_c_le_R] {bandit/information_C.dat};
  \addlegendentry{$p_G\ (c\le R)$}
}
\def\ControlPanel#1{%
  \nextgroupplot[growth axis, title={#1},
    xmin=0, xmax=0.7, ymin=0, ymax=1,
    xtick={0,0.3,0.6}, ytick={0,0.5,1},
    xlabel={$p_H$},
    ylabel={$p_G$},
    legend style={at={(0.97,0.97)},anchor=north east,
      font=\scriptsize,draw=none,fill=none,row sep=-1pt}]

  \addplot[draw=none,fill=black,fill opacity=0.15,forget plot]
    table[x=x,y=y] {neural_bandit/phase_rlvr_band.dat} \closedcycle;
  \addplot[draw=none,fill=orange,fill opacity=0.15,forget plot]
    table[x=x,y=y] {neural_bandit/phase_gr_band.dat} \closedcycle;
  \addplot[draw=none,fill=green!60!black,fill opacity=0.15,forget plot]
    table[x=x,y=y] {neural_bandit/phase_raw_fixed_band.dat} \closedcycle;
  \addplot[draw=none,fill=blue,fill opacity=0.15,forget plot]
    table[x=x,y=y] {neural_bandit/phase_projected_fixed_band.dat} \closedcycle;

  \addplot[black,line width=1.6pt, postaction={decorate},
    decoration={markings,mark=at position 0.55 with {\arrow{stealth}}}]
    table[x=p_H_mean,y=p_G_mean] {neural_bandit/phase_rlvr.dat};
  \addlegendentry{RLVR}

  \addplot[orange,dotted,line width=1.6pt,postaction={decorate},
    decoration={markings,mark=at position 0.55 with {\arrow{stealth}}}]
    table[x=p_H_mean,y=p_G_mean] {neural_bandit/phase_gr.dat};
  \addlegendentry{Grad.}

  \addplot[green!60!black,dashed,line width=1.8pt, postaction={decorate},
    decoration={markings,mark=at position 0.55 with {\arrow{stealth}}}]
    table[x=p_H_mean,y=p_G_mean] {neural_bandit/phase_raw_fixed.dat};
  \addlegendentry{Raw}

  \addplot[blue,line width=1.8pt,postaction={decorate},
    decoration={markings,mark=at position 0.55 with {\arrow{stealth}}}]
    table[x=p_H_mean,y=p_G_mean] {neural_bandit/phase_projected_fixed.dat};
  \addlegendentry{PAC}

  \addplot[black,only marks,mark=o,mark options={fill=white},
    forget plot] table[x=p_H,y=p_G] {neural_bandit/checkpoints.dat};
}
\begin{tikzpicture}
\begin{groupplot}[
  group style={group size=4 by 1,horizontal sep=1.10cm},
]
  \GrowthTimePanel{(a)}
  \GrowthRewardPanel{(b)}
  \InformationPanel{(c)}
  \ControlPanel{(d)}
\end{groupplot}
\end{tikzpicture}
\endgroup
\caption{
Reward hacking and selective control in contextual bandits. Panels (a)--(c) use a log linear policy, and panel (d) uses a neural policy. (a) Acceptance $p$, hacked share $q$, and correctness $p_G$ under verifier flow with $q(0)=0.3$.
(b) Correctness ($J_C = p_G$) against verifier reward ($J_R=p$) for hacked shares $q(0) \in \{0.1,0.3,0.5\}$.
(c) The same verifier flow evaluated under two correctness assignments: $c=R$, under which every accepted response is correct, and $c_1=\mathbf1_G$. Squares mark correctness under $c=R$, which equals acceptance $p$. (d) Correctness $p_G$ against the probability $p_H$ of hacked responses under verifier flow, gradient regularization, raw audit correction, and projected audit correction. In (b) and (d), arrows point in the direction of training.
}
\label{fig:bandit-growth-four-panel}
\end{figure*}
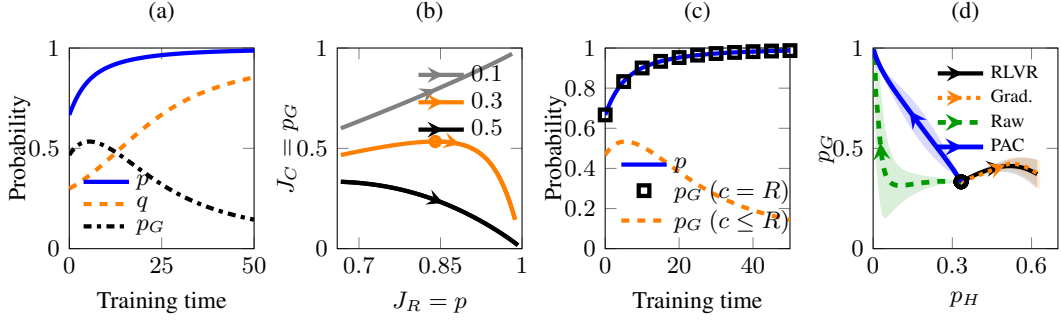
\vspace{-.5em}
\subsection{Language Models: Reward Hacking and Audit Correction}
\label{sub-sec:exp-llm}
To test whether our predictions hold beyond exact gradient flow, we train a language model with sampled gradients and finite Adam updates.

\textbf{Setup.} We consider a task in which the model replaces a sequence of digits according to given rules. A response is correct only if every replacement in the final output is correct, but the imperfect verifier~$R$ checks only the last two digits. We include in each prompt a hint with an incorrect prefix and the correct final pair, so copying it produces a hack. We initialize Qwen2-0.5B with supervised fine-tuning (SFT) on correct responses $G$, hacks $H$, and rejected responses $N$, and then run GRPO for 20 rounds with or without projected audit correction (PAC). We vary audit coverage by letting PAC audit each accepted response independently with probability $1$, $0.5$, or $0.25$. Appendix~\ref{app:llm_experiment} provides the experimental details.

\textbf{Results.} Figure~\ref{fig:llm-main}~(a) shows verifier reward $J_R$ increasing while correctness $J_C$ falls under GRPO, which illustrates reward hacking in the sense of Proposition~\ref{prop:reward_hacking_flow}. Panel~(b) shows PAC increasing $p_G$ and decreasing $p_H$ from initialization at all three audit probabilities $\rho \in \{0.25,0.5,1.0\}$, qualitatively consistent with Theorem~\ref{thm:selective_control}. On test prompts, GRPO reaches $99.4\%$ acceptance but only $2.2\%$ correctness. PAC reaches $97.2\%$ correctness with full auditing and $95.2\%$ with one quarter of accepted responses audited. Additional experiments and ablations are provided in Appendix~\ref{app:llm_experiment}.

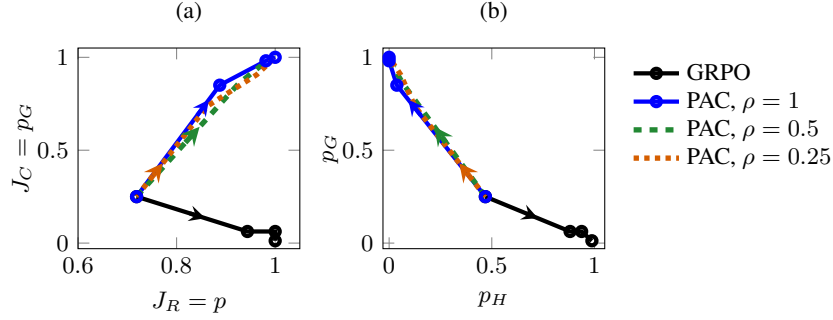
\begin{figure*}[t]
\centering
\begingroup

\def\LLMDataPath{LLM}
\definecolor{llmGreen}{HTML}{228833}
\definecolor{llmVermilion}{HTML}{D55E00}

\begin{tikzpicture}
\begin{groupplot}[
  group style={
    group size=2 by 1,
    horizontal sep=1.10cm
  },
  width=0.21\textwidth,
  height=0.19\textwidth,
  scale only axis,
  grid=none,
  tick label style={font=\small},
  xlabel style={font=\small},
  ylabel style={
    font=\small,
    at={(axis description cs:-0.16,0.5)},
    anchor=south
  },
  title style={font=\small},
  legend style={
    font=\small,
    draw=none,
    fill=none,
    row sep=-1pt,
    cells={anchor=west}
  },
  legend cell align=left,
  enlargelimits=false,
  ymin=-0.03,
  ymax=1.05,
  ytick={0,0.5,1}
]

\nextgroupplot[
  title={(a)},
  xmin=0.60,
  xmax=1.05,
  xtick={0.6,0.8,1},
  xlabel={$J_R=p$},
  ylabel={$J_C=p_G$}
]

\addplot[
  black,solid,line width=1.6pt,
  mark=o,mark size=1.7pt,
  mark options={solid,fill=white},
  forget plot
] table[
  x expr=\thisrow{J_R}/100,
  y expr=\thisrow{J_C}/100
] {\LLMDataPath/main/grpo.dat};

\addplot[
  blue,solid,line width=1.6pt,
  mark=o,mark size=1.7pt,
  mark options={solid,fill=white},
  forget plot
] table[
  x expr=\thisrow{J_R}/100,
  y expr=\thisrow{J_C}/100
] {\LLMDataPath/main/pac.dat};

\addplot[
  llmGreen,dashed,line width=2pt,
  mark=none,forget plot
] table[
  x expr=\thisrow{J_R}/100,
  y expr=\thisrow{J_C}/100
] {\LLMDataPath/main/pac_50.dat};

\addplot[
  llmVermilion,dotted,line width=2pt,
  mark=none,forget plot
] table[
  x expr=\thisrow{J_R}/100,
  y expr=\thisrow{J_C}/100
] {\LLMDataPath/main/pac_25.dat};

\draw[black,line width=1.6pt,->,>=stealth]
  (axis cs:0.826750,0.160000)
  -- (axis cs:0.858250,0.133750);

\draw[blue,line width=1.6pt,->,>=stealth]
  (axis cs:0.8419375,0.688000)
  -- (axis cs:0.8655625,0.772000);

\draw[llmGreen,line width=2pt,->,>=stealth]
  (axis cs:0.817750,0.541000)
  -- (axis cs:0.846625,0.625875);

\draw[llmVermilion,line width=2pt,->,>=stealth]
  (axis cs:0.7518125,0.362125)
  -- (axis cs:0.7719375,0.430375);

\nextgroupplot[
  title={(b)},
  xmin=-0.03,
  xmax=1.05,
  xtick={0,0.5,1},
  xlabel={$p_H$},
  ylabel={$p_G$},
  legend style={
    at={(1.08,0.98)},
    anchor=north west
  }
]

\addplot[
  black,solid,line width=1.6pt,
  mark=o,mark size=1.7pt,
  mark options={solid,fill=white}
] table[
  x expr=\thisrow{p_H}/100,
  y expr=\thisrow{J_C}/100
] {\LLMDataPath/main/grpo.dat};
\addlegendentry{GRPO}

\addplot[
  blue,solid,line width=1.6pt,
  mark=o,mark size=1.7pt,
  mark options={solid,fill=white}
] table[
  x expr=\thisrow{p_H}/100,
  y expr=\thisrow{J_C}/100
] {\LLMDataPath/main/pac.dat};
\addlegendentry{PAC, $\rho=1$}

\addplot[
  llmGreen,dashed,line width=2pt,
  mark=none
] table[
  x expr=\thisrow{p_H}/100,
  y expr=\thisrow{J_C}/100
] {\LLMDataPath/main/pac_50.dat};
\addlegendentry{PAC, $\rho=0.5$}

\addplot[
  llmVermilion,dotted,line width=2pt,
  mark=none
] table[
  x expr=\thisrow{p_H}/100,
  y expr=\thisrow{J_C}/100
] {\LLMDataPath/main/pac_25.dat};
\addlegendentry{PAC, $\rho=0.25$}

\draw[black,line width=1.6pt,->,>=stealth]
  (axis cs:0.666750,0.160000)
  -- (axis cs:0.724500,0.133750);

\draw[blue,line width=1.6pt,->,>=stealth]
  (axis cs:0.1539375,0.688000)
  -- (axis cs:0.0935625,0.772000);

\draw[llmGreen,line width=2pt,->,>=stealth]
  (axis cs:0.276750,0.541000)
  -- (axis cs:0.220750,0.625875);

\draw[llmVermilion,line width=2pt,->,>=stealth]
  (axis cs:0.3896875,0.362125)
  -- (axis cs:0.3415625,0.430375);

\end{groupplot}
\end{tikzpicture}
\endgroup

\caption{
Reward hacking and selective control in the language model, starting from an SFT demonstration mixture $G/H/N=0.3/0.3/0.4$. (a) Correctness $J_C=p_G$ against verifier reward $J_R=p$. GRPO increases reward while reducing correctness; the verifier rewards correct responses and hacks equally. (b) Correctness against the probability $p_H$ of hacks. PAC increases correctness and reduces hacks at all three audit probabilities~$\rho$, including $\rho=0.25$.
Curves show means over five seeds at calibration rounds
0, 5, 10, and 20. Circles mark these evaluations for GRPO and PAC with $\rho=1$; arrows indicate training direction. Appendix~\ref{app:llm_experiment} reports variation across seeds.
}
\label{fig:llm-main}
\end{figure*}

\section{Related Work} \label{sec:related_work}
Improving a proxy reward can reduce the intended reward~\citep{everitt2017reinforcement,NEURIPS2022_3d719fee}, and experiments show this reduction growing with stronger optimization~\citep{gao2023scaling} and appearing under automated verifiers~\citep{helff2026llms}. Methods that mitigate reward hacking constrain the policy~\citep{laidlaw2025correlated}, improve feedback~\citep{coste2024reward,lightman2024lets}, or detect and correct hacking~\citep{baker2025monitoring,wang2026detecting}, and their guarantees depend on what they observe or assume about correctness. We complement this work by fixing a verifier and characterizing theoretically when reward hacking grows, when the information available during training leaves correct and hacked responses indistinguishable, and when audits can reduce hacking while preserving progress on the intended task. Appendix~\ref{app:related_work} extends this discussion.

\section{Conclusion} \label{sec:conclusion} 
This work analyzes RLVR under a fixed imperfect verifier. First, we derive when the share of accepted errors grows, through hack bias and correctness-to-hack leakage, and when verifier reward rises while correctness falls (Propositions~\ref{prop:hacking_growth} and~\ref{prop:reward_hacking_flow}). Second, the complete training record cannot, in general, detect accepted errors, identify correctness, or guarantee fewer accepted errors while preserving correct responses (Propositions~\ref{prop:verifier_information_limit}, \ref{prop:verifier_identification_limit}, and~\ref{prop:verifier_control_limit}). Third, projected audit correction uses correctness labels from audits of some responses to reduce reward hacking while improving correctness, under the conditions of Theorem~\ref{thm:selective_control}. Together, these results show that controlling reward hacking requires information about correctness beyond the verifier. Incomplete audits and inaccurate labels can weaken this correction. Our analysis assumes a fixed verifier and exact gradient flow, and Appendix~\ref{app:limitations_and_practical} discusses limitations and practical considerations. Extending the guarantees throughout training and correcting from limited, imperfect audits remain future work.

\bibliography{iclr2027_conference}

\begin{thebibliography}{36}
\providecommand{\natexlab}[1]{#1}
\providecommand{\url}[1]{\texttt{#1}}
\expandafter\ifx\csname urlstyle\endcsname\relax
  \providecommand{\doi}[1]{doi: #1}\else
  \providecommand{\doi}{doi: \begingroup \urlstyle{rm}\Url}\fi

\bibitem[Ackermann et~al.(2026)Ackermann, Noukhovitch, Ishida, and Sugiyama]{ackermann2026gradient}
Johannes Ackermann, Michael Noukhovitch, Takashi Ishida, and Masashi Sugiyama.
\newblock Gradient regularization mitigates reward hacking in reinforcement learning from human feedback and verifiable rewards.
\newblock In \emph{Forty-third International Conference on Machine Learning}, 2026.

\bibitem[Baker et~al.(2025)Baker, Huizinga, Gao, Dou, Guan, Madry, Zaremba, Pachocki, and Farhi]{baker2025monitoring}
Bowen Baker, Joost Huizinga, Leo Gao, Zehao Dou, Melody~Y. Guan, Aleksander Madry, Wojciech Zaremba, Jakub Pachocki, and David Farhi.
\newblock Monitoring reasoning models for misbehavior and the risks of promoting obfuscation.
\newblock \emph{arXiv preprint arXiv:2503.11926}, 2025.

\bibitem[Beigi et~al.(2026)Beigi, Jin, Zhang, Zhang, Wang, and Huang]{beigi2026ir3}
Mohammad Beigi, Ming Jin, Junshan Zhang, Jiaxin Zhang, Qifan Wang, and Lifu Huang.
\newblock {IR$^3$}: Contrastive inverse reinforcement learning for interpretable detection and mitigation of reward hacking.
\newblock \emph{arXiv preprint arXiv:2602.19416}, 2026.

\bibitem[Cai et~al.(2025)Cai, Wang, Liu, Liu, Niu, and Sugiyama]{cai2025reinforcement}
Xin-Qiang Cai, Wei Wang, Feng Liu, Tongliang Liu, Gang Niu, and Masashi Sugiyama.
\newblock Reinforcement learning with verifiable yet noisy rewards under imperfect verifiers.
\newblock \emph{arXiv preprint arXiv:2510.00915}, 2025.

\bibitem[Coste et~al.(2024)Coste, Anwar, Kirk, and Krueger]{coste2024reward}
Thomas Coste, Usman Anwar, Robert Kirk, and David Krueger.
\newblock Reward model ensembles help mitigate overoptimization.
\newblock In \emph{International Conference on Learning Representations}, pp.\  50905--50931, 2024.

\bibitem[DeepSeek-AI et~al.(2025)DeepSeek-AI, Guo, Yang, Zhang, Song, Zhang, Xu, Zhu, Ma, Wang, Bi, Zhang, Yu, Wu, Wu, Gou, Shao, Li, Gao, Liu, Xue, Wang, Wu, Feng, Lu, Zhao, Deng, Zhang, Ruan, Dai, Chen, Ji, Li, Lin, Dai, Luo, Hao, Chen, Li, Zhang, Bao, Xu, Wang, Ding, Xin, Gao, Qu, Li, Guo, Li, Wang, Chen, Yuan, Qiu, Li, Cai, Ni, Liang, Chen, Dong, Hu, Gao, Guan, Huang, Yu, Wang, Zhang, Zhao, Wang, Zhang, Xu, Xia, Zhang, Zhang, Tang, Li, Wang, Li, Tian, Huang, Zhang, Wang, Chen, Du, Ge, Zhang, Pan, Wang, Chen, Jin, Chen, Lu, Zhou, Chen, Ye, Wang, Yu, Zhou, Pan, and Li]{DBLP:journals/corr/abs-2501-12948}
DeepSeek-AI, Daya Guo, Dejian Yang, Haowei Zhang, Junxiao Song, Ruoyu Zhang, Runxin Xu, Qihao Zhu, Shirong Ma, Peiyi Wang, Xiao Bi, Xiaokang Zhang, Xingkai Yu, Yu~Wu, Z.~F. Wu, Zhibin Gou, Zhihong Shao, Zhuoshu Li, Ziyi Gao, Aixin Liu, Bing Xue, Bingxuan Wang, Bochao Wu, Bei Feng, Chengda Lu, Chenggang Zhao, Chengqi Deng, Chenyu Zhang, Chong Ruan, Damai Dai, Deli Chen, Dongjie Ji, Erhang Li, Fangyun Lin, Fucong Dai, Fuli Luo, Guangbo Hao, Guanting Chen, Guowei Li, H.~Zhang, Han Bao, Hanwei Xu, Haocheng Wang, Honghui Ding, Huajian Xin, Huazuo Gao, Hui Qu, Hui Li, Jianzhong Guo, Jiashi Li, Jiawei Wang, Jingchang Chen, Jingyang Yuan, Junjie Qiu, Junlong Li, J.~L. Cai, Jiaqi Ni, Jian Liang, Jin Chen, Kai Dong, Kai Hu, Kaige Gao, Kang Guan, Kexin Huang, Kuai Yu, Lean Wang, Lecong Zhang, Liang Zhao, Litong Wang, Liyue Zhang, Lei Xu, Leyi Xia, Mingchuan Zhang, Minghua Zhang, Minghui Tang, Meng Li, Miaojun Wang, Mingming Li, Ning Tian, Panpan Huang, Peng Zhang, Qiancheng Wang, Qinyu Chen, Qiushi Du, Ruiqi Ge, Ruisong
  Zhang, Ruizhe Pan, Runji Wang, R.~J. Chen, R.~L. Jin, Ruyi Chen, Shanghao Lu, Shangyan Zhou, Shanhuang Chen, Shengfeng Ye, Shiyu Wang, Shuiping Yu, Shunfeng Zhou, Shuting Pan, and S.~S. Li.
\newblock Deepseek-r1: Incentivizing reasoning capability in llms via reinforcement learning.
\newblock \emph{CoRR}, abs/2501.12948, 2025.

\bibitem[Denison et~al.(2024)Denison, MacDiarmid, Barez, Duvenaud, Kravec, Marks, Schiefer, Soklaski, Tamkin, Kaplan, Shlegeris, Bowman, Perez, and Hubinger]{denison2024sycophancy}
Carson Denison, Monte MacDiarmid, Fazl Barez, David Duvenaud, Shauna Kravec, Samuel Marks, Nicholas Schiefer, Ryan Soklaski, Alex Tamkin, Jared Kaplan, Buck Shlegeris, Samuel~R. Bowman, Ethan Perez, and Evan Hubinger.
\newblock Sycophancy to subterfuge: Investigating reward-tampering in large language models.
\newblock \emph{arXiv preprint arXiv:2406.10162}, 2024.

\bibitem[Eisenstein et~al.(2024)Eisenstein, Nagpal, Agarwal, Beirami, D'Amour, Dvijotham, Fisch, Heller, Pfohl, Ramachandran, Shaw, and Berant]{eisenstein2024helping}
Jacob Eisenstein, Chirag Nagpal, Alekh Agarwal, Ahmad Beirami, Alexander~Nicholas D'Amour, Krishnamurthy~Dj Dvijotham, Adam Fisch, Katherine~A. Heller, Stephen~Robert Pfohl, Deepak Ramachandran, Peter Shaw, and Jonathan Berant.
\newblock Helping or herding? reward model ensembles mitigate but do not eliminate reward hacking.
\newblock In \emph{First Conference on Language Modeling}, 2024.

\bibitem[Elliott et~al.(2026)Elliott, Urdshals, Quarel, and Murfet]{DBLP:journals/corr/abs-2605-08007}
Chris Elliott, Einar Urdshals, David Quarel, and Daniel Murfet.
\newblock Interpreting reinforcement learning agents with susceptibilities.
\newblock \emph{arXiv preprint arXiv:2605.08007}, abs/2605.08007, 2026.

\bibitem[Everitt et~al.(2017)Everitt, Krakovna, Orseau, and Legg]{everitt2017reinforcement}
Tom Everitt, Victoria Krakovna, Laurent Orseau, and Shane Legg.
\newblock Reinforcement learning with a corrupted reward channel.
\newblock In \emph{Proceedings of the 26th International Joint Conference on Artificial Intelligence}, pp.\  4705--4713, 2017.

\bibitem[Gao et~al.(2023)Gao, Schulman, and Hilton]{gao2023scaling}
Leo Gao, John Schulman, and Jacob Hilton.
\newblock Scaling laws for reward model overoptimization.
\newblock In \emph{International Conference on Machine Learning}, pp.\  10835--10866, 2023.

\bibitem[Gauthier et~al.(2026)Gauthier, Bach, and Jordan]{gauthier2026explaining}
Etienne Gauthier, Francis~R. Bach, and Michael~I. Jordan.
\newblock Explaining and preventing alignment collapse in iterative {RLHF}.
\newblock \emph{arXiv preprint arXiv:2605.04266}, 2026.

\bibitem[Helff et~al.(2026)Helff, Delfosse, Steinmann, H{\"a}rle, Shindo, Schramowski, Stammer, Kersting, and Friedrich]{helff2026llms}
Lukas Helff, Quentin Delfosse, David Steinmann, Ruben H{\"a}rle, Hikaru Shindo, Patrick Schramowski, Wolfgang Stammer, Kristian Kersting, and Felix Friedrich.
\newblock Llms gaming verifiers: Rlvr can lead to reward hacking.
\newblock \emph{arXiv preprint arXiv:2604.15149}, 2026.

\bibitem[Karwowski et~al.(2024)Karwowski, Hayman, Bai, Kiendlhofer, Griffin, and Skalse]{karwowski2024goodharts}
Jacek Karwowski, Oliver Hayman, Xingjian Bai, Klaus Kiendlhofer, Charlie Griffin, and Joar Max~Viktor Skalse.
\newblock Goodhart's law in reinforcement learning.
\newblock In \emph{The Twelfth International Conference on Learning Representations}, 2024.

\bibitem[Kenton et~al.(2026)Kenton, Janzer, Greig, Teh, Tyshchuk, Brown-Cohen, Edwards, Rajamanoharan, Siegel, Jaques, and Shah]{kenton2026debate}
Zachary Kenton, Lili Janzer, Rory Greig, Tian~Huey Teh, Kirill Tyshchuk, Jonah Brown-Cohen, Harri Edwards, Senthooran Rajamanoharan, Noah~Y. Siegel, Natasha Jaques, and Rohin Shah.
\newblock Debate training reduces reward hacking in {RLAIF}.
\newblock \emph{arXiv preprint arXiv:2608.17776}, 2026.

\bibitem[Khalaf et~al.(2025)Khalaf, Verdun, Oesterling, Lakkaraju, and Calmon]{khalaf2025inferencetime}
Hadi Khalaf, Claudio~Mayrink Verdun, Alex Oesterling, Himabindu Lakkaraju, and Flavio Calmon.
\newblock Inference-time reward hacking in large language models.
\newblock In \emph{The Thirty-ninth Annual Conference on Neural Information Processing Systems}, 2025.

\bibitem[Khalifa et~al.(2026)Khalifa, Khan, Tafveez, Peng, and Wang]{DBLP:journals/corr/abs-2603-07084}
Muhammad Khalifa, Zohaib Khan, Omer Tafveez, Hao Peng, and Lu~Wang.
\newblock Countdown-code: {A} testbed for studying the emergence and generalization of reward hacking in {RLVR}.
\newblock \emph{arXiv preprint arXiv:2603.07084}, abs/2603.07084, 2026.

\bibitem[Khalil \& Grizzle(2002)Khalil and Grizzle]{khalil2002nonlinear}
Hassan~K Khalil and Jessy~W Grizzle.
\newblock \emph{Nonlinear systems}, volume~3.
\newblock Prentice hall Upper Saddle River, NJ, 2002.

\bibitem[Kwa et~al.(2024)Kwa, Thomas, and Garriga-Alonso]{kwa2024catastrophic}
Thomas Kwa, Drake Thomas, and Adri{\`a} Garriga-Alonso.
\newblock Catastrophic goodhart: regularizing {RLHF} with {KL} divergence does not mitigate heavy-tailed reward misspecification.
\newblock In \emph{The Thirty-eighth Annual Conference on Neural Information Processing Systems}, 2024.

\bibitem[Laidlaw et~al.(2025)Laidlaw, Singhal, and Dragan]{laidlaw2025correlated}
Cassidy Laidlaw, Shivam Singhal, and Anca Dragan.
\newblock Correlated proxies: A new definition and improved mitigation for reward hacking.
\newblock In \emph{The Thirteenth International Conference on Learning Representations}, 2025.

\bibitem[Lambert et~al.(2025)Lambert, Morrison, Pyatkin, Huang, Ivison, Brahman, Miranda, Liu, Dziri, Lyu, Gu, Malik, Graf, Hwang, Yang, Bras, Tafjord, Wilhelm, Soldaini, Smith, Wang, Dasigi, and Hajishirzi]{lambert2025tulu}
Nathan Lambert, Jacob Morrison, Valentina Pyatkin, Shengyi Huang, Hamish Ivison, Faeze Brahman, Lester James~Validad Miranda, Alisa Liu, Nouha Dziri, Xinxi Lyu, Yuling Gu, Saumya Malik, Victoria Graf, Jena~D. Hwang, Jiangjiang Yang, Ronan~Le Bras, Oyvind Tafjord, Christopher Wilhelm, Luca Soldaini, Noah~A. Smith, Yizhong Wang, Pradeep Dasigi, and Hannaneh Hajishirzi.
\newblock Tulu 3: Pushing frontiers in open language model post-training.
\newblock In \emph{Second Conference on Language Modeling}, 2025.

\bibitem[Lightman et~al.(2024)Lightman, Kosaraju, Burda, Edwards, Baker, Lee, Leike, Schulman, Sutskever, and Cobbe]{lightman2024lets}
Hunter Lightman, Vineet Kosaraju, Yuri Burda, Harrison Edwards, Bowen Baker, Teddy Lee, Jan Leike, John Schulman, Ilya Sutskever, and Karl Cobbe.
\newblock Let's verify step by step.
\newblock In \emph{The Twelfth International Conference on Learning Representations}, 2024.

\bibitem[MacDiarmid et~al.(2025)MacDiarmid, Wright, Uesato, Benton, Kutasov, Price, Bouscal, Bowman, Bricken, Cloud, Denison, Gasteiger, Greenblatt, Leike, Lindsey, Mikulik, Perez, Rodrigues, Thomas, Webson, Ziegler, and Hubinger]{macdiarmid2025natural}
Monte MacDiarmid, Benjamin Wright, Jonathan Uesato, Joe Benton, Jonathan Kutasov, Sara Price, Naia Bouscal, Samuel~R. Bowman, Trenton Bricken, Alex Cloud, Carson Denison, Johannes Gasteiger, Ryan Greenblatt, Jan Leike, Jack Lindsey, Vladimir Mikulik, Ethan Perez, Alex Rodrigues, Drake Thomas, Albert Webson, Daniel~M. Ziegler, and Evan Hubinger.
\newblock Natural emergent misalignment from reward hacking in production {RL}.
\newblock \emph{arXiv preprint arXiv:2511.18397}, 2025.

\bibitem[Mahmoud et~al.(2026)Mahmoud, Rezaei, Wang, Gunjal, Liu, and He]{mahmoud2026reward}
Anas Mahmoud, MohammadHossein Rezaei, Zihao Wang, Anisha Gunjal, Bing Liu, and Yunzhong He.
\newblock Reward hacking in rubric-based reinforcement learning.
\newblock In \emph{Second Workshop on Agents in the Wild: Safety, Security, and Beyond}, 2026.

\bibitem[Moya \& Wang(2018)Moya and Wang]{moya2018developing}
Christian Moya and Jiankang Wang.
\newblock Developing correlation indices to identify coordinated cyber-attacks on power grids.
\newblock \emph{IET Cyber-Physical Systems: Theory \& Applications}, 3\penalty0 (4):\penalty0 178--186, 2018.

\bibitem[Moya et~al.(2026)Moya, Semendinger, Lin, and Thornley]{moya2026spurious}
Christian Moya, Alex Semendinger, Guang Lin, and Elliott Thornley.
\newblock Spurious correlation learning in preference optimization: Mechanisms, consequences, and mitigation via tie training.
\newblock In \emph{Forty-third International Conference on Machine Learning}, 2026.

\bibitem[Pan et~al.(2022)Pan, Bhatia, and Steinhardt]{pan2022the}
Alexander Pan, Kush Bhatia, and Jacob Steinhardt.
\newblock The effects of reward misspecification: Mapping and mitigating misaligned models.
\newblock In \emph{International Conference on Learning Representations}, 2022.

\bibitem[Pan et~al.(2024)Pan, He, Bowman, and Feng]{pan2024spontaneous}
Jane Pan, He~He, Samuel~R. Bowman, and Shi Feng.
\newblock Spontaneous reward hacking in iterative self-refinement.
\newblock \emph{arXiv preprint arXiv:2407.04549}, 2024.

\bibitem[Pasqualetti et~al.(2013)Pasqualetti, D{\"o}rfler, and Bullo]{pasqualetti2013attack}
Fabio Pasqualetti, Florian D{\"o}rfler, and Francesco Bullo.
\newblock Attack detection and identification in cyber-physical systems.
\newblock \emph{IEEE transactions on automatic control}, 58\penalty0 (11):\penalty0 2715--2729, 2013.

\bibitem[Skalse et~al.(2022)Skalse, Howe, Krasheninnikov, and Krueger]{NEURIPS2022_3d719fee}
Joar Skalse, Nikolaus Howe, Dmitrii Krasheninnikov, and David Krueger.
\newblock Defining and characterizing reward gaming.
\newblock In \emph{Advances in Neural Information Processing Systems}, volume~35, pp.\  9460--9471, 2022.

\bibitem[Wang \& Murfet(2026)Wang and Murfet]{DBLP:journals/corr/abs-2601-13548}
George Wang and Daniel Murfet.
\newblock Patterning: The dual of interpretability.
\newblock \emph{arXiv preprint arXiv:2601.13548}, abs/2601.13548, 2026.

\bibitem[Wang et~al.(2026{\natexlab{a}})Wang, Pham, Yin, Wang, Chen, Durrett, and Ye]{wang2026detecting}
Songtao Wang, Quang~Hieu Pham, Fangcong Yin, Xinpeng Wang, Jocelyn~Qiaochu Chen, Greg Durrett, and Xi~Ye.
\newblock Detecting and suppressing reward hacking with gradient fingerprints.
\newblock In \emph{Third Conference on Language Modeling}, 2026{\natexlab{a}}.

\bibitem[Wang et~al.(2026{\natexlab{b}})Wang, Joshi, Plank, Angell, and He]{wang2026is}
Xinpeng Wang, Nitish Joshi, Barbara Plank, Rico Angell, and He~He.
\newblock Is it thinking or cheating? detecting implicit reward hacking by measuring reasoning effort.
\newblock In \emph{The Fourteenth International Conference on Learning Representations}, 2026{\natexlab{b}}.

\bibitem[Wang et~al.(2026{\natexlab{c}})Wang, Hao, Hou, Peng, Li, and Wang]{wang2026reproducing}
Xuekang Wang, Zhuoyuan Hao, Shuo Hou, Hao Peng, Juanzi Li, and Xiaozhi Wang.
\newblock Reproducing, analyzing, and detecting reward hacking in rubric-based reinforcement learning.
\newblock \emph{arXiv preprint arXiv:2606.04923}, 2026{\natexlab{c}}.

\bibitem[Zhong et~al.(2025)Zhong, Raghunathan, and Carlini]{zhong2025impossiblebench}
Ziqian Zhong, Aditi Raghunathan, and Nicholas Carlini.
\newblock {ImpossibleBench}: Measuring {LLMs}' propensity of exploiting test cases.
\newblock \emph{arXiv preprint arXiv:2510.20270}, 2025.

\bibitem[Zhuang \& Hadfield-Menell(2020)Zhuang and Hadfield-Menell]{NEURIPS2020_b607ba54}
Simon Zhuang and Dylan Hadfield-Menell.
\newblock Consequences of misaligned ai.
\newblock In \emph{Advances in Neural Information Processing Systems}, volume~33, pp.\  15763--15773, 2020.

\end{thebibliography}
\bibliographystyle{iclr2027_conference}

\newpage 
\appendix

\section{Limitations and Practical Considerations} \label{app:limitations_and_practical}
This appendix states the assumptions that limit the scope of our theoretical guarantees and describes how to implement the audit correction.

\subsection{Limitations} \label{app:limitations}

\textbf{No false negatives.} To isolate whether rewarding hacks can reduce correctness $p_G$, we assume that the verifier accepts every correct response. Under this assumption, a decline in $p_G$ cannot come from the verifier rejecting correct responses. The assumption also gives $p=p_G+p_H$, because the accepted responses then consist of all correct responses and all hacked responses. Thus, any change that reduces $p_H$ without decreasing $p$ increases $p_G$.

If false negatives are allowed, $p_G$ still denotes total correctness, but the probability of rejected correct responses adds a term:
$$
p_G=p-p_H+\Pr_\theta(c=1,R=0).
$$
Differentiating along the flow and using $\dot p=\|g_R\|^2$ gives
$$
\dot p_G
=
\|g_R\|^2-\dot p_H
+\frac{d}{dt}\Pr_{\theta(t)}(c=1,R=0).
$$
The first two terms give the rate at which the probability of accepted correct responses changes. Thus, even when $\dot p_H<0$, total correctness decreases whenever
$$
\frac{d}{dt}\Pr_{\theta(t)}(c=1,R=0) < -\bigl(\|g_R\|^2-\dot p_H\bigr),
$$
that is, whenever the probability of rejected correct responses falls faster than the probability of accepted correct responses rises. Extending the control guarantee to this case therefore requires a lower bound on $\frac{d}{dt}\Pr_{\theta(t)}(c=1,R=0)$. Proposition~\ref{prop:hacking_growth}, which decomposes the growth of the hacked share, still applies within the accepted population, with the average score of accepted correct responses in place of $\bar s_G$.

\textbf{Population gradient flow.} We analyze updates in continuous time driven by the exact gradient of $J_R$. Exact gradients remove sampling noise and isolate the effect of verifier rewards. Practical training departs from this idealization through finite steps, adaptive optimizers such as Adam, gradient clipping, and additional objectives such as KL regularization, each of which can change the dynamics. Our growth and control guarantees therefore hold exactly only for the stated flows. Our experiments with discrete updates check whether the predicted behavior persists in the tested settings, but they do not extend the guarantees to other training algorithms.

\textbf{Local growth and control over a finite interval.} Proposition~\ref{prop:hacking_growth} characterizes the growth of the hacked share at the current policy. The group scores change during training, so a positive growth rate at one time does not guarantee continued growth. Theorem~\ref{thm:selective_control} extends the analysis from a single time to the interval $[0,T]$, but its bound requires the assumptions to hold throughout this interval and does not establish convergence beyond it. Similarly, the identity $\dot p=\|g_R\|^2$ under the projected correction is local: the correction preserves the rate at which acceptance grows at the current policy. Because the correction changes the trajectory of the policy, this identity does not imply that corrected training follows the same acceptance curve or reaches the same final acceptance as uncorrected training.

\textbf{Scope of the information limits.} Our impossibility results (Section~\ref{sec:verifier_feedback}) concern methods that observe only the training record $L_t$: no such method can provide a uniform guarantee, one that holds under every correctness assignment in $\mathcal C_R$. These results do not imply that every monitor fails on every task. Additional knowledge of task requirements can exclude assignments in $\mathcal C_R$, and if it excludes the indistinguishable assignments used in our proofs, our impossibility arguments no longer apply. Conversely, allowing false negatives enlarges the class of assignments, and the enlarged class still contains these indistinguishable assignments, so it admits no uniform guarantee either.

\textbf{Fixed verifier and prompt distribution.} We hold the binary verifier, correctness labels, and prompt distribution fixed during training. In practice, the verifier or prompt distribution can change during training, for example when tests are added or a curriculum reorders prompts. Such changes shift population probabilities without any policy update, and our dynamics exclude them. Because $p_G$ and $p_H$ average over the prompt distribution, increasing $p_G$ or decreasing $p_H$ does not guarantee that correctness improves on every prompt.

\subsection{Practical Considerations} \label{app:practical_considerations}
\textbf{Limited audit coverage.} We assumed full audit coverage in~Assumption (A2) for Theorem~\ref{thm:selective_control}, so that $\nabla p_{H_A}=\nabla p_H$. Without full coverage, the correction changes the growth of~$p_H$ by $\nabla p_H^\top u =-\lambda(P_\perp\nabla p_H)^\top(P_\perp\nabla p_{H_A})$. When the projected gradients align positively, the correction opposes the growth of~$p_H$ and, as in Theorem~\ref{thm:selective_control}, equally promotes the growth of~$p_G$. Full coverage guarantees positive alignment when $P_\perp\nabla p_H\neq0$, but is not necessary: a subset of audits can supply a direction that opposes hacking. We next ask whether this direction remains effective as the policy changes.

\textbf{Sustaining corrective strength.} A direction that opposes hacking at one policy may weaken or reverse as on-policy training shifts the distribution and gradients. Accurate audits are not enough: they must continue to supply sufficient corrective strength throughout training. Retaining the other assumptions of the theorem,
the same ISS-type bound holds under partial coverage if
$(P_\perp\nabla p_H)^\top(P_\perp\nabla p_{H_A})/p
\geq\kappa q$ throughout the interval for some fixed
$\kappa>0$; under full coverage, this condition reduces
to~(A3). We next ask how to estimate the corrective direction from finite samples.

\textbf{Estimating the correction from finite audits.} The correction uses a population gradient, but training provides only sampled responses and audit labels. Randomized audit selection provides a way to estimate the full hack gradient without auditing every generated response in each batch. For samples from the current policy, we estimate~$\nabla_\theta p_H$ using each audited hack's policy score, $\nabla_\theta\log\pi_\theta(y\mid x)$, divided by its selection probability. Audited correct responses contribute zero. Averaging these weighted contributions over all generated responses yields an unbiased estimate~$\smash{\widehat{h}}$, provided these selection probabilities are known and positive for accepted responses and the regularity conditions in Appendix~\ref{app:audit_coverage} hold. Accepted responses with zero selection probability remain outside this guarantee. Even with sufficient coverage, small selection probabilities produce large weights and can yield noisy estimates, so an individual correction may fail to oppose hacking. While sparse audits can recover the direction in expectation, reliable correction requires controlling the variance of this estimate.

\textbf{Implementing the projection~$P_\perp$.} The projection also requires an estimate~$\widehat g_R$ of the reward gradient. Set the correction orthogonal to the estimated reward gradient: $\widehat g_R^\top\widehat u=0$. To isolate projection error, we keep the base verifier gradient exact and consider $\dot\theta=g_R+\widehat u$. Then
$$
\dot p
=
\|g_R\|^2
+
(g_R-\widehat g_R)^\top\widehat u,
$$
where the second term is the projection error. Thus, orthogonality to the estimated gradient need not preserve the true instantaneous growth rate of~$p$. Estimating the base verifier gradient introduces additional error. Optimizer transformations and finite step sizes can introduce further discrepancies. Practical implementation requires validating that discrete updates oppose hacking while approximately preserving acceptance progress.

Our analysis identifies three practical tasks: selecting which responses to audit, estimating the corrective direction from these audits, and validating that each update opposes hacking while approximately preserving acceptance progress. Susceptibility methods for reinforcement learning~\citep{DBLP:journals/corr/abs-2605-08007} and patterning~\citep{DBLP:journals/corr/abs-2601-13548} could help guide audit selection and correction. Adapting these methods to sustain selective control under limited audit and computation budgets remains an open engineering challenge.
\newpage

\newpage

\section{Related Work} \label{app:related_work}
We position our work within two lines of research on reward hacking: why optimizing a proxy reward can undermine the intended reward, and how to prevent it.

\paragraph{Reward hacking.} Theoretical work shows that improving a proxy reward can reduce the intended reward~\citep{everitt2017reinforcement,NEURIPS2022_3d719fee}. This conflict can arise when the proxy omits relevant attributes~\citep{NEURIPS2020_b607ba54}, and its severity depends on properties of the proxy and its errors~\citep{laidlaw2025correlated,kwa2024catastrophic}. In experiments, stronger optimization and more capable models can widen the gap between proxy and intended rewards~\citep{gao2023scaling,pan2022the}. Language models also show this gap during iterative self-refinement~\citep{pan2024spontaneous} and when automated verifiers accept responses that violate task requirements~\citep{helff2026llms,zhong2025impossiblebench,mahmoud2026reward}. Training can amplify such reward hacking~\citep{DBLP:journals/corr/abs-2603-07084}, and learned hacking can generalize to reward tampering or broader misalignment~\citep{denison2024sycophancy,macdiarmid2025natural}. To explain how reward hacking develops, prior work studies the geometry of optimization~\citep{karwowski2024goodharts}, selection at inference time~\citep{khalaf2025inferencetime}, and feedback between policy training and retraining of the reward model~\citep{gauthier2026explaining}. For rewards assigned by judges, \citet{wang2026reproducing} further separate how easily models discover biases from how strongly they exploit them. Compared to these works, we study how reward hacking grows under a fixed binary verifier by deriving conditions under which policy updates favor hacked responses over equally rewarded correct responses. We then show that the reward objective provides no preference that would restore correct responses within the accepted set, which explains why reward maximization alone need not reverse this growth.

\paragraph{Mitigating reward hacking.}
Regularization can limit reward hacking by constraining how much a policy changes~\citep{laidlaw2025correlated} or by favoring flatter optima, a property that theory links to the accuracy of rewards under additional assumptions~\citep{ackermann2026gradient}. Other methods improve the feedback used for optimization by combining reward models~\citep{coste2024reward}, adding constraints~\citep{helff2026llms}, supervising reasoning steps~\citep{lightman2024lets}, or supplying criticism through debate~\citep{kenton2026debate}. When verifier errors follow an assumed noise model, known error rates can also support corrections to training updates~\citep{cai2025reinforcement}. Even improved feedback can leave errors unresolved, as \citet{eisenstein2024helping} show for reward ensembles. To detect reward hacking, monitors inspect a model's reasoning~\citep{baker2025monitoring} or measure how much reasoning it needs to pass the verifier~\citep{wang2026is}. Detection can then guide correction: Grift groups responses by their gradients, labels each group by inspecting examples, and filters responses for fine-tuning~\citep{wang2026detecting}, while IR$^3$ reconstructs rewards and targets features it identifies as problematic~\citep{beigi2026ir3}. The guarantees of these approaches depend on what they observe or assume about correctness, a dependence that \citet{everitt2017reinforcement} examine for corrupted rewards. Our identifiability results specify when the information available during training leaves correct and hacked responses indistinguishable. We then characterize how auditing supplies information for selective correction and how limited coverage or errors in the audits can undermine this correction.
\newpage

\section{Technical Details and Proofs} \label{app:details_and_proofs}
\subsection{Probabilities and Dynamics for Section~\ref{sec:problem_formulation}}

\paragraph{Idealized conditions.}  We use the fixed verifier $R$ and correctness rule $c$ from Section~\ref{sec:problem_formulation}, with $c\le R$, so the verifier has no false negatives. Without correction, the parameters follow exact Euclidean gradient ascent on the verifier reward, $\dot\theta=\nabla J_R(\theta)$. We assume that the objectives and group probabilities are continuously differentiable and that the stated flows exist on the intervals considered. When differentiating expectations, we assume sufficient regularity to interchange differentiation and expectation.

\subsubsection{Fixed Sets of Correct Responses and Accepted Errors} \label{app:persistent_verifier_errors}
For each prompt $x$, the assumption $c(x,y)\leq R(x,y)$ gives $G_x\subseteq A_x$ and hence
$$
\begin{aligned}
H_x&=A_x\setminus G_x=\{y\in\mathcal Y_x:R(x,y)=1,\ c(x,y)=0\},\\
A_x&=G_x\sqcup H_x,
\end{aligned}
$$
where $\sqcup$ denotes disjoint union. These memberships stay fixed throughout training. Their probabilities can change, and $H_x$ may be empty.

\subsubsection{Acceptance and the Hacked Share} \label{app:acceptance_and_hacking}
At a fixed prompt, conditioning on acceptance gives
$$
\begin{aligned}
q_x(\theta)
&=\frac{\Pr_{\pi_\theta}(Y\in H_x\mid x)}{p_x(\theta)},\\
\Pr_{\pi_\theta}(Y\in H_x\mid x)
&=p_x(\theta)q_x(\theta),\\
\Pr_{\pi_\theta}(Y\in G_x\mid x)
&=p_x(\theta)\bigl(1-q_x(\theta)\bigr),
\end{aligned}
\qquad \text{when }p_x(\theta)>0.
$$
If $p_x=0$, the probability of an accepted error is zero and $q_x$ is undefined. Since the binary verifier rewards every response in $A_x$, iterated expectation gives
$$
J_R(\theta)
=\mathbb E_{x\sim\mathcal D}\!\left[\mathbb E_{Y\sim\pi_\theta(\cdot\mid x)}R(x,Y)\right]
=\mathbb E_{x\sim\mathcal D}[p_x(\theta)].
$$
The objective therefore depends on total acceptance, without distinguishing its correct and incorrect parts.

\subsubsection{Exact Gradient Flow} \label{app:rlvr_dynamics}
Along the flow~\eqref{eq:verifier_flow}, the chain rule gives
$$
\dot J_R
= \nabla_\theta J_R^\top \dot\theta
= \|g_R\|_2^2 \geq 0,
\qquad
\dot p_x = \nabla_\theta p_x^\top g_R,
\qquad
\dot q_x = \nabla_\theta q_x^\top g_R,
$$ 
where the identity for $\dot q_x$ applies when $p_x>0$. Under the stated assumptions, there is no general guarantee that $\dot q_x\leq 0$ or that $\dot p_x\geq 0$ at every prompt. Thus, monotonic improvement in average acceptance guarantees neither a nonincreasing hacked share nor nondecreasing acceptance at every prompt.

\newpage
\subsection{Technical Details for Section~\ref{sec:hacking}}

\subsubsection{Aggregate Acceptance and Hacking} \label{app:what_reward_ignores}
Write $\mathbb E_\theta$ for expectation under $X\sim\mathcal D$, $Y\mid X\sim\pi_\theta$. Since $G\cup H$ is the accepted set,
$$
\begin{aligned}
p&=J_R=p_G+p_H=\mathbb E_{x\sim\mathcal D}[p_x],\\
q&=\frac{p_H}{p}
=\frac{\displaystyle\int_{\{x:p_x>0\}}p_xq_x\,\mathcal D(dx)}{p}
\quad\text{when }p>0.
\end{aligned}
$$
Thus, relative to the prompt distribution $\mathcal D$, prompts are weighted by their acceptance probabilities. Prompts with zero acceptance contribute nothing. An aggregate trend need not hold at every prompt. At fixed $p$, redistributing probability between $G$ and $H$ leaves reward unchanged, although shared parameters can couple their dynamics.

\subsubsection{Proof of Proposition~\ref{prop:reward_hacking_flow}}
\label{app:proof_reward_hacking_flow} 
\citet{NEURIPS2022_3d719fee} call a proxy reward hackable if it ranks some pair of policies in strictly opposite orders from the true reward. In our setting, verifier flow generates the policies, acceptance $J_R=p$ is the proxy reward, and correctness $J_C=p_G$ is the true reward. We show that such a pair occurs along the trajectory exactly when, at some time, the correctness lost to a rising hacked share outweighs the correctness gained from rising acceptance.

\begin{proof}
Write $p(t)=p(\theta(t))$ and similarly for the other quantities.
Since $p_G=p(1-q)$,
$$
\dot J_C=\dot p_G
=
\underbrace{(1-q)\dot p}_{\text{acceptance contribution}}
-
\underbrace{p\dot q}_{\text{composition contribution}}.
$$
Consequently, the proposed inequality is equivalent to
$\dot J_C<0$. The key point is that, along verifier flow,
a strict decrease in correctness necessarily accompanies
a strict increase in verifier reward.

\emph{Sufficiency.}
Suppose that, at some $t_\star\in(0,T)$,
$$
p(t_\star)\dot q(t_\star)
>
(1-q(t_\star))\dot p(t_\star).
$$
Then $\dot J_C(t_\star)<0$. Since
$$
\dot J_C(t_\star)
=
\nabla_\theta J_C(\theta(t_\star))^\top
g_R(\theta(t_\star)),
$$
we must have $g_R(\theta(t_\star))\neq0$. Hence
$$
\dot J_R(t_\star)
=
\|g_R(\theta(t_\star))\|_2^2>0.
$$
By continuity, both strict inequalities hold on a sufficiently
short interval $[t_\star,t_\star+\varepsilon]\subset(0,T)$.
Integrating over this interval gives
$$
\begin{aligned}
J_R(\theta(t_\star+\varepsilon))
&>J_R(\theta(t_\star)),\\
J_C(\theta(t_\star+\varepsilon))
&<J_C(\theta(t_\star)).
\end{aligned}
$$
Thus the two policies at its endpoints witness reward hacking.

\emph{Necessity.}
Conversely, suppose there exist $t_1<t_2$ in $[0,T]$ such that
verifier reward increases strictly and correctness decreases
strictly. By the mean value theorem, there is some
$t_\star\in(t_1,t_2)$ with
$$
\dot p_G(t_\star)
=
\frac{p_G(t_2)-p_G(t_1)}{t_2-t_1}<0.
$$
Substituting
$\dot p_G=(1-q)\dot p-p\dot q$
and rearranging yields
$$
p(t_\star)\dot q(t_\star)
>
(1-q(t_\star))\dot p(t_\star),
$$
as required.
\end{proof}
The criterion distinguishes a growing hacked share from reward hacking in the sense of \citet{NEURIPS2022_3d719fee}. A growing hacked share can coexist with improving correctness if acceptance rises fast enough. The two rewards disagree exactly when, at some time along the flow, the correctness lost to a rising hacked share exceeds the correctness gained from rising acceptance. This inequality needs to hold strictly at only a single time: by continuity, it then holds on a short interval whose endpoints witness hackability, even if correctness later recovers. 

\subsubsection{Proof of Proposition~\ref{prop:hacking_growth}} \label{app:proof_of_mechanism}
When $p_H,p_G>0$, total acceptance cancels in the ratio:
$$
z=\log\frac{p_H}{p_G}=\log\frac{q}{1-q},
\qquad
\frac{dz}{dq}=\frac{1}{q(1-q)}>0.
$$
The log odds~$z$ therefore increase with the hacked share~$q$. They are finite only when correct responses and hacks both have positive probability, so we require this condition wherever we use finite log odds.

\paragraph{Evolution of the log odds.} We derive~\eqref{eq:log_odds_dynamics} using the gradients $\bar s_S=\nabla_\theta\log\Pr_\theta(S)$ from Section~\ref{sec:hacking}.

\subparagraph{Differentiate the log ratio.}
For $p_H,p_G>0$, the log odds $z=\log(p_H/p_G)$ are finite, and differentiating along the flow gives
$$
\dot z=\frac{\dot p_H}{p_H}-\frac{\dot p_G}{p_G}.
$$
The rate therefore compares the proportional growth of the two groups. This identity needs only differentiability of the group probabilities.

To express each proportional rate through the policy, we use the score $s_\theta(x,y):=\nabla_\theta\log\pi_\theta(y\mid x)$. For a fixed group $S$ with positive probability and indicator $\mathbf 1_S$, the conditions for differentiating expectations give
$$
\nabla_\theta\Pr_\theta(S)=\mathbb E_\theta[\mathbf 1_S s_\theta]=\Pr_\theta(S)\,\bar s_S,
\qquad
\bar s_S:=\frac{\mathbb E_\theta[\mathbf 1_Ss_\theta]}{\Pr_\theta(S)} = \mathbb E_\theta[s_\theta\mid S].
$$
The gradient of $\log\Pr_\theta(S)$ is thus the mean score $\bar s_S$ of the group. Applying this to $H$ and $G$, and using $\dot\theta=g_R$, gives
$$
\nabla_\theta z=\bar s_H-\bar s_G,
\qquad
\dot z=(\bar s_H-\bar s_G)^\top g_R.
$$
The log odds therefore grow precisely when $g_R^\top\bar s_H>g_R^\top\bar s_G$.

\paragraph{Decomposition of the growth rate.}  Assume $G,H,N$ all have positive probability and evaluate quantities at the current policy.

\subparagraph{Step 1: Use normalization.}
Differentiating the accepted probability and the sum of all group probabilities gives
$$
g_R=p\bigl[(1-q)\bar s_G+q\bar s_H\bigr],
\qquad
g_R+(1-p)\bar s_N=0.
$$
Combining these identities yields
$$
g_R=p(1-p)\bigl[(1-q)\bar s_G+q\bar s_H-\bar s_N\bigr].
$$

\subparagraph{Step 2: Substitute into the rate of the log odds.}
Using~\eqref{eq:log_odds_dynamics} and collecting $\bar s_H-\bar s_G$ gives
$$
\begin{aligned}
\dot z
&=p(1-p)(\bar s_H-\bar s_G)^\top
\bigl[(1-q)\bar s_G+q\bar s_H-\bar s_N\bigr]\\
&=p(1-p)\left[
\underbrace{(\bar s_H-\bar s_G)^\top(\bar s_G-\bar s_N)}_{\text{correctness-to-hack leakage}}
+ \underbrace{q\|\bar s_H-\bar s_G\|^2}_{\text{hack bias}}
\right],
\end{aligned}
$$
which establishes~\eqref{eq:hack_mechanisms} and reveals the mechanisms. The first term can favor or oppose hacking. The second is nonnegative and proportional to $q$ at fixed $p$ and gradients; those quantities also change during training. Hence growth depends on the full bracket, without implying acceleration or growth at every policy.

\subparagraph{Step 3: State the growth condition.}
When $\bar s_H\ne\bar s_G$, the bracket is positive exactly when
$$
q>-
\frac{(\bar s_H-\bar s_G)^\top(\bar s_G-\bar s_N)}
{\|\bar s_H-\bar s_G\|^2}.
$$
The threshold may lie outside $(0,1)$ and change during training. If the two gradients agree, $\dot z=0$. Thus the current share alone does not determine growth. The alignment of the gradients also matters.

\paragraph{Simultaneous growth of reward and hacking} \label{app:reward_improves_hacking_grows}
For $p_H,p_G>0$, differentiating $z=\log(q/(1-q))$ and using the verifier flow gives
$$
\dot q=q(1-q)\dot z,
\qquad
\dot p=\|g_R\|^2.
$$
Since $p_G,p_H,p_N>0$, both $p$ and $q$ lie in $(0,1)$. Thus, $\dot q>0$ if and only if the bracket in~\eqref{eq:hack_mechanisms} is positive. By~\eqref{eq:log_odds_dynamics}, a positive $\dot z$ requires $g_R\ne0$. Both $\dot q$ and $\dot p$ are then strictly positive, completing the proof of Proposition~\ref{prop:hacking_growth}. \hfill$\square$

\paragraph{Example: Shared normalization can favor hacking.}
For one prompt with a response in each group, use softmax logits $\theta=(\theta_G,\theta_H,\theta_N)$:
$$
\Pr_\theta(S)=\frac{e^{\theta_S}}{e^{\theta_G}+e^{\theta_H}+e^{\theta_N}},
\qquad S\in\{G,H,N\}.
$$
Here $z=\theta_H-\theta_G$ and
$$
\begin{aligned}
g_R&=\bigl(p_G(1-p),\ p_H(1-p),\ -p(1-p)\bigr)^\top,\\
\dot z&=(1-p)(p_H-p_G).
\end{aligned}
$$
Thus $p_H>p_G$ gives simultaneous growth of the hacked share and verifier acceptance on a sufficiently short interval. The example illustrates a mechanism caused by the parameterization; it does not assert that every neural policy follows this pattern.

\newpage
\subsection{Technical Details for Section~\ref{sec:verifier_feedback}} \label{app:verifier_feedback}

\subsubsection{Comparison Setup} \label{app:detector_info}
Fix the prompt distribution, verifier, policy family, initialization, and training procedure. The record $L_t$ contains prompts, responses, verifier rewards, policy updates, and quantities computed from these observations. Training and monitoring may adapt to the record, but neither uses the hidden correctness assignment or additional correctness feedback.

We compare two fixed correctness assignments:
$$
c_0=R,
\qquad
c_1\in\mathcal C_R
\quad\text{with}\quad
\Pr_\theta(R=1,c_1=0)>0,
$$
at the policy $\theta$. Under $c_0$, every accepted response is correct. Under $c_1$, some accepted responses are hacks.

The assignments are fixed before training. We compare what the same procedure observes under each assignment. Correctness does not change during either run. For the deterministic argument below, fix all random choices used by training and use those same choices in both runs.

\subsubsection{Proof of Proposition~\ref{prop:verifier_information_limit}}
\label{app:verifier_only_detection}
The proof has three steps: the assignments produce the same record, the monitor therefore makes the same decision, but correct detection requires different decisions.

\paragraph{Step 1: The training records are identical.}  We argue by induction on the training step. Both runs start from the same initialization, so their initial records agree. Suppose their records agree through step $n$. Because the training procedure depends only on the record and the shared random choices, both runs select the same prompt and sample the same response from the same policy. The fixed verifier returns the same reward. Both runs therefore make the same update and append the same information to the record, so their records agree through step $n+1$.

By induction, the records are identical at every training
step:
$$
L_n^{(c_0)}=L_n^{(c_1)}
\qquad\text{for every }n.
$$
The argument also covers verifier queries that depend on earlier results and any quantity computed from the record. It likewise covers $J_R$ and its gradients, which depend on the verifier and the policy but not on which correctness assignment in $\mathcal C_R$ is true. For training in continuous time, we assume that the fixed procedure determines a unique trajectory. Because this trajectory depends only on the verifier and the policy, changing the correctness assignment within $\mathcal C_R$ leaves it unchanged. This argument applies to any two fixed assignments in $\mathcal C_R$. It does not require either assignment to equal $R$.

\paragraph{Step 2: The monitor makes the same decision.} A deterministic detection monitor computes a binary alarm $D_t=\delta_t(L_t)$, where $D_t=1$ indicates the presence of hacks. Identical records give identical alarms:
$$
D_t^{(c_0)}
=\delta_t(L_t^{(c_0)})
=\delta_t(L_t^{(c_1)})
=D_t^{(c_1)}.
$$

\paragraph{Step 3: The same decision cannot be correct in both cases.} Under $c_0$, there are no hacks, so a correct detector must remain silent. Under $c_1$, hacks have positive probability, so a correct detector must raise an alarm. Because the monitor makes the same decision in both cases, it either raises a false alarm under $c_0$ or misses the hacks under $c_1$.

Thus, no deterministic monitor using only $L_t$ can guarantee correct detection under both compatible assignments. More observations from the same verifier do not resolve this ambiguity: Step 1 continues to apply as the record grows.

\paragraph{Scope.} This conclusion concerns a guarantee over the stated class $\mathcal C_R$. A monitor may succeed under a particular assignment. Additional knowledge connecting response content to correctness may also rule out alternatives. The impossibility applies when the two assignments remain admissible and the training procedure never observes information that distinguishes them.

\paragraph{Remark on randomization} Randomization does not change the conclusion. Use the same random choices for training and monitoring in the two runs. For every such choice, the records and alarms remain identical. Averaging over these choices therefore gives
$$
\Pr_{c=c_1}(D_t=1)
=
\Pr_{c=c_0}(D_t=1).
$$
Here the probabilities are over training and monitor randomness. Each correctness assignment remains fixed. The left side is the detection rate when hacks exist, and the right side is the false-alarm rate when they do not. Thus, the monitor has no detection advantage between these two assignments.

Likewise, the training records have the same distribution:
$$
\operatorname{Law}_{c=c_1}(L_t)
=
\operatorname{Law}_{c=c_0}(L_t),
$$
where $\operatorname{Law}_{c}(L_t)$ denotes the distribution
of the record over training randomness under assignment $c$.
This proves the proposition. \hfill$\square$

\subsubsection{Proof of Proposition~\ref{prop:verifier_identification_limit}}
\label{app:verifier_only_identification}
The proof compares two compatible assignments that disagree on every accepted response. Both admit hacks, so knowing that hacks exist does not distinguish them.

\paragraph{Step 1: Construct two opposite correctness assignments.} By the premise, choose a fixed assignment $c_1\in\mathcal C_R$ under which the policy $\theta$ assigns positive probability to both correct responses and hacks. Define
$$
c_2(x,y)=R(x,y)-c_1(x,y).
$$
On rejected responses, both assignments equal zero. On accepted responses, $c_2=1-c_1$, so the assignments give opposite correctness labels. Thus $c_2\in\mathcal C_R$, and the two assignments exchange the correct and hacked groups:
$$
p_{H,c_2}(\theta)=p_{G,c_1}(\theta)>0,
\qquad
p_{G,c_2}(\theta)=p_{H,c_1}(\theta)>0.
$$
In particular, both assignments admit hacks with positive probability.

\paragraph{Step 2: The monitor makes the same prediction.} Fix any accepted pair $(x,y)$ and any deterministic monitor using only verifier feedback records. As shown in Appendix~\ref{app:verifier_only_detection}, the two assignments produce identical training records when the same random choices are used. The monitor therefore makes the same prediction in both cases:
$$
\widehat c_t^{(c_1)}(x,y)
=
\widehat c_t^{(c_2)}(x,y).
$$

\paragraph{Step 3: The prediction cannot be correct under both assignments.} Because $R(x,y)=1$, the true labels satisfy
$$
c_2(x,y)=1-c_1(x,y).
$$
The monitor's common binary prediction therefore matches exactly one of the two labels. It must be incorrect under the other assignment.

Thus, no deterministic monitor can guarantee the correct label of an accepted response under every compatible assignment. Knowing that hacks exist does not resolve the ambiguity, because both assignments satisfy that additional information.

\paragraph{Randomization and the error bound.} The same reasoning applies when training or monitoring is randomized. Use the same random choices in both runs. For every such choice, the monitor gives the same prediction under both assignments, and exactly one assignment makes that prediction wrong.

Averaging over these choices gives
\begin{equation}
\label{eq:verifier_identification_error}
\Pr_{c=c_1}\!\left(
\widehat c_t(x,y)\ne c_1(x,y)
\right)
+
\Pr_{c=c_2}\!\left(
\widehat c_t(x,y)\ne c_2(x,y)
\right)
=1.
\end{equation}
Here the probabilities are over training and monitor randomness, with the queried pair and correctness assignments held fixed. At least one of these error probabilities is therefore at least $1/2$.

We fixed both candidate assignments before training, so neither depends on the realized training record. Which of the two attains the bound may depend on the monitor and the queried pair, but not on that record. Because both assignments contain hacks, the proposition holds even when the monitor knows that hacks exist. This completes the proof .\hfill$\square$

\subsubsection{Proof of Proposition~\ref{prop:verifier_control_limit}} \label{app:limit_control}

\paragraph{Controlled dynamics and the selective objective.} Consider a controller that chooses a correction $u(t)$ using only the training record $L_t$. Training then follows
$$
\dot\theta=g_R+u(t).
$$
We assume that the controlled trajectory is well defined and that the group probabilities are differentiable along it. For a deterministic argument, fix all random choices used during training.

Under a correctness assignment $c$, \emph{selective control} requires
$$
\dot p_{H,c}<0,
\qquad
\dot p_{G,c}\geq0
\quad\text{whenever }p_{H,c}>0.
$$
These conditions concern the full update:
$$
\dot p_{H,c}
=\nabla_\theta p_{H,c}^{\top}(g_R+u),
\qquad
\dot p_{G,c}
=\nabla_\theta p_{G,c}^{\top}(g_R+u).
$$
A correction that opposes hack growth is insufficient unless the full update $g_R+u(t)$ reduces hack probability while preserving correctness. A uniform guarantee requires the same controller rule to satisfy these conditions under every compatible correctness assignment and at every time along the evolving policy trajectory whenever hack probability
is positive. The correction itself may adapt to the training record.

\paragraph{Impossibility of guaranteed selective control.} We compare two assignments that exchange the correct and hacked groups. The controller sees the same record under both assignments, but selective control requires opposite changes in their probabilities.

\subparagraph{Step 1: Exchange the correct and hacked groups.} By the premise, choose a fixed assignment $c_1\in\mathcal C_R$ under which the initial policy assigns positive probability to both correct responses and hacks. Define
$$
c_2=R-c_1.
$$
As in Appendix~\ref{app:verifier_only_identification}, both assignments are compatible with the verifier, and they exchange the two accepted groups:
$$
p_{H,c_2}=p_{G,c_1},
\qquad
p_{G,c_2}=p_{H,c_1}.
$$
Both groups have positive probability initially. By continuity, they remain positive on a sufficiently short initial interval.

\subparagraph{Step 2: The controller produces the same update.} We extend the argument of Appendix~\ref{app:verifier_only_detection}, which shows that the two runs produce identical records, to include the controller. Both runs begin at the same policy. Whenever their records agree, the controller chooses the same correction, because it uses only the record. The verifier gradient also agrees, because it depends on the policy and the verifier but not on the correctness assignment. Both runs therefore make the same corrected update, and their records remain identical.

Thus, under both assignments, the fixed training and control procedures generate the same record and the same trajectory of the controlled policy. In particular, both runs have the same velocity $\dot\theta=g_R+u$ at every time.

\paragraph{Step 3: Selective control requires incompatible signs.}
Along this common trajectory, exchanging the groups also exchanges their derivatives:
$$
\dot p_{H,c_2}=\dot p_{G,c_1},
\qquad
\dot p_{G,c_2}=\dot p_{H,c_1}.
$$
On the initial interval where both groups remain positive, selectivity under the two assignments therefore requires
\begin{equation}
\label{eq:incompatible_selective_rates}
\begin{aligned}
c_1:\quad&
\dot p_{H,c_1}<0,
&\dot p_{G,c_1}\geq0,\\
c_2:\quad&
\dot p_{G,c_1}<0,
&\dot p_{H,c_1}\geq0.
\end{aligned}
\end{equation}
Each derivative would have to be both negative and nonnegative. No update can satisfy both rows.

Thus, no controller that uses only the training record can guarantee selective control under every correctness assignment in $\mathcal C_R$. The obstruction appears already on an initial interval of training, so it does not depend on behavior later in training.

\paragraph{Remark on randomization.}
The same obstruction applies to a randomized controller. We couple the two runs by using the same random choices for training and control under both assignments. For each realization, the records, corrections, and trajectories then coincide, so no realization satisfies the requirements for selective control under both assignments.

A guarantee that holds with probability one under both assignments would require both sets of conditions to hold on a common event of probability one. Since no realization satisfies both, that event is empty, so no such guarantee exists. Randomization cannot help, because the two candidate assignments are fixed before training and the random choices therefore carry no information that distinguishes them. The proposition follows.\hfill$\square$

\paragraph{Scope.} The result rules out a uniform guarantee over $\mathcal C_R$, but a controller may still achieve selective control under a particular assignment. Additional information or assumptions can distinguish assignments in $\mathcal C_R$, and if they exclude either of the two assignments used in the proof, the argument above no longer applies.

\paragraph{Connection to secure control.}
The argument here uses an indistinguishability principle also used when studying attack detection and identification in cyber-physical systems: no monitor can distinguish admissible scenarios that produce identical observations when it uses only those observations~\citep{pasqualetti2013attack}. In our paper, the indistinguishable scenarios are correctness assignments that produce the same training record $L_t$. Exchanging the correct and hacked groups between two such assignments makes their requirements for selective control incompatible. The argument therefore identifies the information that a uniform guarantee needs: observations or structural assumptions that exclude one of these assignments. In power grid security, models of grid operation supply such structural correctness knowledge: by relating coordinated cyberattacks to their physical consequences, they tell defenders which scenarios cause harm and let them target defense accordingly~\citep{moya2018developing}. Audits can play an analogous role by supplying correctness information that distinguishes otherwise compatible assignments. The subsequent analysis examines when this information supports selective control.

\subsubsection{Regularization and Its Limitations} \label{app:control_examples}
We analyze three regularization penalties and distinguish what each controls from what it guarantees about correctness. Throughout, we use exact gradient ascent, fixed penalty weights, and the fixed verifier from Section~\ref{sec:problem_formulation}. We then apply Proposition~\ref{prop:verifier_control_limit} to determine which guarantees of selective control these penalties can provide.

\paragraph{KL divergence from a reference}
For a fixed reference policy under the same prompt distribution, define
$$
K(\theta)
=
\mathbb E_{x\sim\mathcal D}
D_{\mathrm{KL}}\!\left(
\pi_\theta(\cdot\mid x)
\,\Vert\,
\pi_{\mathrm{ref}}(\cdot\mid x)
\right),
\qquad
F_\beta=J_R-\beta K,
$$
where $\beta>0$ is fixed.

\subparagraph{Step 1: Bound deviation from the reference.} Assume the flow remains in a region where $K$ is finite and continuously differentiable. Gradient ascent on $F_\beta$ gives
$$
\dot\theta=g_R-\beta\nabla_\theta K,
\qquad
\dot F_\beta=\|\nabla_\theta F_\beta\|^2\geq0.
$$
Thus $F_\beta(\theta(t))\geq F_\beta(\theta(0))$. Rearranging and using $J_R=p\leq1$ yields
$$
\begin{aligned}
K(\theta(t))
&\leq
K(\theta(0))
+\frac{p(\theta(t))-p(\theta(0))}{\beta}\\
&\leq
K(\theta(0))
+\frac{1-p(\theta(0))}{\beta}.
\end{aligned}
$$
In particular, initialization at the reference gives $K(\theta(0))=0$. The penalty therefore bounds departure from the reference along the exact flow.

\subparagraph{Step 2: Bound changes in hack probability.} Because the prompt distribution is fixed, $K$ also equals the KL divergence between the joint distributions of prompts and responses. Pinsker's inequality therefore gives, for the fixed hacked set $H$,
$$
|p_H(\theta)-p_{H,\mathrm{ref}}|
\leq \sqrt{\frac{K(\theta)}{2}},
\qquad
p_{H,\mathrm{ref}}:=\Pr_{\mathrm{ref}}(H).
$$
Thus, closeness to the reference limits the change in hack probability.

\subparagraph{Limitation.} The penalty bounds how much the policy changes, not the direction of that change. It therefore guarantees neither $\dot p_H<0$ nor $\dot p_G\ge0$. The bound controls hack probability relative to $p_{H,\mathrm{ref}}$. It implies a small absolute hack probability when both the reference's hack probability and the KL bound are small. The penalty depends on the policy and the reference but not on the correctness assignment, so with the same reference under every assignment in $\mathcal C_R$, Proposition~\ref{prop:verifier_control_limit} still applies.

\paragraph{Penalizing the gradient of reward.} Following the objective of
\citet{ackermann2026gradient}, we consider the idealized gradient regularization penalty
$$
F_{\mathrm{GR}}=J_R-\gamma\|g_R\|^2,
$$
with fixed $\gamma>0$.

\subparagraph{Step 1: Derive the correction.} Assume $J_R$ is twice continuously differentiable and write $B_R:=\nabla_\theta^2J_R$. Since $B_R$ is symmetric,
$$
\nabla_\theta\|g_R\|^2=2B_Rg_R.
$$
Gradient ascent on $F_{\mathrm{GR}}$ therefore gives
$$
\dot\theta=g_R-2\gamma B_Rg_R.
$$
The correction depends entirely on the verifier objective and its derivatives.

\subparagraph{Step 2: Separate reward sensitivity from correctness.} The norm $\|g_R\|$ measures first-order sensitivity of expected verifier reward to parameter changes. It does not by itself determine correctness. For example, if the verifier accepts every response, then $J_R\equiv1$ and $g_R\equiv0$. The correction vanishes even if the policy assigns positive probability to incorrect responses.

\subparagraph{Scope of the cited analysis.} \citet{ackermann2026gradient} connect flatter optima to the accuracy of the proxy reward under additional assumptions: continuous actions, a Gaussian policy with fixed covariance, regularity conditions on the policy and reward functions, and a true reward that is Lipschitz continuous. Our setup does not impose these policy and reward assumptions, and $\mathcal C_R$ places no corresponding regularity restriction on correctness assignments. Thus, the cited results do not establish selective control uniformly over $\mathcal C_R$.

Exchanging correctness assignments in $\mathcal C_R$ leaves $J_R$, its derivatives, and the correction unchanged, so Proposition~\ref{prop:verifier_control_limit} applies to this regularizer. This conclusion does not contradict the theoretical results of \citet{ackermann2026gradient}, which hold under the additional assumptions above, or their empirical improvements, which concern particular tasks rather than a uniform guarantee over $\mathcal C_R$.

\paragraph{Stability of relative probabilities.} Motivated by tie training for reducing reliance on spurious features in preference optimization \citep{moya2026spurious}, we analyze a simplified penalty on changes in relative response probabilities. The penalty discourages deviations from a fixed reference ratio between two verifier-accepted responses. Unlike pairs constructed to have equal utility, equal verifier rewards alone do not establish that the responses are equally correct. We therefore examine what this penalty guarantees when pair selection uses only verifier information.

Fix two accepted responses $y_1,y_2$ to a prompt $x$, with positive probabilities under the current and reference policies. Define
$$
e(\theta)
=
\log\frac{\pi_\theta(y_1\mid x)}
              {\pi_\theta(y_2\mid x)}
-
\log\frac{\pi_{\mathrm{ref}}(y_1\mid x)}
              {\pi_{\mathrm{ref}}(y_2\mid x)}.
$$
Thus $e=0$ means that the current relative probability matches the reference ratio. Assume $e$ is continuously differentiable on a neighborhood of the trajectory.

\subparagraph{Step 1: Derive the restoring term.}
For fixed $\lambda>0$, gradient ascent on
$$
F_{\mathrm I}=J_R-\frac{\lambda}{2}e^2
$$
adds the correction
$$
u(t)=-\lambda e\,\nabla_\theta e.
$$
Along the corrected flow,
\[
\dot e=b-\lambda ae,
\qquad
b:=\nabla_\theta e^\top g_R,
\qquad
a:=\|\nabla_\theta e\|^2.
\]
The term $b$ is the change in the log ratio induced by
verifier training. The term $-\lambda ae$ opposes deviation
from the reference ratio.

\subparagraph{Step 2: Bound the deviation.}
Consider an interval $[0,T]$ on which the flow exists and both response probabilities remain positive. Suppose
$$
a(t)\geq m>0,
\qquad
|b(t)|\leq B
\quad\text{for }0\leq t\leq T.
$$
An integrating factor gives
$$
e(t)
=
\exp\!\left(-\lambda\int_0^t a(s)\,ds\right)e(0)
+
\int_0^t
\exp\!\left(-\lambda\int_\tau^t a(s)\,ds\right)
b(\tau)\,d\tau.
$$
Using the bounds on $a$ and $b$,
$$
\begin{aligned}
|e(t)|
&\leq
e^{-\lambda mt}|e(0)|
+
B\int_0^t e^{-\lambda m(t-\tau)}\,d\tau\\
&=
e^{-\lambda mt}|e(0)|
+
\frac{B}{\lambda m}
\left(1-e^{-\lambda mt}\right).
\end{aligned}
$$
The bound separates two contributions: one from the initial deviation, which decays, and one from the drift $b$ that the verifier induces. If $b\equiv0$, the deviation decays exponentially. If the drift persists, the bound allows a nonzero deviation to persist as well.

This is a scalar estimate of the form used in input-to-state stability analysis \citep{khalil2002nonlinear}. Here, the conclusion concerns $e$ under the stated bounds. It guarantees neither stability of all policy parameters nor exact preservation of the ratio.

\subparagraph{Limitation: Restoration can increase errors.} Consider one prompt with two responses, $y_1$ correct and $y_2$ incorrect, both accepted. Let
$$
p_G(\theta)=\frac{1}{1+e^{-\theta}},
\qquad
p_H(\theta)=1-p_G(\theta).
$$
For a reference parameter $\theta_{\mathrm{ref}}$, we have $J_R\equiv1$ and $e=\theta-\theta_{\mathrm{ref}}$.
Thus,
$$
\dot\theta=-\lambda e,
\qquad
\dot p_G=-\lambda e\,p_Gp_H,
\qquad
\dot p_H=\lambda e\,p_Gp_H.
$$
If $\theta(0)>\theta_{\mathrm{ref}}$, then $e(t)=e(0)e^{-\lambda t}>0$ at every finite time. Restoration therefore decreases correctness and increases hack probability while reducing the deviation from the reference ratio.

The current policy initially favors the correct response more strongly than the reference does, so restoring the reference reverses that improvement. If initialized at the reference instead, the flow remains stationary and does not strictly reduce hack probability.

\paragraph{Common limit of the three penalties.}
We hold the reference policies, penalty weights, and selected response pairs fixed across correctness assignments and give the penalties no additional correctness information. At the same policy, each penalty and its gradient then agree under every assignment in $\mathcal C_R$. The argument about selective control above, which shows that the records coincide, therefore gives the same controlled trajectory under the compared assignments.

Under the assumptions of Proposition~\ref{prop:verifier_control_limit}, none of these penalties can guarantee $\dot p_H<0$ and $\dot p_G\ge0$ along the trajectory under every assignment in $\mathcal C_R$ whenever $p_H>0$. A penalty may control the departure from a reference, the sensitivity of the reward, or the relative probabilities of paired responses without providing this uniform guarantee of selective control, although it may still succeed on particular tasks.

Additional correctness information or justified structural assumptions can restrict $\mathcal C_R$ and exclude indistinguishable assignments that require incompatible corrections. Excluding these assignments removes the obstruction exhibited here, even without identifying every hack. Removing the obstruction does not by itself give a guarantee, however: a guarantee of selective control also requires showing that the available updates satisfy $\dot p_H<0$ and $\dot p_G\ge0$ throughout the trajectory.

\newpage
\subsection{Technical Details for Section~\ref{sec:selective_control}}

\subsubsection{Projected Correction and Its Immediate Effects}
\label{app:controlled_dynamics} 
Write $h_A=\nabla_\theta p_{H_A}$ and $h=\nabla_\theta p_H$. Since $p=p_H+p_G$, we have $\nabla_\theta p_G=g_R-h$. We first derive the effects of the projected correction without assuming full audit coverage.

\paragraph{Step 1: Preserve instantaneous acceptance growth.}
Define
$$
P_\perp=
\begin{cases}
I_d-\dfrac{g_Rg_R^\top}{\|g_R\|^2},
&g_R\ne0,\\[6pt]
I_d,&g_R=0.
\end{cases}
$$
This orthogonal projector satisfies
$P_\perp^\top=P_\perp$,
$P_\perp^2=P_\perp$, and
$g_R^\top P_\perp=0$.
For the correction $u=-\lambda P_\perp h_A$, with $\lambda>0$, we therefore have
$$
g_R^\top u=0.
$$
Along the corrected flow $\dot\theta=g_R+u$, the chain rule gives
$$
\dot p=g_R^\top(g_R+u)=\|g_R\|^2.
$$
Thus the correction preserves the instantaneous acceptance growth produced by the verifier gradient at the current policy.

\paragraph{Step 2: Oppose growth of audited hacks.} Symmetry and idempotence of the projector give
$$
h_A^\top u
=-\lambda h_A^\top P_\perp h_A
=-\lambda\|P_\perp h_A\|^2.
$$
Consequently,
$$
\dot p_{H_A}
=h_A^\top g_R-\lambda\|P_\perp h_A\|^2.
$$
The correction contributes a strictly negative term when $P_\perp h_A\ne0$. The full rate is negative only when this contribution outweighs any positive growth induced by the verifier gradient.

\paragraph{Step 3: Relate the correction to the full population.}
For the full hack probability, the correction contributes
$$
h^\top u
=-\lambda(P_\perp h)^\top(P_\perp h_A).
$$
Because $g_R^\top u=0$, its contribution to correct responses is equal and opposite:
$$
\nabla_\theta p_G^\top u
=(g_R-h)^\top u
=-h^\top u.
$$
Thus, when the two projected gradients have positive inner product, the correction opposes hack growth and contributes equally to correct-response growth. These statements concern the correction's contribution.
Selectivity of the full update also depends on $g_R$.

If $P_\perp h=0$, every correction orthogonal to $g_R$ has zero instantaneous effect on $p_H$. The projection therefore requires a component of the hack gradient orthogonal to the reward gradient. The theorem below shows how full coverage and sufficient corrective strength turn these local identities into selective control and a bound along the evolving policy.

\subsubsection{Proof of Theorem~\ref{thm:selective_control}}
\label{app:proof_selective_control}
The proof has four steps. Full audit coverage identifies the gradient $h=\nabla p_H$. Projection removes the component of $h$ along $g_R$, so the correction does not slow the growth of acceptance. A strong enough correction then gives $\dot p_H<0$ and $\dot p_G\ge0$. Finally, integrating these rates over $[0,T]$ bounds the hacked share.

\paragraph{Step 1: Full coverage identifies the gradient of $p_H$.}
Under (A2), unaudited hacks have zero probability
along the policy trajectory:
$$
p_H-p_{H_A}=0.
$$
As a function of $\theta$, this difference is nonnegative everywhere, because audited hacks form a subset of all hacks. At every point of the trajectory, it therefore attains its minimum value of zero, and its gradient vanishes:
$$
\nabla p_{H_A}=\nabla p_H.
$$
Thus the audited population gradient coincides with
the full hack gradient along the trajectory. The correction becomes $u=-\lambda P_\perp h$, where $P_\perp$ projects onto the orthogonal complement of $g_R$, and we write $S=\|P_\perp h\|^2$ for the squared norm of the projected gradient.

\paragraph{Step 2: Derive the corrected dynamics.}
The corrected flow is $\dot\theta=g_R+u=g_R-\lambda P_\perp h$. Because $P_\perp$ is an orthogonal projection onto the complement of $g_R$, we have $g_R^\top P_\perp h=0$ and $h^\top P_\perp h=\|P_\perp h\|^2=S$ (Appendix~\ref{app:controlled_dynamics}). Differentiating along the flow therefore gives
$$
\dot p=\|g_R\|^2,
\qquad
\dot p_H=h^\top g_R-\lambda S,
\qquad
\dot p_G=\|g_R\|^2-\dot p_H.
$$
At the current policy, the correction preserves the instantaneous acceptance growth, lowers the rate
$\dot p_H$ by $\lambda S$, and raises the rate
$\dot p_G$ by the same amount.

To compare the growth of the two groups, recall $z=\log(p_H/p_G)$. Without the correction, $z$ changes at the rate
$$
b=(\bar s_H-\bar s_G)^\top g_R,
$$
which we call the pressure that verifier rewards exert on the hacked share. Taking the logarithmic derivative gives
$$
\begin{aligned}
\dot z
&=\frac{\dot p_H}{p_H}
  -\frac{\dot p_G}{p_G}\\
&=b-\lambda S\left(\frac{1}{p_H}
                       +\frac{1}{p_G}\right)\\
&=b-\frac{\lambda S}{p\,q(1-q)}.
\end{aligned}
$$
Since $\dot q=q(1-q)\dot z$, we also obtain
$$
\dot q=q(1-q)b-\frac{\lambda S}{p}.
$$
These expressions separate the verifier's growth pressure from the opposing effect of the correction.

\paragraph{Step 3: Identify when the update is selective.}
The correction lowers $\dot p_H$ by $\lambda S$ and raises $\dot p_G$ by the same amount, so
$$
\dot z=\frac{\dot p_H}{p_H}-\frac{\dot p_G}{p_G}
=b-\lambda S\Bigl(\frac1{p_H}+\frac1{p_G}\Bigr).
$$
If the correction term exceeds $b$, then $\dot z<0$, and hence $\dot q<0$, since $z$ increases with $q$. Differentiating $p_G=p(1-q)$ and using $\dot p=\|g_R\|^2$ then gives
$$
\dot p_G=(1-q)\|g_R\|^2-p\dot q>0.
$$
Because acceptance never decreases under the correction, a falling hacked share implies rising correctness.

A falling hacked share does not require $p_H$ itself to fall, since acceptance may still grow. For this stronger conclusion, the correction must overcome the verifier's contribution to $\dot p_H$:
$$
\lambda S>h^\top g_R.
$$
Then $\dot p_H<0$ and $\dot p_G=\|g_R\|^2-\dot p_H>0$. Because $p_H$ falls while $p_G$ rises, the log ratio $z$ decreases as well, and part~(i) follows.

\paragraph{Step 4: Bound the hacked share during training.}
Because $q$ is the logistic function of $z$, the rate from Step~3 gives
$$
\dot q=q(1-q)\,\dot z=q(1-q)b-\frac{\lambda S}{p},
$$
where the second term uses $p_H=pq$ and $p_G=p(1-q)$. Assumption (A3) bounds the correction from below, $S/p\ge\kappa q$, and the pressure from verifier rewards is bounded above by $b\le D$ with $D\ge0$. Since $q(1-q)\le1/4$, these bounds give
$$
\dot q\le\frac{D}{4}-\lambda\kappa q.
$$
The first term caps the pressure from verifier rewards, and the second is a correction proportional to the current hacked share.

Rearranging gives $\dot q+\lambda\kappa q\le D/4$, and multiplying by the integrating factor $e^{\lambda\kappa t}$ gives
$$
\frac{d}{dt}\left(e^{\lambda\kappa t}q(t)\right)
\leq \frac{D}{4}e^{\lambda\kappa t}.
$$
Integrating from $0$ to $t$ and rearranging yields
$$
q(t)\leq
e^{-\lambda\kappa t}q(0)
+\frac{D}{4\lambda\kappa}
 \left(1-e^{-\lambda\kappa t}\right),
\qquad t\in[0,T].
$$
The initial hacked share contributes a term that decays at rate $\lambda\kappa$, while the pressure from verifier rewards contributes a term that rises toward $D/(4\lambda\kappa)$. Part~(ii) follows, which completes the proof of the theorem.  \hfill$\square$

\subsubsection{Estimating the gradient of $p_H$ from audits}
\label{app:audit_coverage}

We fix the current policy and assume that audits reveal exact correctness labels. The estimator weights each observed hack by the inverse of its audit probability, which compensates for hacks that are generated but not audited.

\paragraph{Step 1: Express the gradient as an expectation.}
With the policy score $s_\theta(x,y)=\nabla_\theta\log\pi_\theta(y\mid x)$, and since $R(1-c)$ indicates a hack, differentiating the probability of hacks gives
$$
h:=\nabla_\theta p_H
=\mathbb E_\theta
 \left[R(X,Y)(1-c(X,Y))s_\theta(X,Y)\right].
$$
Each hack contributes its policy score, and all other responses contribute zero.

\paragraph{Step 2: Weight the observed contributions.}
We draw $n$ independent pairs with $X_i\sim\mathcal D$
and $Y_i\sim\pi_\theta(\cdot\mid X_i)$. We audit each accepted response with known probability $\rho_i=\rho(X_i,Y_i)>0$ and never audit rejected responses. The indicator $I_i$ records whether an audit occurs, and when $I_i=1$, the audit reveals $C_i=c(X_i,Y_i)$. The estimator is
$$
\xi_i=
\begin{cases}
\dfrac{1-C_i}{\rho_i}s_\theta(X_i,Y_i),
&I_i=1,\\[6pt]
0,&I_i=0,
\end{cases}
\qquad
\widehat h=\frac1n\sum_{i=1}^n\xi_i.
$$
An audited hack receives weight $1/\rho_i$, while audited correct responses and unaudited responses contribute zero. The average divides by all $n$ generated responses, including those that were not audited.

\paragraph{Step 3: Show why the weighting works.}
For an accepted response, an audit occurs with probability $\rho_i$, and the weight $1/\rho_i$ cancels this probability in expectation. For a rejected response, both sides are zero. Hence
$$
\mathbb E[\xi_i\mid X_i,Y_i]
=
R(X_i,Y_i)(1-c(X_i,Y_i))s_\theta(X_i,Y_i),
$$
and averaging over the generated responses gives
\begin{equation}
\label{eq:audit_gradient}
\mathbb E[\widehat h]=h,
\end{equation}
where the expectation covers both response generation and audit selection. Audits of only some responses therefore give an unbiased estimate of $h$, provided that every accepted response has a positive audit probability.

\paragraph{Consequence for the projected correction.}
We keep $g_R$ and $P_\perp$ exact and fix $\lambda>0$ before drawing the batch. Every realized correction $\widehat u=-\lambda P_\perp\widehat h$ then satisfies $g_R^\top\widehat u=0$, so it preserves the instantaneous acceptance growth
at the fixed policy. Moreover,
$$
\begin{aligned}
\mathbb E[h^\top(g_R+\widehat u)]
&=h^\top g_R-\lambda h^\top P_\perp
  \mathbb E[\widehat h]\\
&=h^\top g_R-\lambda\|P_\perp h\|^2.
\end{aligned}
$$
At the fixed policy, the correction from a batch therefore changes $p_H$ at the same expected rate as the population correction. A single batch can still produce a direction that does not decrease $p_H$.

Under the stated assumptions, the estimator is unbiased. Small audit probabilities, however, produce large weights $1/\rho_i$ and can increase sampling variability. The argument does not cover accepted responses with zero audit probability or audits with systematic label errors. Appendix~\ref{app:practical_considerations} discusses the practical consequences of both cases.

\newpage
\section{Experimental Details} \label{app:experimental_details}
\subsection{The Gaussian Contextual Bandit Experiment}
\label{app:gaussian_bandit}

\paragraph{Motivation.} To test our theoretical predictions, we use a controlled contextual bandit in which we know the correctness of every response and can compute population gradients exactly.

\paragraph{Setup.}

\subparagraph{Policy.} The policy assigns each response a logit using shared parameters $\theta\in\mathbb R^4$, fixed features $\phi(x,y)\in\mathbb R^4$, and a fixed offset $a(x,y)$:
$$
\pi_\theta(y\mid x)=
\frac{\exp\bigl(\theta^\top\phi(x,y)+a(x,y)\bigr)}
{\sum_{y'\in\mathcal Y_x}
 \exp\bigl(\theta^\top\phi(x,y')+a(x,y')\bigr)}.
$$
Training changes only $\theta$. Sharing parameters allows updates from different prompts to interact.

\subparagraph{Initialization.} We initialize $\theta(0)=0$ and use offsets to set the initial acceptance $p(0)$ and hacked share $q(0)$:
$$
a(x,y)=
\begin{cases}
\log\bigl[p(0)(1-q(0))/16\bigr], & y\in G_x,\\
\log\bigl[p(0)q(0)/16\bigr],     & y\in H_x,\\
\log\bigl[(1-p(0))/16\bigr],     & y\in N_x.
\end{cases}
$$
This gives $p_G(0)=p(0)(1-q(0))$, $p_H(0)=p(0)q(0)$, and $p_N(0)=1-p(0)$, with equal probabilities within each group. Changing the offsets varies these probabilities while keeping the features fixed.

\paragraph{Prompt and Group Distributions.}

\subparagraph{Prompts and labels.} We use eight equally weighted prompts, each with 16 responses in each of three groups. Correct responses $G_x$ have $(R,c)=(1,1)$, hacks $H_x$ have $(R,c)=(1,0)$, and rejected responses $N_x$ have $(R,c)=(0,0)$. Group membership and labels remain fixed during training.

\subparagraph{Accepted features.} We construct features for hacks by centering Gaussian noise and adding the mean $(1,1,0,0)^\top$. We obtain features for correct responses by negating the second coordinate. Writing $\phi_{S,x,j}$ for response $j$ in group $S_x$, we draw
$$
\varepsilon_{x,j}\overset{\mathrm{iid}}{\sim}
\mathcal N(0,0.25^2I_4),
\qquad j=1,\ldots,16,
$$
and set
$$
\begin{aligned}
\phi_{H,x,j}
&=(1,1,0,0)^\top+\varepsilon_{x,j}
  -\frac1{16}\sum_{k=1}^{16}\varepsilon_{x,k},\\
\phi_{G,x,j}
&=\operatorname{diag}(1,-1,1,1)\phi_{H,x,j}.
\end{aligned}
$$
The empirical means are exactly $(1,1,0,0)^\top$ for hacks and $(1,-1,0,0)^\top$ for correct responses in every prompt.

\subparagraph{Rejected features.} We draw eight Gaussian vectors per prompt and append copies with the second coordinate negated:
$$
\xi_{x,j}\overset{\mathrm{iid}}{\sim}
\mathcal N(0,0.25^2I_4),
\qquad
v_{x,j}=\xi_{x,j},
\qquad
v_{x,j+8}=\operatorname{diag}(1,-1,1,1)\xi_{x,j},
$$
for $j=1,\ldots,8$. We center these vectors and add the desired mean $\mu_N=(0,\mu_{N,2},0,0)^\top$:
$$
\phi_{N,x,j}
=v_{x,j}-\frac1{16}\sum_{k=1}^{16}v_{x,k}+\mu_N,
\qquad j=1,\ldots,16.
$$
The empirical mean is therefore exactly $\mu_N$. The experiments use
$\mu_{N,2}\in\{-1.5,-0.5,0,0.5\}$.

We draw the raw noise independently across prompts and between the accepted and rejected groups. Within each seed, we reuse these draws across initializations, rejected means, and learning rates.

\paragraph{Objective.}

\subparagraph{Training objective.} We maximize verifier acceptance:
$$
J_R(\theta)=
\frac18\sum_{x=1}^{8}
\sum_{y\in\mathcal Y_x}\pi_\theta(y\mid x)R(x,y)
=p(\theta).
$$
Correct responses and hacks receive the same reward. Correctness labels define the controlled initializations and evaluation metrics; they do not enter the training objective.

\subparagraph{Training dynamics.} We evaluate the objective and gradient by summing over every prompt and response. We numerically integrate the verifier flow~\eqref{eq:verifier_flow},
$$
\dot\theta=g_R=\nabla_\theta J_R,
$$
using DOP853 until $T=50$. For the learning rate ablation, we instead take discrete updates:
$$
\theta_{k+1}=\theta_k+\eta g_R(\theta_k).
$$
Both procedures use exact population gradients. 

\paragraph{Metrics.}

\subparagraph{Acceptance and correctness.}
We compute each group's probability over all prompts:
$$
p_S(\theta)=
\frac18\sum_{x=1}^{8}\sum_{y\in S_x}\pi_\theta(y\mid x),
\qquad S\in\{G,H,N\}.
$$
We report acceptance $p=p_G+p_H$, hacked share $q=p_H/p$, and correctness $p_G$. We form $q$ after aggregating over prompts. The proxy and true objectives are $J_R=p$ and $J_C=p_G$.

\subparagraph{Log odds and group scores.}
We track $z=\log(p_H/p_G)$ and compute the group scores $\bar s_S=\nabla_\theta\log p_S$ exactly:
$$
\bar s_S=
\frac{1}{8p_S}
\sum_{x=1}^{8}\sum_{y\in S_x}
\pi_\theta(y\mid x)
\left[
\phi(x,y)-
\sum_{y'\in\mathcal Y_x}\pi_\theta(y'\mid x)\phi(x,y')
\right].
$$
We subtract each prompt's mean feature before aggregating. Although the features remain fixed, the group scores change as the policy re-weights responses.

\subparagraph{Hack bias and leakage.} We compute the two contributions in~\eqref{eq:hack_mechanisms}:
$$
\dot z=p(1-p)\left[
\underbrace{
(\bar s_H-\bar s_G)^\top(\bar s_G-\bar s_N)
}_{\text{leakage}}
+
\underbrace{
q\|\bar s_H-\bar s_G\|^2
}_{\text{hack bias}}
\right].
$$
Their sum determines the sign of $\dot z$ and hence the direction of change in the hacked share. We evaluate the full expression within each seed before averaging.

\subparagraph{Growth and reward hacking.} We compute the rates of acceptance, hacked share, and correctness:
$$
\dot p=\|g_R\|^2,
\qquad
\dot q=q(1-q)\dot z,
\qquad
\dot p_G=(1-q)\dot p-p\dot q.
$$
We identify declining correctness using $p\dot q>(1-q)\dot p$ and locate stationary correctness where $\dot p_G=0$. These quantities distinguish growth of the hacked share from a decrease in the probability of correct responses.

\subparagraph{Variation across seeds.} The appendix experiments use ten independent Gaussian feature draws and report means and one sample standard deviation. Each draw defines deterministic training. For the comparison of correctness against acceptance, we interpolate each trajectory onto acceptance values reached by every seed and geometry before computing these statistics.

\subsubsection{When Hacking Grows} \label{app:experiment_hacking}

\paragraph{Controlled comparisons.} We fix $p(0)=2/3$. In the main paper, panel (a) tracks $p$, $q$, and $p_G$ with $q(0)=0.3$ and $\mu_{N,2}=-0.5$. Panel (b) compares $J_C$ against $J_R$ for $q(0)\in\{0.1,0.3,0.5\}$ using the same features.

The centered construction lets us vary initial leakage and hack bias separately. At initialization,
$$
(\bar s_H-\bar s_G)^\top(\bar s_G-\bar s_N)
=-2(1+\mu_{N,2}),
\qquad
q(0)\|\bar s_H-\bar s_G\|^2=4q(0).
$$
Changing the rejected mean changes initial leakage. Changing the initial hacked share changes initial hack bias while preserving the differences between group scores. The following experiments examine how these contributions evolve during training.

\paragraph{Results.}
\subparagraph{Experiment I: hack bias and correctness-to-hack leakage.} Starting from initial acceptance~$p(0)=2/3$, hack share~$q(0)=0.3$, and mean rejected strength~$\mu_{N,2}=-0.5$, we track hack bias, leakage from correctness to hacking, and their sum. For the same trajectories, we plot $\dot z(t)$ and compare it with central time differences of $z$. Together, these experiments relate the competing contributions to the growth rate of hacking through the factor $p(\theta)(1-p(\theta))$ in~\eqref{eq:hack_mechanisms}.

Figure~\ref{fig:growth-appendix}(a) shows that hack bias increases while leakage from correctness to hacking remains near~$-1$. This value matches initialization, where the centered feature means give
$$
(\bar s_H-\bar s_G)^\top(\bar s_G-\bar s_N)=-2(1+\mu_{N,2})=-1.
$$
Leakage stays near this value, consistent with the similar Gaussian spreads across groups: policy reweighting shifts their mean features similarly, leaving the contrasts between groups approximately unchanged. Negative leakage opposes growth of the hacked share, but hack bias outweighs it throughout the plotted interval. Their sum therefore stays positive, and Proposition~\ref{prop:hacking_growth} predicts $\dot z>0$ and an increasing hacked share. The predicted rate agrees with central time differences in Figure~\ref{fig:growth-appendix}~(b). Thus, negative leakage can oppose hacking without overcoming the reinforcement from hack bias.

\subparagraph{Ablation I: varying the mean rejected strength.} We keep $p(0)=2/3$ and $q(0)=0.3$ fixed and vary $\mu_{N,2}\in\{-1.5,-0.5,0.5\}$, holding the accepted features and underlying Gaussian draws fixed. Since initial leakage equals $-2(1+\mu_{N,2})$, these values give initial leakage of $+1$, $-1$, and $-3$: positive, negative but weaker than the initial hack bias, and negative and stronger than it. For each condition, we plot correctness $J_C=p_G$ against verifier reward $J_R=p$. Plotting against reward rather than time compares the conditions at the same acceptance level, which isolates how rejected responses affect correctness as verifier reward rises.

Figure~\ref{fig:growth-appendix}~(c) shows how the mean rejected strength changes the relation between verifier reward and correctness. For $\mu_{N,2}=-1.5$, where initial leakage is positive, correctness declines after a brief initial increase, so reward hacking emerges early. For $\mu_{N,2}=-0.5$, where negative leakage is weaker than hack bias, acceptance and correctness first improve together, but correctness later decreases despite further gains in acceptance. For $\mu_{N,2}=0.5$, where negative leakage is stronger, both quantities increase throughout the plotted range, with no reward hacking evident in the mean curve. These declines illustrate Proposition~\ref{prop:reward_hacking_flow}: reward hacking occurs when $p\dot q>(1-q)\dot p$, that is, when the shift toward hacked responses outweighs the correctness gained from rising acceptance. Thus, even with identical initial acceptance and hacked share, the geometry of the rejected responses can change whether and when reward hacking emerges.

\subparagraph{Ablation II: learning rate in exact gradient ascent.}
To test how closely gradient flow approximates discrete training, we keep $p(0)=2/3$ and $q(0)=1/2$ fixed and vary the learning rate~$\eta$, holding the features and initialization fixed within each seed and geometry. We compare iterate~$k$ of each discrete trajectory with the reference flow at the matched time $t=k\eta$. For each seed, we take the maximum discrepancy in $z$, $p$, $q$, and $p_G$ over times, and we then average these maxima across seeds.

Figure~\ref{fig:growth-appendix}~(d) shows that discrepancies in $z$, $p$, $q$, and $p_G$ increase with the learning rate, reaching the order of $10^{-2}$ at $\eta=0.4$. Because the updates use exact population gradients, these discrepancies reflect only the effect of taking finite steps. As the learning rate decreases, the discrepancies shrink, which supports using gradient flow to approximate discrete training over the tested interval.

\paragraph{Takeaway.} The verifier rewards every accepted response, whether correct or hacked. Which group benefits therefore depends on the feature geometry, including the rejected responses that shape the direction in which acceptance increases. Along this direction, hack bias can overcome negative leakage and shift the accepted population toward hacked responses. Correctness falls once this shift outweighs the correctness gained from rising acceptance. Because discrete gradient ascent with small learning rates closely tracks gradient flow, the flow offers a reliable way to study this competition.

\begin{figure*}[t]
\centering
\begingroup

\pgfplotsset{
  growth appendix axis/.style={
    width=0.38\textwidth,
    height=0.25\textwidth,
    scale only axis,
    tick label style={font=\small},
    xlabel style={font=\small},
    ylabel style={font=\small},
    title style={font=\small},
    every axis plot/.append style={line width=1.6pt},
    legend style={
      font=\small,
      draw=none,
      fill=none,
      row sep=-1pt,
      cells={anchor=west}
    },
    legend cell align=left,
    enlargelimits=false,
    unbounded coords=jump
  }
}

\def\GrowthBand#1#2#3#4#5{%
  \addplot[draw=none,name path=#5lower,forget plot]
    table[
      x=#2,
      y expr={\thisrow{#3_mean}-\thisrow{#3_std}}
    ] {#1};

  \addplot[draw=none,name path=#5upper,forget plot]
    table[
      x=#2,
      y expr={\thisrow{#3_mean}+\thisrow{#3_std}}
    ] {#1};

  \addplot[
    draw=none,fill=#4,fill opacity=0.25,forget plot
  ] fill between[of=#5lower and #5upper];
}

\def\GrowthStepBand#1#2#3{%
  \addplot[draw=none,name path=step#3lower,forget plot]
    table[
      x=eta,
      y expr={
        (\thisrow{#1_max_flow_error_mean}
         -\thisrow{#1_max_flow_error_std}) > 0
        ? \thisrow{#1_max_flow_error_mean}
          -\thisrow{#1_max_flow_error_std}
        : nan
      }
    ] {bandit/step_error_summary.dat};

  \addplot[draw=none,name path=step#3upper,forget plot]
    table[
      x=eta,
      y expr={
        (\thisrow{#1_max_flow_error_mean}
         -\thisrow{#1_max_flow_error_std}) > 0
        ? \thisrow{#1_max_flow_error_mean}
          +\thisrow{#1_max_flow_error_std}
        : nan
      }
    ] {bandit/step_error_summary.dat};

  \addplot[
    draw=none,fill=#2,fill opacity=0.25,forget plot
  ] fill between[of=step#3lower and step#3upper];
}

\begin{tikzpicture}
\begin{groupplot}[
  growth appendix axis,
  group style={
    group size=2 by 2,
    horizontal sep=0.13\textwidth,
    vertical sep=1.65cm
  }
]

\nextgroupplot[
  title={(a)},
  xmin=0, xmax=50,
  xtick={0,25,50},
  xlabel={Training time},
  ylabel={Growth bracket terms},
  legend style={at={(0.97,0.075)},anchor=south east}
]

\GrowthBand{bandit/mechanisms.dat}{t}{leakage}{blue}{leakage}
\GrowthBand{bandit/mechanisms.dat}{t}{bias}{orange}{bias}
\GrowthBand{bandit/mechanisms.dat}{t}{bracket}{black}{bracket}

\addplot[gray,thin,forget plot]
  coordinates {(0,0) (50,0)};

\addplot[blue,solid]
  table[x=t,y=leakage_mean] {bandit/mechanisms.dat};
\addlegendentry{Leakage}

\addplot[orange,dashed]
  table[x=t,y=bias_mean] {bandit/mechanisms.dat};
\addlegendentry{Hack bias}

\addplot[black,dashdotted]
  table[x=t,y=bracket_mean] {bandit/mechanisms.dat};
\addlegendentry{Sum}

\nextgroupplot[
  title={(b)},
  xmin=0, xmax=50,
  xtick={0,25,50},
  xlabel={Training time},
  ylabel={$\dot z$},
  legend style={at={(0.97,0.03)},anchor=south east}
]

\GrowthBand
  {bandit/mechanisms.dat}{t}
  {z_dot_decomposition}{blue}{zdot}

\addplot[gray,thin,forget plot]
  coordinates {(0,0) (50,0)};

\addplot[blue,solid]
  table[x=t,y=z_dot_decomposition_mean]
    {bandit/mechanisms.dat};
\addlegendentry{Predicted}

\addplot[
  black,
  only marks,
  mark=o,
  mark size=1.7pt,
  mark options={solid,fill=white}
]
  table[x=t,y=measured_mean]
    {bandit/z_dot_checks.dat};
\addlegendentry{Time differences}

\nextgroupplot[
  title={(c)},
  xmin=0.66, xmax=1,
  ymin=0, ymax=1,
  xtick={0.7,0.85,1},
  ytick={0,0.5,1},
  xlabel={$J_R=p$},
  ylabel={$J_C=p_G$},
  legend style={at={(0.03,0.03)},anchor=south west}
]

\GrowthBand{bandit/phase_00.dat}{p}{p_G}{blue}{geometryA}
\GrowthBand{bandit/phase_01.dat}{p}{p_G}{orange}{geometryB}
\GrowthBand
  {bandit/phase_02.dat}{p}{p_G}
  {green!60!black}{geometryC}

\addplot[blue,solid]
  table[x=p,y=p_G_mean] {bandit/phase_00.dat};
\addlegendentry{$\mu_{N,2}=-1.5$}

\addplot[orange,dashed]
  table[x=p,y=p_G_mean] {bandit/phase_01.dat};
\addlegendentry{$\mu_{N,2}=-0.5$}

\addplot[green!60!black,dashdotted]
  table[x=p,y=p_G_mean] {bandit/phase_02.dat};
\addlegendentry{$\mu_{N,2}=0.5$}

\nextgroupplot[
  title={(d)},
  xmode=log,
  ymode=log,
  xmin=0.025, xmax=0.4,
  xtick={0.025,0.05,0.1,0.2,0.4},
  xticklabels={$0.025$,$0.05$,$0.1$,$0.2$,$0.4$},
  xlabel={Learning rate $\eta$},
  ylabel={Maximum deviation from flow},
  legend columns=2,
  legend style={at={(0.03,0.97)},anchor=north west}
]

\GrowthStepBand{z}{blue}{z}
\GrowthStepBand{p}{orange}{p}
\GrowthStepBand{q}{green!60!black}{q}
\GrowthStepBand{p_G}{black}{pg}

\addplot[blue,solid,mark=*,mark size=1.7pt]
  table[x=eta,y=z_max_flow_error_mean]
    {bandit/step_error_summary.dat};
\addlegendentry{$z$}

\addplot[orange,dashed,mark=square*,mark size=1.7pt]
  table[x=eta,y=p_max_flow_error_mean]
    {bandit/step_error_summary.dat};
\addlegendentry{$p$}

\addplot[
  green!60!black,dashdotted,
  mark=triangle*,mark size=2pt
]
  table[x=eta,y=q_max_flow_error_mean]
    {bandit/step_error_summary.dat};
\addlegendentry{$q$}

\addplot[black,dotted,mark=diamond*,mark size=2pt]
  table[x=eta,y=p_G_max_flow_error_mean]
    {bandit/step_error_summary.dat};
\addlegendentry{$p_G$}

\end{groupplot}
\end{tikzpicture}
\endgroup

\caption{Growth of hacking in the Gaussian contextual bandit. (a) Hack bias outweighs negative correctness-to-hack leakage, keeping their sum positive. (b) The resulting growth rate $\dot z > 0$ of the log odds agrees with central time differences along the same trajectories. (c) Correctness against acceptance for three values of the mean rejected strength $\mu_{N,2}$. Changing the rejected features changes whether and when correctness declines as verifier reward rises. (d) Discrete gradient ascent approaches the reference flow as the learning rate decreases. We compare iterate $k$ with the flow at $t=k\eta$ and report, for each variable and seed, the maximum absolute discrepancy over evaluated times. Curves show means across ten independent Gaussian feature draws, and shading shows $\pm$ one sample standard deviation. Each draw defines deterministic dynamics with exact population gradients.}
\label{fig:growth-appendix}
\end{figure*}
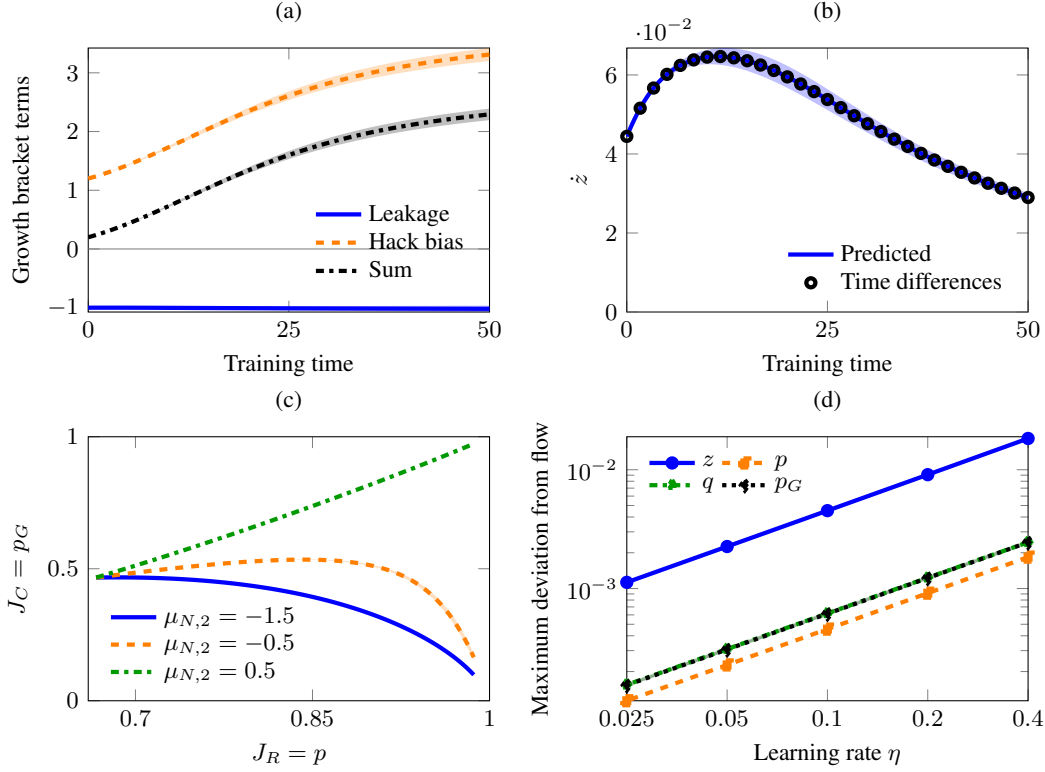

\paragraph{Hyperparameters.}
Table~\ref{tab:hacking-hyperparameters} lists the experimental settings.

\begin{table}[t]
\caption{Settings for the Gaussian contextual-bandit growth experiments.}
\label{tab:hacking-hyperparameters}
\centering
\begin{tabular}{ll}
\hline
\textbf{Parameter} & \textbf{Value} \\
\hline
Number of prompts & $8$ \\
Responses per group per prompt & $16$ \\
Feature dimension & $4$ \\
Gaussian noise standard deviation & $0.25$ \\
Initial parameters & $\theta(0)=0$ \\
Independent feature seeds & $10$ (seeds $0,\ldots,9$) \\
Initial acceptance & $p(0)=2/3$ \\
\hline
Initial hacked share, Experiment I and Ablation I
  & $q(0)=0.3$ \\
Initial hacked share, Ablation II
  & $q(0)=1/2$ \\
Rejected mean, Experiment I
  & $\mu_{N,2}=-0.5$ \\
Rejected means, Ablation I
  & $\mu_{N,2}\in\{-1.5,\,-0.5,\,0.5\}$ \\
Rejected means, Ablation II
  & $\mu_{N,2}\in\{-0.5,\,0,\,0.5\}$ \\
\hline
Flow integrator & DOP853 \\
Maximum integration step & $0.25$ \\
Training horizon & $T=50$ \\
Recorded flow times & $0,\,0.1,\,\ldots,\,50$ \\
Gradient-ascent learning rates $\eta$
  & $\{0.4,\,0.2,\,0.1,\,0.05,\,0.025\}$ \\
\hline
\end{tabular}
\end{table}

\newpage
\subsubsection{Limits of Verifier Feedback}
\label{app:experiment_verifier_feedback}
We study the limits of selective control from verifier feedback (Section~\ref{subsec:limits_of_control}). We evaluate gradient regularization under two compatible correctness assignments along the same sequence of policies. The assignments exchange correct responses and hacks, so reducing hack probability under one reduces correctness under the other. We vary the regularization weight to measure the fraction of recorded times when control is selective under each assignment.

\paragraph{Additional experimental settings.} We use the policy and feature construction from Appendix~\ref{app:gaussian_bandit} to study control from verifier information, as described in Section~\ref{subsec:limits_of_control}.

\subparagraph{Fixed settings.} We set $p(0)=2/3$, $q(0)=1/2$, $\mu_N=(0,-0.5,0,0)^\top$, and $\theta(0)=0$. The offsets are $a(x,y)=\log(1/48)$, giving $p_G(0)=p_H(0)=p_N(0)=1/3$. We use ten independent feature draws and reuse each draw across regularization weights and correctness assignments. The verifier and initial policy remain fixed across comparisons.

\subparagraph{Correctness assignments.} We keep the groups $G,H,N$ fixed and evaluate two assignments: $c_1=\mathbf1_G$ and $c_2=\mathbf1_H=R-c_1$. Under $c_1$, correct responses belong to $G$ and hacks belong to $H$. Under $c_2$, these roles reverse:
$$
p_{G,c_1}=p_G,\qquad p_{H,c_1}=p_H,
\qquad
p_{G,c_2}=p_H,\qquad p_{H,c_2}=p_G.
$$
Both assignments use the same verifier. Correctness labels enter only the evaluation.

For panel Figure~\ref{fig:bandit-growth-four-panel}, Panel (c) in the main paper, we also evaluate $c=R$, under which $p_{G,c}=p$ and $p_{H,c}=0$, along the same verifier flow used for $c_1$.

\subparagraph{Exact derivatives.} Using the group scores computed above, we obtain
$$
\nabla p_S=p_S\bar s_S,\qquad S\in\{G,H,N\},
\qquad
g_R=\nabla p_G+\nabla p_H.
$$
We compute the reward gradient by summing over all prompts and responses.

\subparagraph{Controller.} Gradient regularization penalizes the squared norm of the reward gradient:
$$
F_{\mathrm{GR}}(\theta)
=J_R(\theta)-\gamma\|g_R(\theta)\|^2,
\qquad
B_R(\theta)=\nabla_\theta^2J_R(\theta).
$$
We follow its gradient by adding a correction to the verifier flow:
$$
u(t)=-2\gamma B_R(\theta(t))g_R(\theta(t)),
\qquad
\dot\theta=g_R+u(t)=\nabla_\theta F_{\mathrm{GR}}.
$$
We recompute the derivatives at the current policy and keep $\gamma$ fixed within each run. The main comparison uses $\gamma=16$; the ablation uses $\gamma\in\{0,1,4,16\}$. Setting $\gamma=0$ recovers the verifier flow. Because the update uses no correctness labels, both assignments give the same policy $\pi_{\theta(t)}$ at every time.

\subparagraph{Numerical integration.}
We integrate to $T=1$ using classical fourth order Runge--Kutta
with step $0.01$ and record the policy at every step.
We check integration accuracy by repeating each run with step $0.005$.
Table~\ref{tab:verifier-hyperparameters} lists the numerical settings.

\subparagraph{Selective control.} We evaluate changes in hack probability and correctness using the full parameter update:
$$
\dot p_{H,c}=\nabla p_{H,c}^{\top}(g_R+u),
\qquad
\dot p_{G,c}=\nabla p_{G,c}^{\top}(g_R+u).
$$
For numerical classification, we require $\dot p_{H,c}<-\epsilon$ and $\dot p_{G,c}\geq-\epsilon$, with $\epsilon=10^{-8}$. We verify the exchange of rates across assignments:
$$
\dot p_{H,c_1}=\dot p_{G,c_2},
\qquad
\dot p_{G,c_1}=\dot p_{H,c_2}.
$$
For each seed, we compute the fraction of recorded times satisfying selectivity. We report means and sample standard deviations across seeds for both the rates and these fractions.

\paragraph{Results.}

\subparagraph{Experiment I: verifier-only control.} We run gradient regularization with $\gamma=16$ and evaluate the resulting trajectory under $c_1$ and $c_2$. The controller uses only verifier information, so changing the correctness assignment leaves the trajectory unchanged. We plot $\dot p_{G,c}$ and $\dot p_{H,c}$ and identify intervals where $\dot p_{H,c}<0$ and $\dot p_{G,c}\geq0$. 

Figure~\ref{fig:verifier-control-appendix}(a) shows that the mean rate $\dot p_{G,c_1}=\dot p_{H,c_2}$ decreases but remains positive, while $\dot p_{H,c_1}=\dot p_{G,c_2}$ rises from negative to positive. Initially, the controller therefore reduces mean hack probability and increases mean correctness under $c_1$, with opposite effects under $c_2$. Later, both mean rates become positive, so neither assignment shows a reduction in mean hack probability. The exchange of rates in~\eqref{eq:incompatible_selective_rates} prevents selective control under both assignments at the same time, illustrating Proposition~\ref{prop:verifier_control_limit}.

\subparagraph{Ablation: gradient regularization weight.} We vary $\gamma\in\{0,1,4,16\}$ while keeping the features and initial policy fixed within each seed. For each weight, we evaluate both correctness assignments along the same sequence of policies. We compute the fraction of recorded times satisfying selective control under each assignment and report the mean and sample standard deviation across ten seeds. At every evaluated time, we also check that the policy update is never selective under both assignments.

Figure~\ref{fig:verifier-control-appendix}~(b) shows that the mean fraction of selective updates under $c_1$ is zero for $\gamma\in\{0,1,4\}$ and below $25\%$ for $\gamma=16$. Under $c_2$, the fraction is $100\%$ for $\gamma=0$, above $50\%$ for $\gamma=1$, and zero for $\gamma\in\{4,16\}$. Without regularization, training increases the probability of $H$ and decreases that of $G$ throughout the recorded interval. This is selective under $c_2$, which labels $H$ as correct and $G$ as hacks, but has the opposite effect under $c_1$. Increasing regularization to $\gamma=16$ introduces a period of selectivity under $c_1$ while eliminating selectivity under $c_2$. These results illustrate the information limit in Proposition~\ref{prop:verifier_control_limit}: a controller using only RLVR observations cannot guarantee selective control across compatible correctness assignments.

\paragraph{Takeaway.} Gradient regularization changes which accepted responses lose probability, but verifier observations do not reveal whether those responses are hacks or correct. The same reduction therefore removes hacks under one compatible assignment and correct responses under the other. Tuning the regularization weight changes which assignment benefits without resolving this ambiguity. Guaranteeing selective control requires additional information that distinguishes the correctness assignments.

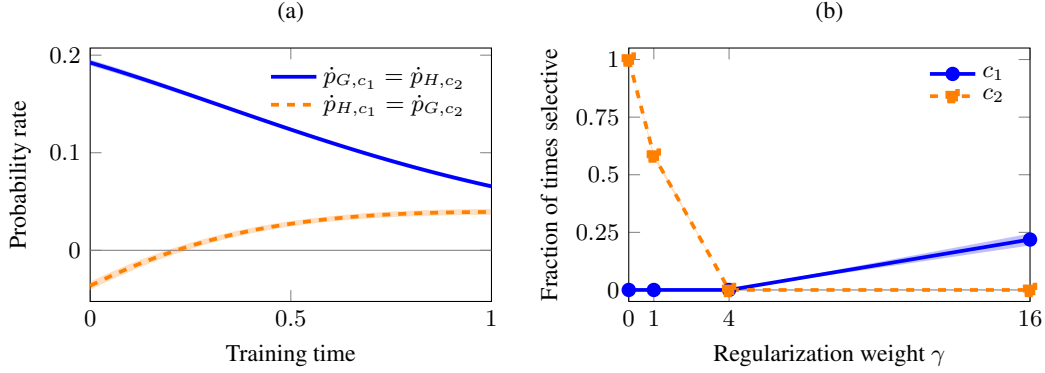
\begin{figure*}[t]
\centering
\begingroup

\def\VerifierDir{bandit}

\pgfplotsset{
  verifier axis/.style={
    width=0.38\textwidth,
    height=0.24\textwidth,
    scale only axis,
    tick label style={font=\small},
    xlabel style={font=\small},
    ylabel style={font=\small},
    title style={font=\small},
    every axis plot/.append style={line width=1.4pt},
    legend style={
      font=\small,
      draw=none,
      fill=none,
      row sep=-1pt,
      cells={anchor=west}
    },
    legend cell align=left,
    enlarge x limits=false,
    enlarge y limits=0.05
  }
}

\begin{tikzpicture}
\begin{groupplot}[
  verifier axis,
  group style={
    group size=2 by 1,
    horizontal sep=0.13\textwidth
  }
]

\nextgroupplot[
  title={(a)},
  xmin=0, xmax=1,
  xtick={0,0.5,1},
  xlabel={Training time},
  ylabel={Probability rate},
  legend style={
    at={(0.97,0.97)},
    anchor=north east
  }
]

\addplot[
  draw=none,
  fill=blue,
  fill opacity=0.25,
  forget plot
]
  table[x=x,y=y]
  {\VerifierDir/rate_G_band.dat}
  \closedcycle;

\addplot[
  draw=none,
  fill=orange,
  fill opacity=0.25,
  forget plot
]
  table[x=x,y=y]
  {\VerifierDir/rate_H_band.dat}
  \closedcycle;

\addplot[gray,thin,forget plot]
  coordinates {(0,0) (1,0)};

\addplot[blue,solid]
  table[x=t,y=r_G_c1_mean]
  {\VerifierDir/control_rates.dat};
\addlegendentry{$\dot p_{G,c_1}=\dot p_{H,c_2}$}

\addplot[orange,dashed]
  table[x=t,y=r_H_c1_mean]
  {\VerifierDir/control_rates.dat};
\addlegendentry{$\dot p_{H,c_1}=\dot p_{G,c_2}$}

\nextgroupplot[
  title={(b)},
  xmin=0, xmax=16,
  xtick={0,1,4,16},
  ytick={0,0.25,0.5,0.75,1},
  xlabel={Regularization weight $\gamma$},
  ylabel={Fraction of times selective},
  legend style={
    at={(0.97,0.97)},
    anchor=north east
  }
]

\addplot[
  draw=none,
  fill=blue,
  fill opacity=0.25,
  forget plot
]
  table[x=x,y=y]
  {\VerifierDir/fraction_c1_band.dat}
  \closedcycle;

\addplot[
  draw=none,
  fill=orange,
  fill opacity=0.25,
  forget plot
]
  table[x=x,y=y]
  {\VerifierDir/fraction_c2_band.dat}
  \closedcycle;

\addplot[gray,thin,forget plot]
  coordinates {(0,0) (16,0)};

\addplot[
  blue,
  solid,
  mark=*,
  mark size=2pt
]
  table[x=gamma,y=fraction_c1_mean]
  {\VerifierDir/selectivity_fraction.dat};
\addlegendentry{$c_1$}

\addplot[
  orange,
  dashed,
  mark=square*,
  mark size=2pt
]
  table[x=gamma,y=fraction_c2_mean]
  {\VerifierDir/selectivity_fraction.dat};
\addlegendentry{$c_2$}

\end{groupplot}
\end{tikzpicture}
\endgroup
\caption{
Control using only RLVR training observations in the Gaussian contextual bandit. (a) Rates of change in correctness and hack probability under gradient regularization with $\gamma=16$. The assignments $c_1,c_2$ exchange these rates along the same sequence of policies. (b) Fraction of recorded times satisfying selective control for each regularization weight $\gamma$: $\dot p_{H,c}<-\epsilon$ and $\dot p_{G,c}\geq-\epsilon$, with $\epsilon=10^{-8}$. We compute each fraction within a seed before averaging. Lines show means across ten independent Gaussian feature draws; shading shows $\pm$ one sample standard deviation.
}
\label{fig:verifier-control-appendix}
\end{figure*}

\paragraph{Hyperparameters.}
Table~\ref{tab:verifier-hyperparameters} lists the additional and changed settings for this experiment. All other settings follow Table~\ref{tab:hacking-hyperparameters}.

\begin{table}[h!]
\caption{Additional settings for the verifier-only control experiments
(Appendix~\ref{app:experiment_verifier_feedback}).
The main experiment uses $\gamma=16$; the ablation varies its weight.
Both correctness assignments share the same controlled trajectory.
All expectations are evaluated exactly. Other model settings follow
Table~\ref{tab:hacking-hyperparameters}.}
\label{tab:verifier-hyperparameters}
\centering
\begin{tabular}{ll}
\hline
\textbf{Parameter} & \textbf{Value} \\
\hline
Rejected-group second-coordinate mean & $-0.5$ \\
Initial hacked share $q(0)$ & $1/2$ \\
Initial acceptance $p(0)$ & $2/3$ \\
Initial parameters & $\theta(0)=0$ \\
Independent feature seeds & $10$ (seeds $0,\ldots,9$) \\
Main regularization weight $\gamma$ & $16$ \\
Regularization-weight sweep & $\{0,\,1,\,4,\,16\}$ \\
\hline
Flow integrator & Classical fourth-order Runge--Kutta \\
Training horizon & $T=1$ \\
Recorded times & $0,\,0.01,\,\ldots,\,1$ \\
Tolerance for selectivity $\epsilon$ & $10^{-8}$ \\
\hline
\end{tabular}
\end{table}

\newpage
\subsubsection{Selective Control with Correctness Feedback}
\label{app:experiment_selective_control}
We studyprojected audit corrections, examine the effects of audit coverage and projection error, and illustrate the ISS bound in Theorem~\ref{thm:selective_control}.

\paragraph{Additional experimental settings.}

\subparagraph{Features and initialization.} We use the policy and Gaussian construction from Appendix~\ref{app:gaussian_bandit}, changing the accepted feature means to $(4,-1,0,0)^\top$ for $G$ and $(4,1,0,0)^\top$ for $H$. The rejected mean remains $(0,-0.5,0,0)^\top$. These means give $\nabla p_G(0)^\top\nabla p_H(0)=29/324>0$, so audit correction without projection (raw) initially opposes correctness. We fix $c=c_1=\mathbf1_G$, $p(0)=2/3$, $q(0)=1/2$, and $\theta(0)=0$. Thus, every offset equals $\log(1/48)$ and $p_G(0)=p_H(0)=p_N(0)=1/3$. We reuse features and initialization across methods within each seed.

\subparagraph{Audits and exact gradients.} A fixed audit set $A_\mathrm{aud}$ reveals correctness labels for selected accepted responses. Its audited hacks are $H_A=A_\mathrm{aud}\cap H$. We compute their probability gradient exactly:
$$
\nabla p_{H_A}
=
\frac18\sum_x\sum_{y:(x,y)\in H_A}
\pi_\theta(y\mid x)
\left[
\phi(x,y)-\sum_{y'}\pi_\theta(y'\mid x)\phi(x,y')
\right].
$$
We use full coverage except in the coverage ablation. Under full coverage, $p_{H_A}=p_H$ and their gradients coincide. Labels outside the audit set enter only the evaluation.

\subparagraph{Corrections.}
Each method follows $\dot\theta=g_R+u$, with
$$
u=
\begin{cases}
0,
& \text{verifier flow},\\
-2\gamma B_Rg_R,
& \text{gradient regularization},\\
-\lambda\nabla p_{H_A},
& \text{raw audit correction},\\
-\lambda P_\perp\nabla p_{H_A},
& \text{projected audit correction}.
\end{cases}
$$
Here $B_R=\nabla^2J_R$ and $P_\perp=I-g_Rg_R^\top/\|g_R\|^2$ when $g_R\ne0$, with $P_\perp=I$ otherwise. We fix $\lambda=6$ and $\gamma=16$ and recompute derivatives at the current policy.

\paragraph{Additional metrics.}

\subparagraph{Selective control.} We retain $p_G,p_H,p$, and $q$ from Appendix~\ref{app:gaussian_bandit}. We evaluate each method using the full update:
$$
\dot p_H=\nabla p_H^\top(g_R+u),
\qquad
\dot p_G=\nabla p_G^\top(g_R+u).
$$
For numerical classification, we require $\dot p_H<-\epsilon$ and $\dot p_G\geq-\epsilon$, with $\epsilon=10^{-8}$. In the coverage ablation, we also record
$(P_\perp\nabla p_H)^\top(P_\perp\nabla p_{H_A})$. A positive value means that the correction opposes growth of total hack probability.

\subparagraph{Projection error.} We construct the correction $\widehat u$ using an estimated reward gradient $\widehat g_R$ inside the projector. The verifier direction $g_R$ and hack gradient $\nabla p_H$ remain exact. We measure the resulting error in acceptance growth:
$$
\left|\dot p-\|g_R\|^2\right|
=
\left|g_R^\top\widehat u\right|
=
\left|(g_R-\widehat g_R)^\top\widehat u\right|.
$$

\subparagraph{Numerical ISS envelope.} For each seed, we estimate $D$ from the maximum of $(\bar s_H-\bar s_G)^\top g_R$ and zero along the trajectory. We estimate $\kappa$ from the minimum of $\|P_\perp\nabla p_H\|^2/p_H$. We refine the integration and increase the evaluation grid from $1{,}001$ to $2{,}001$ times. We insert the constants $D$ and $\kappa$ into Theorem~\ref{thm:selective_control} and compare its envelope with $q(t)$. These numerical estimates do not certify the bound between evaluated times.

\paragraph{Results}

\subparagraph{Experiment I: selective control.} Under $c=c_1$ and full audit coverage, we compare verifier flow, gradient regularization, raw audit correction, and projected audit
correction. We compute exact gradients and keep the features and initial policy fixed across methods. Both audit corrections use the same fixed gain $\lambda$. We plot correctness $p_G$ against hack probability $p_H$ and identify intervals of selective control, where $\dot p_H<0$ and $\dot p_G\geq0$.

Figure~\ref{fig:audit-appendix}~(a) shows that projected audit correction increases correctness $p_G$ and decreases hack probability $p_H$ throughout the recorded interval. Raw audit correction initially decreases both probabilities. Later, correctness increases while hack probability remains nearly constant. Both verifier flow and gradient regularization increase hack probability. These mean trajectories show that projection enables sustained hack reduction without the initial loss of correctness observed under raw audit correction.

\subparagraph{Experiment II: ISS bound.} We apply projected audit correction with full coverage and constant gain and compare $q(t)$ with the envelope in Theorem~\ref{thm:selective_control}. For each feature seed, we estimate $D$ and $\kappa$ along the computed trajectory, then check these estimates using a finer time grid and tighter integration tolerances. We construct the envelope separately for each seed and record its gap from $q(t)$. This comparison illustrates the bound numerically over the tested interval.

Figure~\ref{fig:audit-appendix}~(b) shows that the hacked share $q(t)$ decreases throughout the recorded interval. It equals the ISS envelope at initialization and remains below it afterward, providing a numerical illustration of Theorem~\ref{thm:selective_control} along the tested trajectories.

\subparagraph{Ablation I: audit coverage.} We vary the fraction of candidates audited in each accepted group and prompt over $\{0,1/8,1/4,1/2,3/4,1\}$. We use the known groups to construct nested audit sets and keep their candidate indices fixed across seeds and throughout training. The controller receives correctness labels only for audited responses. For each fraction, we apply projected audit correction with $\lambda=6$ until $T=1$, using the exact gradient of the audited hack probability $p_{H_A}$. We compare total hack probability $p_H(T)$ across initial coverage levels and record the alignment between the projected gradients of $p_H$ and $p_{H_A}$.

Figure~\ref{fig:audit-appendix}~(c) shows that increasing audit coverage lowers the final hack probability. Without audits, the correction vanishes and verifier flow increases $p_H$. With sufficient coverage, the correction reduces $p_H$ below its initial value. Although the correction uses only audited responses, we evaluate whether it reduces the total probability of hacks.

\subparagraph{Ablation II: projection error.} At $\theta=0$ and full audit coverage, we estimate the reward gradient using batches of $n\in\{32,128,512,2048\}$ independent prompts and responses. We average $R(x,y)\nabla_\theta\log\pi_\theta(y\mid x)$ within each batch and use this estimate only to construct the projection. The verifier direction $g_R$ and hack gradient $\nabla p_H$ remain exact, and the correction gain is $\lambda=6$.
Without advancing the policy, we measure $|\dot p-\|g_R\|^2|=|g_R^\top\widehat u|$. For each batch size and feature seed, we average this error over $1{,}000$ independent batches, then report the mean and sample standard deviation across ten feature seeds.

Figure~\ref{fig:audit-appendix}~(d) shows that larger batches reduce the mean error in preserving acceptance growth. The exact projection gives zero error. Estimating the projection introduces the discrepancy $\dot p-\|g_R\|^2=(g_R-\widehat g_R)^\top\widehat u$. Thus, more samples reduce the error in the acceptance growth rate. Panel~(a) checks whether the correction also reduces hack probability.

\paragraph{Takeaway.} Audits reveal which accepted responses are wrong, but suppressing those responses can also suppress correct ones because they share policy parameters. Projection preserves the instantaneous growth rate of acceptance. When the corrected update reduces hack probability, correctness must therefore increase. The ISS bound describes how sustained correction limits the hacked share despite pressure toward hacks from verifier training. In these experiments, broader audit coverage improves hack reduction, while larger sample batches reduce projection error. Selective control reduces hack probability without reducing correctness. Projection preserves the  instantaneous growth rate of acceptance, so any decrease in hack probability under the corrected flow must increase correctness.


\begin{figure*}[t]
\centering
\begingroup
\pgfplotsset{
  audit axis/.style={
    width=0.36\textwidth,
    height=0.24\textwidth,
    scale only axis,
    tick label style={font=\small},
    xlabel style={font=\small},
    ylabel style={font=\small},
    title style={font=\small},
    legend style={
      font=\small, draw=none, fill=none,
      cells={anchor=west}
    },
    legend cell align=left,
    enlargelimits=false
  }
}

\begin{tikzpicture}
\begin{groupplot}[
  audit axis,
  group style={
    group size=2 by 2,
    horizontal sep=0.14\textwidth,
    vertical sep=1.5cm
  }
]

\nextgroupplot[
  title={(a)},
  xmin=0, xmax=0.65,
  ymin=0, ymax=1,
  xlabel={$p_H$},
  ylabel={$p_G$},
  legend style={at={(0.97,0.97)},anchor=north east}
]
\addplot[draw=none,fill=black,fill opacity=0.18,forget plot]
  table[x=x,y=y] {bandit/phase_rlvr_band.dat} \closedcycle;
\addplot[draw=none,fill=orange,fill opacity=0.22,forget plot]
  table[x=x,y=y] {bandit/phase_gr_band.dat} \closedcycle;
\addplot[draw=none,fill=green!60!black,fill opacity=0.22,forget plot]
  table[x=x,y=y] {bandit/phase_raw_band.dat} \closedcycle;
\addplot[draw=none,fill=blue,fill opacity=0.22,forget plot]
  table[x=x,y=y] {bandit/phase_projected_band.dat} \closedcycle;

\addplot[black,line width=1.3pt,
  postaction={decorate},
  decoration={markings,mark=at position 0.55 with {\arrow{stealth}}}]
  table[x=p_H_mean,y=p_G_mean] {bandit/phase_rlvr.dat};
\addlegendentry{RLVR}

\addplot[orange,dotted,line width=1.3pt,
  postaction={decorate},
  decoration={markings,mark=at position 0.55 with {\arrow{stealth}}}]
  table[x=p_H_mean,y=p_G_mean] {bandit/phase_gr.dat};
\addlegendentry{Grad. reg.}

\addplot[green!60!black,dashed,line width=1.3pt,
  postaction={decorate},
  decoration={markings,mark=at position 0.55 with {\arrow{stealth}}}]
  table[x=p_H_mean,y=p_G_mean] {bandit/phase_raw.dat};
\addlegendentry{Raw audit}

\addplot[blue,line width=1.3pt,
  postaction={decorate},
  decoration={markings,mark=at position 0.55 with {\arrow{stealth}}}]
  table[x=p_H_mean,y=p_G_mean] {bandit/phase_projected.dat};
\addlegendentry{Projected}

\nextgroupplot[
  title={(b)},
  xmin=0, xmax=1,
  ymin=0, ymax=0.55,
  xlabel={Training time},
  ylabel={$q$},
  legend style={at={(0.97,0.97)},anchor=north east}
]
\addplot[draw=none,fill=blue,fill opacity=0.24,forget plot]
  table[x=x,y=y] {bandit/q_band.dat} \closedcycle;
\addplot[draw=none,fill=black,fill opacity=0.14,forget plot]
  table[x=x,y=y] {bandit/envelope_band.dat} \closedcycle;

\addplot[blue,line width=1.3pt]
  table[x=t,y=q_mean] {bandit/iss.dat};
\addlegendentry{Hacked share}

\addplot[black,dashed,line width=1.3pt]
  table[x=t,y=envelope_mean] {bandit/iss.dat};
\addlegendentry{Envelope}

\nextgroupplot[
  title={(c)},
  xmin=0, xmax=1,
  ymin=0, ymax=0.6,
  xlabel={Initial audit coverage},
  ylabel={$p_H(T)$},
  legend style={at={(0.97,0.97)},anchor=north east}
]
\addplot[draw=none,fill=blue,fill opacity=0.3,forget plot]
  table[x=x,y=y] {bandit/coverage_band.dat} \closedcycle;

\addplot[gray,dashed,line width=0.9pt]
  coordinates {(0,0.333333) (1,0.333333)};
\addlegendentry{$p_H(0)$}

\addplot[blue,line width=1.3pt,mark=*,mark size=1.5pt]
  table[x=coverage,y=p_H_final_mean]
  {bandit/coverage.dat};
\addlegendentry{Audited control}

\nextgroupplot[
  title={(d)},
  xmode=log,
  log basis x=2,
  xmin=32, xmax=2048,
  ymin=-0.02,
  xtick={32,128,512,2048},
  xlabel={Samples for $\widehat g_R$},
  ylabel={$|\dot p-\|g_R\|^2|$},
  legend style={at={(0.97,0.97)},anchor=north east}
]
\addplot[draw=none,fill=blue,fill opacity=0.3,forget plot]
  table[x=x,y=y]
  {bandit/projection_error_band.dat} \closedcycle;

\addplot[blue,line width=1.3pt,mark=*,mark size=1.7pt]
  table[x=batch_size,y=abs_error_mean]
  {bandit/projection_error.dat};
\addlegendentry{Estimated projection}

\addplot[black,dashed,line width=0.9pt]
  coordinates {(32,0) (2048,0)};
\addlegendentry{Exact-gradient reference}

\end{groupplot}
\end{tikzpicture}
\endgroup
\caption{Selective control with correctness feedback in the Gaussian bandit. (a) Correctness $p_G$ versus hack probability $p_H$; arrows indicate
training direction. (b) Hacked share $q(t)$ and the numerical ISS envelope. (c) Final hack probability versus the initial fraction of hacks audited; the dashed line marks the initial hack probability. (d) Error $|\dot p-\|g_R\|^2|$ when estimating the reward gradient used for projection. Lines show means across ten feature seeds; shading shows $\pm$ one sample standard deviation. The envelope in (b) provides a numerical check and does not certify the bound between evaluated times.
}
\label{fig:audit-appendix}
\end{figure*}
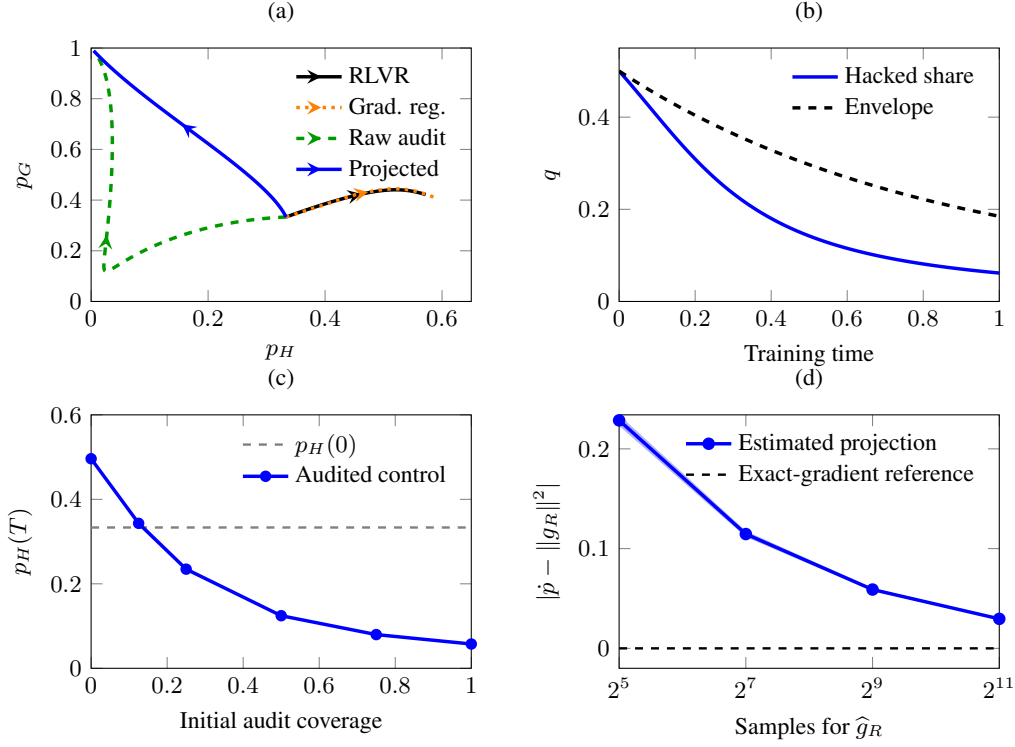

\paragraph{Hyperparameters.}
Table~\ref{tab:selective-hyperparameters} lists only added or changed settings for this experiment. All other applicable settings follow Table~\ref{tab:hacking-hyperparameters}. 

\begin{table}[h!]
\caption{Additional and changed settings for selective control with correctness under $c_1$. Other model settings follow Table~\ref{tab:hacking-hyperparameters}.}
\label{tab:selective-hyperparameters}
\centering
\begin{tabular}{ll}
\hline
\textbf{Parameter} & \textbf{Value} \\
\hline
Accepted-group feature means
  & $(4,-1,0,0)$, $(4,1,0,0)$ \\
Rejected-group second-coordinate mean & $-0.5$ \\
Initial hacked share $q(0)$ & $1/2$ \\
Initial acceptance $p(0)$ & $2/3$ \\
Initial parameters; warmup & $\theta(0)=0$; none \\
Independent feature seeds & $10$ (seeds $0,\ldots,9$) \\
Audit correction gain $\lambda$ & $6$ \\
Gradient-regularization weight $\gamma$ & $16$ \\
Audit labels & Exact \\
\hline
Initial audited fraction $p_{H_A}(0)/p_H(0)$
  & $\{0,\,1/8,\,1/4,\,1/2,\,3/4,\,1\}$ \\
Projection batch sizes $n$ & $\{32,\,128,\,512,\,2048\}$ \\
Policy for projection-error ablation & $\theta=0$ \\
\hline
Flow integrator & DOP853 \\
Horizon, control comparison & $T=10$ \\
Horizon, coverage and ISS & $T=1$ \\
Tolerance for selectivity $\epsilon$ & $10^{-8}$ \\
\hline
\end{tabular}
\end{table}
\newpage 

\subsection{The Neural Contextual Bandit Experiment}
\label{app:neural_bandit}

\paragraph{Motivation.} We replace the log-linear policy in Appendix~\ref{app:gaussian_bandit} with a neural policy to examine reward hacking (Section~\ref{sec:hacking}) and the limits of verifier feedback (Section~\ref{sec:verifier_feedback}) when the policy learns its representation. We also examine how audit coverage and projection accuracy affect projected audit correction (Section~\ref{sec:selective_control}).

\paragraph{Experimental setting.} We use a shared MLP $f_\theta$ with four inputs, one hidden layer of 16 tanh units, and a scalar output. We train all weights and hidden biases and omit the output bias. The policy is
$$
\pi_\theta(y\mid x)\propto
\exp\!\left(
a(x,y)+f_\theta(\phi(x,y))
-f_{\theta(0)}(\phi(x,y))
\right).
$$
We freeze the subtracted network, so the offsets $a(x,y)$ (see Appendix~\ref{app:gaussian_bandit}) determine the initial group probabilities.

\subparagraph{Data.}  We use eight equally weighted prompts, 16 responses per group, and four-dimensional Gaussian features with noise standard deviation $0.25$. We follow the centering and reflection construction in Appendix~\ref{app:gaussian_bandit}, fixing the rejected group mean to $(0,-0.5,0,0)^\top$. The shared first-coordinate mean of the accepted groups is $1$ in Experiments~I--II and $4$ in the ablations. Within each experiment, comparisons share the Gaussian draws and network initialization for each seed. The network receives response features without group labels or a separate prompt embedding.

Under the correctness assignment $c_1$, the groups $G_x,H_x,N_x$ have labels $(R,c_1)=(1,1),(1,0),(0,0)$, respectively. Experiment~II also evaluates $c=R$ and $c_2=R-c_1$ while keeping these groups and the training trajectory fixed.

\subparagraph{RLVR training.} Experiments~I--II follow verifier flow~\eqref{eq:verifier_flow}, $\dot\theta=g_R$, where $g_R=\nabla_\theta J_R$ and $J_R=p$. We compute probabilities and gradients by summing over all prompts and responses. We integrate the flow using DOP853, recomputing the gradients at each policy. The coverage ablation uses the same integrator for the corrected flow. Table~\ref{tab:neural-bandit-hyperparameters} lists other numerical settings.

\subparagraph{Audits.} The coverage ablation uses fixed, nested audit sets $A_\mathrm{aud}$ under $c_1$. Each audit reveals the correctness of an accepted response. With $H_A=A_\mathrm{aud}\cap H$, we apply
$$
\dot\theta=g_R-\lambda P_\perp\nabla_\theta p_{H_A}.
$$
Both gradients are exact. The projector $P_\perp$ acts in the space of neural parameters, orthogonally to $g_R$, and equals $I$ when $g_R=0$.

The projection ablation evaluates corrections at the initial policy without advancing it. We estimate $g_R$ from independent draws $x\sim\mathcal D$ and $y\sim\pi_\theta(\cdot\mid x)$ by averaging $R(x,y)\nabla_\theta\log\pi_\theta(y\mid x)$. Only the projector uses this estimate. The verifier direction $g_R$ and the full hack gradient $\nabla_\theta p_H$ remain exact.

\subparagraph{Metrics.} We report acceptance $p$, correctness $p_G$, hack probability $p_H$, and hacked share $q=p_H/p$. We average group probabilities over prompts before forming this ratio. Experiment~I compares $J_C=p_G$ with $J_R=p$ along training. In Experiment~II, the subscript $c$ identifies the correctness assignment defining $p_{G,c}$ and $p_{H,c}$. The coverage ablation reports $p_H(T)$ against initial coverage $p_{H_A}(0)/p_H(0)$. The projection ablation reports the deviation $|\dot p-\|g_R\|^2|$ from the acceptance growth rate preserved by the exact projector.

We report means and one sample standard deviation across ten independent seeds. For projection error, we first average absolute deviations over independent batches within each seed.

\paragraph{Results.}

\subparagraph{Experiment I: reward hacking.} We study whether increasing acceptance can amplify hacks and reduce correctness with a neural policy. Under verifier flow, we vary the hack share $q(0)\in\{0.1,0.3,0.5\}$ while fixing initial acceptance $p(0)=2/3$ and the response features. For $q(0)=0.3$, we track acceptance $p$, hacked share $q$, and correctness $p_G$ over time. We also compare the proxy objective $J_R=p$ with the true objective $J_C=p_G$ to examine reward hacking as defined by~\citep{NEURIPS2022_3d719fee}. RLVR training supplies the policies required by this definition: two times $t_1<t_2$ exhibit reward hacking when $J_R(\theta(t_2))>J_R(\theta(t_1))$ but $J_C(\theta(t_2))<J_C(\theta(t_1))$. Thus, the comparison reveals whether improving verifier acceptance comes at the expense of correctness.

Figure~\ref{fig:neural-bandit-appendix} (a) shows that acceptance $p$ and hacked share $q$ increase together from $q(0)=0.3$. As training progresses, correctness $p_G=p(1-q)$ eventually decreases despite increasing acceptance. The shift toward hacks then outweighs the gain from accepting more responses (Section~\ref{subsec:acceptance_hack_together}.) 

Figure~\ref{fig:neural-bandit-appendix} (b) shows how the proxy objective $J_R=p$ and the true objective $J_C=p_G$ change along training. The mean curves show reward hacking throughout the recorded interval for $q(0)=0.5$: acceptance increases while correctness decreases. For $q(0)=0.3$, acceptance and correctness initially increase together. Reward hacking emerges later, when correctness decreases despite further gains in acceptance. For $q(0)=0.1$, both quantities increase throughout the recorded interval, with no reward hacking evident in the mean curve. Hacks are present at initialization in all three cases, but reward hacking occurs when improving acceptance reduces correctness.

\subparagraph{Experiment II: compatible correctness assignments.}
We examine whether the same neural training record can support different conclusions about correctness. Starting from $q(0)=0.3$, we evaluate each verifier trajectory under the correctness assignments: $c=R$, $c_1=\mathbf{1}_G$, and $c_2=R-c_1$. The features, offsets, initialization, and verifier are identical across these assignments. Their correctness probabilities are $p,p_G,p_H$, respectively, and their hack probabilities are $0,p_H,p_G$. Thus, comparing $c=R$ with $c_1$ tests whether the record reveals the presence of hacks, while comparing $c_1$ with $c_2$ tests whether it identifies which accepted responses are wrong (Section~\ref{sec:verifier_feedback}).

Figure~\ref{fig:neural-bandit-appendix} (c) shows identical acceptance~$p$ under three compatible correctness assignments. Correctness increases under $c=R$ and $c_2=R-c_1$, whereas it eventually decreases under $c_1$. All three rules share the same training record $L_t$, so verifier observations cannot determine which correctness trajectory applies. This illustrates the limits of detection and identification in Propositions~\ref{prop:verifier_information_limit} and~\ref{prop:verifier_identification_limit}. Because $c_1$ and $c_2$ exchange correct responses and hacks, the comparison also illustrates why these observations cannot guarantee selective control under every compatible assignment (Proposition~\ref{prop:verifier_control_limit}).

\subparagraph{Ablation I: audit coverage.} We examine how much of the hack population a fixed \emph{audit} set must expose for the correction to reduce total hack probability. We audit fractions $\{0,1/8,1/4,1/2,3/4,1\}$ of the candidates in each accepted group and prompt. The sets are nested, with candidate indices fixed across seeds and throughout training. The evaluator's groups serve only to control coverage. Each condition uses the exact gradient $\nabla p_{H_A}$ and the same correction gain, without inverse probability weighting. We compare $p_H(T)$  across initial coverage levels and track how the audited and full hack gradients align. Zero coverage recovers verifier flow, while full coverage recovers the ideal projected correction. Partial coverage can also oppose hacking when the projected gradients align positively, as described in Appendix~\ref{app:practical_considerations}.

Figure~\ref{fig:neural-bandit-appendix}~(d) shows that the final hack probability $p_H(T)$ remains large when the initial audited fraction $p_{H_A}(0)/p_H(0)$ is small. Increasing this fraction reduces $p_H(T)$ to near zero in the tested setting. The correction uses only the audited hacks $H_A=A_\mathrm{aud}\cap H$, while the plot measures the full hack population $H$. Thus, the comparison shows how expanding audit coverage improves suppression beyond the audited subset (see Appendix~\ref{app:practical_considerations}).

\subparagraph{Ablation II: projection error.} We test how estimating the reward gradient affects the projection's preservation of acceptance growth. At the initial policy, we draw independent prompts $x\sim\mathcal D$ and responses $y\sim\pi_\theta(\cdot\mid x)$. We estimate $g_R$ by averaging $R(x,y)\nabla_\theta\log\pi_\theta(y\mid x)$ and use this estimate to construct the projection.  The verifier direction $g_R$ and hack gradient $\nabla p_H$
remain exact, isolating the effect of estimating the projector. Without advancing the policy, we measure $|\dot p-\|g_R\|^2|=|g_R^\top\widehat u|$. For each batch size, we average absolute discrepancies over independent batches within each seed, then report their mean and standard deviation across seeds.

Figure~\ref{fig:neural-bandit-appendix} (e) shows that larger sample batches reduce the deviation $|\dot p-\|g_R\|^2|$ caused by estimating the projection. Larger batches improve preservation of the acceptance growth rate, as predicted by $\dot p-\|g_R\|^2=(g_R-\widehat g_R)^\top\widehat u$. Selective control additionally requires the correction to overcome the growth of hacks, as demonstrated in Figure~\ref{fig:bandit-growth-four-panel} (d). 

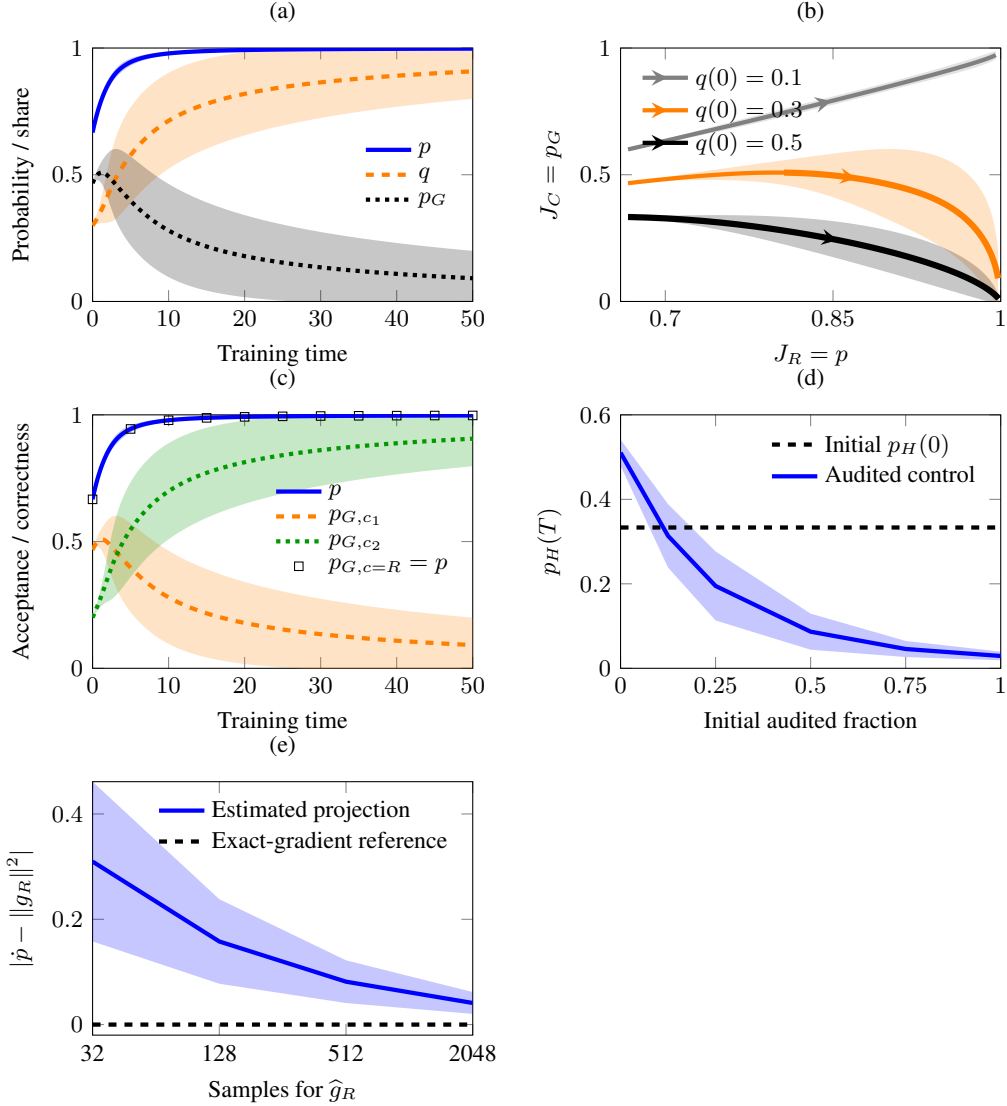
\begin{figure*}[t]
\centering
\begingroup
\def\NeuralDataDir{neural_bandit}
\def\NeuralLineWidth{1.6pt}
\def\NeuralDeclineWidth{2.34pt}

\pgfplotsset{
  neural appendix axis/.style={
    width=0.36\textwidth,
    height=0.24\textwidth,
    scale only axis,
    tick label style={font=\small},
    xlabel style={font=\small},
    ylabel style={font=\small},
    title style={font=\small},
    legend cell align=left,
    legend style={
      at={(0.97,0.97)},anchor=north east,
      font=\small,draw=none,fill=none,
      cells={anchor=west}
    },
    enlargelimits=false
  },
  neural mean/.style={line width=\NeuralLineWidth,no marks},
  neural arrow/.style={
    postaction={decorate},
    decoration={markings,
      mark=at position 0.55 with {\arrow{stealth}}}
  }
}

\newcommand{\NeuralBand}[2]{%
  \addplot[draw=none,fill=#1,fill opacity=0.22,forget plot]
    table[x=x,y=y] {\NeuralDataDir/#2_band.dat} \closedcycle;
}
\newcommand{\NeuralCurve}[3]{%
  \addplot[neural mean,#1]
    table[x=x,y=y] {\NeuralDataDir/#2.dat};
  \addlegendentry{#3}
}
\newcommand{\NeuralDecline}[2]{%
  \addplot[#1,line width=\NeuralDeclineWidth,forget plot,
    unbounded coords=jump]
    table[x=x,y=y] {\NeuralDataDir/#2_declining.dat};
}

\begin{tikzpicture}
\begin{groupplot}[
  neural appendix axis,
  group style={group size=2 by 3,
    horizontal sep=0.14\textwidth,
    vertical sep=1.5cm}
]

\nextgroupplot[
  title={(a)},
  xlabel={Training time},ylabel={Probability / share},
  xmin=0,ymin=0,ymax=1,ytick={0,0.5,1},
  legend style={at={(0.97,0.50)},anchor=east}
]
\NeuralBand{blue}{growth_A_00}
\NeuralBand{orange}{growth_A_01}
\NeuralBand{black}{growth_A_02}
\NeuralCurve{blue}{growth_A_00}{$p$}
\NeuralCurve{orange,dashed}{growth_A_01}{$q$}
\NeuralCurve{black,dotted}{growth_A_02}{$p_G$}

\nextgroupplot[
  title={(b)},
  xlabel={$J_R=p$},ylabel={$J_C=p_G$},
  xmin=0.66,xmax=1,ymin=0,ymax=1,
  xtick={0.7,0.85,1},ytick={0,0.5,1},
  legend style={at={(0.03,0.97)},anchor=north west}
]
\NeuralBand{gray}{growth_B_00}
\NeuralBand{orange}{growth_B_01}
\NeuralBand{black}{growth_B_02}
\NeuralCurve{gray,neural arrow}{growth_B_00}{$q(0)=0.1$}
\NeuralCurve{orange,neural arrow}{growth_B_01}{$q(0)=0.3$}
\NeuralCurve{black,neural arrow}{growth_B_02}{$q(0)=0.5$}
\NeuralDecline{gray}{growth_B_00}
\NeuralDecline{orange}{growth_B_01}
\NeuralDecline{black}{growth_B_02}

\nextgroupplot[
  title={(c)},
  xlabel={Training time},ylabel={Acceptance / correctness},
  xmin=0,ymin=0,ymax=1,ytick={0,0.5,1},
  legend style={at={(0.97,0.55)},anchor=east}
]
\NeuralBand{blue}{information_C_00}
\NeuralBand{orange}{information_C_02}
\NeuralBand{green!60!black}{information_C_03}
\NeuralCurve{blue}{information_C_00}{$p$}
\NeuralCurve{orange,dashed}{information_C_02}{$p_{G,c_1}$}
\NeuralCurve{green!60!black,dotted}{information_C_03}{$p_{G,c_2}$}
\addplot[black,only marks,mark=square,mark size=1.6pt,
  mark options={fill=none}]
  table[x=x,y=y] {\NeuralDataDir/information_C_01.dat};
\addlegendentry{$p_{G,c=R}=p$}

\nextgroupplot[
  title={(d)},
  xlabel={Initial audited fraction},ylabel={$p_H(T)$},
  xmin=0,xmax=1,ymin=0,ymax=0.6,
  xtick={0,0.25,0.5,0.75,1}
]
\NeuralBand{blue}{audit_coverage_00}
\NeuralCurve{black,dashed}{audit_coverage_01}{Initial $p_H(0)$}
\NeuralCurve{blue,mark=*,mark size=1.5pt}{audit_coverage_00}{Audited control}

\nextgroupplot[
  title={(e)},
  xlabel={Samples for $\widehat g_R$},
  ylabel={$|\dot p-\|g_R\|^2|$},
  xmode=log,log basis x=2,
  xmin=32,xmax=2048,
  xtick={32,128,512,2048},
  xticklabels={32,128,512,2048},
  ymin=-0.02
]
\NeuralBand{blue}{projection_error_00}
\NeuralCurve{blue,mark=*,mark size=1.7pt}{projection_error_00}{Estimated projection}
\NeuralCurve{black,dashed}{projection_error_01}{Exact-gradient reference}

\nextgroupplot[group/empty plot]
\end{groupplot}
\end{tikzpicture}
\endgroup
\caption{Neural contextual bandit experiments.
(a)~Acceptance $p$ and hacked share $q$ increase together from~$q(0)=0.3$, while correctness~$p_G$ eventually decreases. (b)~Correctness $J_C=p_G$ versus acceptance $J_R=p$ for $q(0)\in\{0.1,0.3,0.5\}$. Increasing acceptance accompanied by decreasing correctness exhibits reward hacking. (c)~The same training record yields different correctness trajectories under $c=R$, $c_1$, and $c_2=R-c_1$. Square markers show correctness under $c=R$, which equals acceptance. (d)~Final hack probability $p_H(T)$ versus the initial audited fraction $p_{H_A}(0)/p_H(0)$ for fixed audit sets. Greater coverage reduces the probability of hacks. (e)~Deviation $|\dot p-\|g_R\|^2|$ versus the number of samples used to estimate the gradient defining the projector. The verifier direction and audit gradient remain exact. Larger batches reduce the deviation. The exact projector preserves $\dot p=\|g_R\|^2$. Curves show means across ten seeds; shading shows one standard deviation. In (b), vertical bands measure variability in correctness at matched training times.}
\label{fig:neural-bandit-appendix}
\end{figure*}

\paragraph{Takeaway.} Rising verifier reward can conceal declining performance on the intended task, even for a neural policy that learns its own representation. The decline stays hidden because the information available during training cannot separate correct responses from accepted errors. Audits supply this missing information. An intervention based on audits works only as well as their coverage, corrective strength, and projection accuracy allow.

\paragraph{Hyperparameters.}
Table~\ref{tab:neural-bandit-hyperparameters} lists the settings. Data construction follows Appendix~\ref{app:gaussian_bandit}. All studies use ten seeds, each determining the Gaussian features and an independent network initialization. The flows start without warmup and use DOP853. 

\begin{table}[h!]
\caption{Settings for the neural bandit experiments and ablations.}
\label{tab:neural-bandit-hyperparameters}
\begin{center}
\begin{tabular}{ll}
\multicolumn{1}{c}{\bf Parameter} &
\multicolumn{1}{c}{\bf Value}\\
\hline \\[-1ex]
Network widths; activation
  & $4$--$16$--$1$; tanh \\
Input / output weight distributions
  & $\mathcal N(0,1/4)$ / $\mathcal N(0,1/16)$ \\
Initial hidden biases; output bias
  & Zero; omitted \\
Independent seeds
  & $10$ ($0,\ldots,9$) \\
Prompts; responses per group
  & $8$; $16$ \\
Gaussian feature standard deviation
  & $0.25$ \\
Initial acceptance $p(0)$
  & $2/3$ \\
Initial hacked share, Experiment I
  & $\{0.1,0.3,0.5\}$ \\
Initial hacked share, time plot and Experiment II
  & $0.3$ \\
Initial hacked share, both ablations
  & $1/2$ \\
Accepted first-coordinate mean, I--II / ablations
  & $1$ / $4$ \\
Rejected feature mean $\mu_N$
  & $(0,-0.5,0,0)^\top$ \\
\hline \\[-1ex]
Correction gain $\lambda$
  & $6$ \\
Fractions of candidates audited
  & $\{0,1/8,1/4,1/2,3/4,1\}$ \\
Seed for audit sets
  & $2026$ \\
Audit labels
  & Exact \\
Batch sizes for estimating $g_R$
  & $\{32,128,512,2048\}$ \\
Independent batches per size and seed
  & $1{,}000$ \\
Sampling seed stream
  & $(2027,\text{feature seed})$ \\
\hline \\[-1ex]
Integrator
  & DOP853 \\
Horizon, Experiments I--II / coverage
  & $50$ / $1$ \\
Recording interval, Experiments I--II / coverage
  & $0.1$ / $0.01$ \\
\hline
\end{tabular}
\end{center}
\end{table}
\newpage 

\subsection{The Language Model Experiment} \label{app:llm_experiment}

\paragraph{Motivation.} To test whether reward hacking and audit correction behave as our theory predicts beyond exact gradient flow, we train an autoregressive language model with sampled gradients and finite updates. Because the task below provides correctness labels, we can separate gains in verifier reward from gains in correctness.

\paragraph{Setup.}

\subparagraph{Task.} Each prompt $x$ specifies a rule table $r:\{1,2\}\to\{1,2\}^2$ and six input digits $u_1,\ldots,u_6$. The model replaces each digit with its image under $r$ and concatenates the results:
$$
t(x)=r(u_1)\cdots r(u_6)\in\{1,2\}^{12}.
$$
The prompt asks for each replacement and then for the final answer $t(x)$. Each replacement depends only on its own input digit, so no state carries across positions.

\subparagraph{Correctness and verifier labels.} An output is valid if it contains exactly one \texttt{Final answer:} field, on its last nonempty line, with 12 digits from $\{1,2\}$. For a valid response $y$ with parsed answer $\widehat t$, the correctness and verifier labels are
$$
c(x,y)=\mathbf1\{\widehat t=t(x)\},\qquad
R(x,y)=\mathbf1\{\widehat t_{11:12}=t(x)_{11:12}\}.
$$
Both labels are zero for invalid outputs. Correctness checks every replacement in the final answer, while the imperfect verifier checks only the last pair, and neither checks the intermediate replacements. Because $\widehat t=t(x)$ implies that the last pairs match, $c\le R$: the verifier has no false negatives for this definition of correctness. The labels therefore partition responses into three groups:
$$
G_x=\{y:c=1\},\qquad H_x=\{y:R=1,c=0\},\qquad N_x=\{y:R=0\}.
$$

\subparagraph{Hint and demonstrations.} Every prompt contains a hint with the correct final pair and an incorrect prefix. We draw the prefix uniformly from the $2^{10}-1$ incorrect prefixes and keep it fixed for that prompt. Because the final pair is correct, copying the hint in the required format produces a \emph{hack}. Supervised demonstrations cover all three groups: they give the correct solution ($G$), copy the hint ($H$), or change the last digit of the correct answer ($N$). Rejected demonstrations keep valid formatting, so their rejection comes from the wrong final pair rather than from a formatting failure.

\subparagraph{Policy.} We use \texttt{Qwen/Qwen2-0.5B} with LoRA. Training updates only the adapter parameters $\theta$ and keeps the base model fixed. All methods train this single policy without KL regularization.

\subparagraph{Initialization.} We first train on correct solutions with supervised fine-tuning (SFT) and then add demonstrations that copy the hint. From this common checkpoint, we run 80 further SFT updates with demonstrations from $G$, $H$, and $N$. We sample these groups with probabilities $(0.4,0.4,0.2)$ or $(0.3,0.3,0.4)$ and name the two mixtures N20 and N40 after their share of rejected demonstrations. These probabilities describe the demonstration data, not the group probabilities of the resulting policy. For each mixture, all methods start from the same SFT checkpoint, and we run reinforcement learning with five seeds.

\subparagraph{Data split.} The 16 rule tables and 64 inputs give 1,024 tasks. We assign 608 tasks to training, 208 to calibration, and 208 to testing, so that no pair of rule table and input appears in more than one partition. Rule tables recur across partitions, so testing evaluates unseen combinations of seen rules and inputs.

\subparagraph{Training objective.} The verifier objective is
$J_R(\theta)=\mathbb E_{x\sim D_{\mathrm{train}},\,y\sim\pi_\theta(\cdot\mid x)}[R(x,y)]$,
where $D_{\mathrm{train}}$ is uniform over training tasks. Because $R$ checks only the last pair, correct responses and hacks receive the same reward. Correctness labels enter training only through the audits in PAC.

\subparagraph{GRPO updates.} Each round samples $K=2$ training prompts and $B=8$ responses per prompt (group). For response $y_{ji}$ to prompt $x_j$, the score $s_{ji}=\nabla_\theta\log\pi_\theta(y_{ji}\mid x_j)$ sums the token scores, including the end token when emitted. With rewards $R_{ji}=R(x_j,y_{ji})$, each group normalizes its advantages by the standard deviation of its binary rewards:
$$
\bar R_j=\frac1B\sum_i R_{ji},\qquad
a_{ji}=\frac{R_{ji}-\bar R_j}
{\sqrt{\bar R_j(1-\bar R_j)}+10^{-4}}.
$$
Treating the advantages as constants, we estimate the gradient by
$$
\widehat v_{\mathrm{GRPO}}=\frac1{KBM}\sum_{j=1}^K\sum_{i=1}^B a_{ji}s_{ji},
\qquad M=192,
$$
where the normalizer $M$ is a constant, independent of response length. Each fresh batch supplies at most one AdamW step. A group whose rewards are all equal has zero advantages, and a batch in which every advantage is zero skips its step. Skipped steps still count toward the 20 rounds of every run.

\paragraph{Metrics.}

\subparagraph{Acceptance and correctness.} We report correctness, acceptance, and the hacked share:
$$
J_C=p_G,\qquad J_R=p=p_G+p_H,\qquad q=p_H/p,
$$
where $p=p_G+p_H$ because the verifier has no false negatives. For $n$ sampled responses, we estimate $p_G=n_G/n$, $p_H=n_H/n$, and $p=(n_G+n_H)/n$. We estimate $q=n_H/(n_G+n_H)$ from counts pooled across prompts and leave it undefined when no response is accepted. A rising $q$ indicates growth of hacking, and a rising $J_R$ with a falling $J_C$ is reward hacking in the sense of Proposition~\ref{prop:reward_hacking_flow}.

\subparagraph{Evaluation and variation across seeds.} At rounds 0, 5, 10, and 20, we evaluate the policy on eight fixed calibration prompts with four responses each. Final evaluations use 32 fixed test prompts, also with four responses each, so calibration trajectories and final test results come from different prompts. We sample at temperature $1$ from the full token distribution and record separately the responses with format errors and those that reach the generation limit. Figure~\ref{fig:llm-appendix} shows means over five training seeds. Table~\ref{tab:llm_appendix_outcomes} reports means $\pm$ one sample standard deviation across these five seeds.

\subsubsection{When Hacking Grows} \label{app:llm_growth}
\paragraph{Controlled comparisons.} We train GRPO from both SFT checkpoints with the same settings, so each run samples 320 responses: 20 rounds of $K=2$ prompts with $B=8$ responses each. We evaluate the final checkpoint of each run.

\paragraph{Results.}

\subparagraph{Experiment I: reward hacking under GRPO.} In N40, from the SFT checkpoint to the final round, mean test acceptance rises from $69.5\%$ to $99.4\%$, while correctness falls from $30.5\%$ to $2.2\%$ and the hacked share rises from $56.2\%$ to $97.8\%$. All five seeds show the same three changes. Rising acceptance with falling correctness is reward hacking in the sense of Proposition~\ref{prop:reward_hacking_flow}, even though sampled updates with AdamW do not satisfy its assumption of exact gradient flow. Table~\ref{tab:llm_appendix_outcomes} summarizes the final test outcomes.

\subparagraph{Ablation I: SFT mixture.} Figure~\ref{fig:llm-appendix}~(a,b) shows the calibration trajectories for N20. The N20 mixture produces the same pattern as N40: acceptance rises from $80.5\%$ to $98.8\%$, correctness falls from $31.2\%$ to $0.5\%$, and the hacked share rises from $61.2\%$ to $99.5\%$. Under GRPO, both SFT checkpoints therefore lead to growth of hacking and to reward hacking. Because changing the SFT mixture also changes the learned gradient geometry, this comparison does not isolate the effect of the initial composition.

\paragraph{Takeaway.} For both SFT mixtures, GRPO increases verifier reward while reducing correctness, so reward hacking persists beyond exact gradient flow.

\subsubsection{Selective Control with Correctness Feedback}
\label{app:llm_control}

\paragraph{Additional experimental settings.}

\subparagraph{Audits and gradient estimates.} We audit each accepted response independently with probability $\rho\in\{1,0.5,0.25\}$. The indicator $Z_{ji}\sim\operatorname{Bernoulli}(\rho)$ marks whether response $i$ to prompt $j$ is audited, and reweighting by $1/\rho$ gives an unbiased estimate of whether the response is a hack:
$$
\widetilde H_{ji}=\frac{Z_{ji}R_{ji}(1-c_{ji})}{\rho}.
$$
This estimate is zero for unaudited and rejected responses, so it requires correctness labels only for audited accepted responses. Let $s_{ji}=\nabla_\theta\log\pi_\theta(y_{ji}\mid x_j)$ be the score of response $i$ to prompt $j$. With baselines $\bar R_{j,-i}$ and $\overline{\widetilde H}_{j,-i}$ that average the other $B-1$ responses to prompt $j$, we use the same $KB=16$ responses as GRPO to estimate
$$
\begin{aligned}
\widehat g&=\frac1{KB}\sum_{j,i}(R_{ji}-\bar R_{j,-i})s_{ji},\\
\widehat h&=\frac1{KB}\sum_{j,i}
(\widetilde H_{ji}-\overline{\widetilde H}_{j,-i})s_{ji}.
\end{aligned}
$$
At a fixed policy, $\widehat g$ and $\widehat h$ are unbiased estimates of $g=\nabla p$ and $h=\nabla p_H$, so $\widehat g-\widehat h$ estimates $\nabla p_G$. The estimate $\widehat g$ uses the rewards of all responses, while $\widehat h$ uses only the correctness labels revealed by audits. The exact labels used for reporting never enter training.

\subparagraph{Corrections.} Both corrections subtract a direction $v$ from the GRPO step. Raw audit correction uses $v_{\mathrm{raw}}=\widehat h$, the estimated gradient of $p_H$. PAC first removes the component of $\widehat h$ along the estimated acceptance gradient, so that the correction leaves acceptance unchanged to first order:
$$
v_{\mathrm{PAC}}=\widehat h-
\frac{\widehat g^\top\widehat h}{\|\widehat g\|^2}\widehat g.
$$
The projection uses the Euclidean inner product on adapter parameters, and when $\widehat g=0$, PAC reduces to raw audit correction. We compute both estimates before the AdamW step. 

We normalize both correction directions and scale them using the same rule based on the optimizer step. With $d$ the actual AdamW displacement, $m$ the number of adapter parameters, and $\eta=10^{-5}$, each method applies
$$
b=\max\{\|d\|,0.25\eta\sqrt m\},\qquad
\Delta\theta=d-b\frac{v}{\|v\|}.
$$
For a nonzero direction, the effective gain $\lambda=b/\|v\|$ varies across updates. We skip the correction when $\|v\|$ is numerically zero. Because $b\ge 0.25\eta\sqrt m$, a nonzero correction can still act when every GRPO advantage is zero and AdamW skips its step. Both methods use the same number of sampled responses, but their audit counts can differ because their acceptance differs.

\paragraph{Results.}
\subparagraph{Experiment I: selective correction.} Table~\ref{tab:llm_appendix_outcomes} reports test correctness $p_G$ and the probability $p_H$ of hacks at the SFT checkpoint and after PAC with full auditing, averaged over five seeds for each SFT mixture. In all ten runs, PAC increases correctness and decreases hacks. In N40, for example, mean correctness rises from $30.5\%$ to $97.2\%$, and mean $p_H$ falls from $39.1\%$ to $1.7\%$. These are the directions that Theorem~\ref{thm:selective_control} predicts, even though PAC uses sampled gradients and finite updates.

\subparagraph{Ablation I: removing projection.} Figure~\ref{fig:llm-appendix}(a)--(c) compares the calibration trajectories of both corrections and GRPO for N20 and N40. In N40, raw audit correction also reaches high final correctness: $98.4\%$ on test prompts, compared with $97.2\%$ for PAC (Figure~\ref{fig:llm-appendix}(f), $\rho=1$). PAC's advantage appears earlier in training. In N40, mean calibration correctness at rounds 5 and 10 is $85.0\%$ and $98.1\%$ for PAC, compared with $65.6\%$ and $90.0\%$ for raw audit correction. Under raw audit correction, acceptance falls below its initial value at round 5 in four of five seeds and later recovers, while correctness shows no such decline at the recorded rounds. PAC therefore reaches high correctness sooner with the same number of sampled responses. It does not end with higher correctness, and because the two methods audit different numbers of responses, this comparison does not hold audit cost fixed.

\subparagraph{Ablation II: audit coverage.} Figure~\ref{fig:llm-appendix}~(d,e) shows the calibration trajectories under partial auditing. Panel~(f) relates final test correctness to audit count. For N40, we reduce $\rho$ to $0.5$ and $0.25$ and compare with the runs at full auditing $\rho=1$. At $\rho=0.25$, PAC uses 72.0 audits per run on average, compared with 292.6 at full auditing, a $75.4\%$ reduction. With this coverage, PAC still reaches $95.2\%$ test correctness, and hacks make up $4.4\%$ of responses. At all three audit probabilities, both raw audit correction and PAC increase correctness and reduce hacks relative to the SFT checkpoint in every seed. Final correctness is not monotone in $\rho$. This ablation reduces the number of correctness labels, while every run still samples 320 responses. Because each accepted response is audited independently, every accepted response can receive an audit, and no subset is permanently excluded.

\paragraph{Takeaway.} Audit corrections avoid the reward hacking that GRPO exhibits: from the same SFT checkpoints, they increase correctness and reduce hacks. Projection raises correctness earlier in training, while raw audit correction reaches similar final correctness. Auditing one quarter of accepted responses retains about $97\%$ of the gain in final correctness that full auditing achieves.

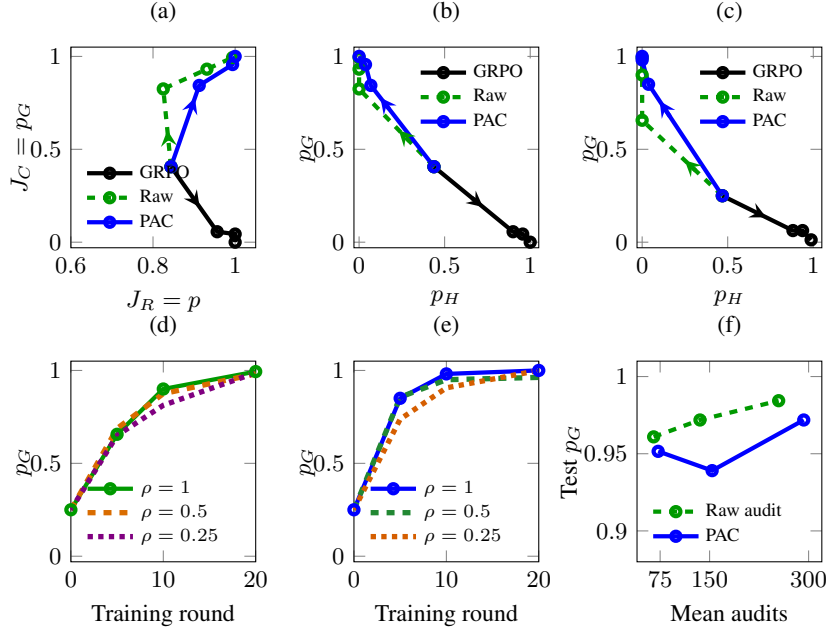
\begin{figure*}[t]
\centering
\begingroup

\def\LLMDataPath{LLM}
\definecolor{llmGreen}{HTML}{228833}
\definecolor{llmVermilion}{HTML}{D55E00}

\begin{tikzpicture}
\begin{groupplot}[
  group style={
    group size=3 by 2,
    horizontal sep=1.30cm,
    vertical sep=1.5cm
  },
  width=0.175\textwidth,
  height=0.19\textwidth,
  scale only axis,
  grid=none,
  tick label style={font=\small},
  xlabel style={font=\small},
  ylabel style={
    font=\small,
    at={(axis description cs:-0.16,0.5)},
    anchor=south
  },
  title style={font=\small},
  legend style={
    font=\scriptsize,
    draw=none,
    fill=none,
    row sep=-1pt,
    cells={anchor=west},
    at={(0.03,0.03)},
    anchor=south west
  },
  legend cell align=left,
  enlargelimits=false,
  ymin=-0.03,
  ymax=1.05,
  ytick={0,0.5,1}
]

\nextgroupplot[
  title={(a)},
  xmin=0.60,xmax=1.05,
  xtick={0.6,0.8,1},
  xlabel={$J_R=p$},
  ylabel={$J_C=p_G$}
]

\addplot[black,solid,line width=1.6pt,
  mark=o,mark size=1.7pt,mark options={solid,fill=white}]
  table[x expr=\thisrow{J_R}/100,y expr=\thisrow{J_C}/100]
  {\LLMDataPath/appendix/N20_grpo.dat};
\addlegendentry{GRPO}

\addplot[green!60!black,dashed,line width=1.6pt,
  mark=o,mark size=1.7pt,mark options={solid,fill=white}]
  table[x expr=\thisrow{J_R}/100,y expr=\thisrow{J_C}/100]
  {\LLMDataPath/appendix/N20_raw.dat};
\addlegendentry{Raw}

\addplot[blue,solid,line width=1.6pt,
  mark=o,mark size=1.7pt,mark options={solid,fill=white}]
  table[x expr=\thisrow{J_R}/100,y expr=\thisrow{J_C}/100]
  {\LLMDataPath/appendix/N20_pac.dat};
\addlegendentry{PAC}

\draw[black,line width=1.6pt,->,>=stealth]
  (axis cs:0.897750,0.238250)
  -- (axis cs:0.913500,0.189250);
\draw[green!60!black,line width=1.6pt,->,>=stealth]
  (axis cs:0.8375625,0.5444375)
  -- (axis cs:0.8349375,0.6030625);
\draw[blue,line width=1.6pt,->,>=stealth]
  (axis cs:0.8925625,0.716875)
  -- (axis cs:0.9021875,0.778125);

\nextgroupplot[
  title={(b)},
  xmin=-0.03,xmax=1.05,
  xtick={0,0.5,1},
  xlabel={$p_H$},
  ylabel={$p_G$},
  legend style={
  at={(0.99,0.97)},
  anchor=north east
}
]
\addplot[black,solid,line width=1.6pt,
  mark=o,mark size=1.7pt,mark options={solid,fill=white}]
  table[x expr=\thisrow{p_H}/100,y expr=\thisrow{J_C}/100]
  {\LLMDataPath/appendix/N20_grpo.dat};
\addlegendentry{GRPO}

\addplot[green!60!black,dashed,line width=1.6pt,
  mark=o,mark size=1.7pt,mark options={solid,fill=white}]
  table[x expr=\thisrow{p_H}/100,y expr=\thisrow{J_C}/100]
  {\LLMDataPath/appendix/N20_raw.dat};
\addlegendentry{Raw}

\addplot[blue,solid,line width=1.6pt,
  mark=o,mark size=1.7pt,mark options={solid,fill=white}]
  table[x expr=\thisrow{p_H}/100,y expr=\thisrow{J_C}/100]
  {\LLMDataPath/appendix/N20_pac.dat};
\addlegendentry{PAC}

\draw[black,line width=1.6pt,->,>=stealth]
  (axis cs:0.659500,0.238250)
  -- (axis cs:0.724250,0.189250);
\draw[green!60!black,line width=1.6pt,->,>=stealth]
  (axis cs:0.293125,0.5444375)
  -- (axis cs:0.231875,0.6030625);
\draw[blue,line width=1.6pt,->,>=stealth]
  (axis cs:0.1756875,0.716875)
  -- (axis cs:0.1240625,0.778125);

\nextgroupplot[
  title={(c)},
  xmin=-0.03,xmax=1.05,
  xtick={0,0.5,1},
  xlabel={$p_H$},
  ylabel={$p_G$},
  legend style={
  at={(0.97,0.97)},
  anchor=north east
}
]

\addplot[black,solid,line width=1.6pt,
  mark=o,mark size=1.7pt,mark options={solid,fill=white}]
  table[x expr=\thisrow{p_H}/100,y expr=\thisrow{J_C}/100]
  {\LLMDataPath/appendix/N40_grpo.dat};
\addlegendentry{GRPO}

\addplot[green!60!black,dashed,line width=1.6pt,
  mark=o,mark size=1.7pt,mark options={solid,fill=white}]
  table[x expr=\thisrow{p_H}/100,y expr=\thisrow{J_C}/100]
  {\LLMDataPath/appendix/N40_raw.dat};
\addlegendentry{Raw}

\addplot[blue,solid,line width=1.6pt,
  mark=o,mark size=1.7pt,mark options={solid,fill=white}]
  table[x expr=\thisrow{p_H}/100,y expr=\thisrow{J_C}/100]
  {\LLMDataPath/appendix/N40_pac.dat};
\addlegendentry{PAC}

\draw[black,line width=1.6pt,->,>=stealth]
  (axis cs:0.666750,0.160000)
  -- (axis cs:0.724500,0.133750);
\draw[green!60!black,line width=1.6pt,->,>=stealth]
  (axis cs:0.3140625,0.3840625)
  -- (axis cs:0.2484375,0.4409375);
\draw[blue,line width=1.6pt,->,>=stealth]
  (axis cs:0.1625625,0.676000)
  -- (axis cs:0.1021875,0.760000);

\nextgroupplot[
  title={(d)},
  xmin=0,xmax=20,
  xtick={0,10,20},
  xlabel={Training round},
  ylabel={$p_G$}
]

\addplot[green!60!black,solid,line width=1.6pt,
  mark=o,mark size=1.7pt,mark options={solid,fill=white}]
  table[x=round,y expr=\thisrow{J_C}/100]
  {\LLMDataPath/appendix/N40_raw.dat};
\addlegendentry{$\rho=1$}

\addplot[orange!85!black,dashed,line width=2pt,mark=none]
  table[x=round,y expr=\thisrow{J_C}/100]
  {\LLMDataPath/appendix/N40_raw_50.dat};
\addlegendentry{$\rho=0.5$}

\addplot[violet,dotted,line width=2pt,mark=none]
  table[x=round,y expr=\thisrow{J_C}/100]
  {\LLMDataPath/appendix/N40_raw_25.dat};
\addlegendentry{$\rho=0.25$}

\nextgroupplot[
  title={(e)},
  xmin=0,xmax=20,
  xtick={0,10,20},
  xlabel={Training round},
  ylabel={$p_G$}
]

\addplot[blue,solid,line width=1.6pt,
  mark=o,mark size=1.7pt,mark options={solid,fill=white}]
  table[x=round,y expr=\thisrow{J_C}/100]
  {\LLMDataPath/appendix/N40_pac.dat};
\addlegendentry{$\rho=1$}

\addplot[llmGreen,dashed,line width=2pt,mark=none]
  table[x=round,y expr=\thisrow{J_C}/100]
  {\LLMDataPath/appendix/N40_pac_50.dat};
\addlegendentry{$\rho=0.5$}

\addplot[llmVermilion,dotted,line width=2pt,mark=none]
  table[x=round,y expr=\thisrow{J_C}/100]
  {\LLMDataPath/appendix/N40_pac_25.dat};
\addlegendentry{$\rho=0.25$}

\nextgroupplot[
  title={(f)},
  xmin=40,xmax=320,
  xtick={75,150,300},
  ymin=0.88,ymax=1.01,
  ytick={0.9,0.95,1},
  xlabel={Mean audits},
  ylabel={Test $p_G$},
ylabel style={
  font=\small,
  at={(axis description cs:-0.25,0.5)},
  anchor=south
},
]

\addplot[green!60!black,dashed,line width=1.6pt,
  mark=o,mark size=1.7pt,mark options={solid,fill=white}]
  table[x=audits,y expr=\thisrow{J_C}/100]
  {\LLMDataPath/test/raw_coverage.dat};
\addlegendentry{Raw audit}

\addplot[blue,solid,line width=1.6pt,
  mark=o,mark size=1.7pt,mark options={solid,fill=white}]
  table[x=audits,y expr=\thisrow{J_C}/100]
  {\LLMDataPath/test/pac_coverage.dat};
\addlegendentry{PAC}

\end{groupplot}
\end{tikzpicture}
\endgroup

\caption{
Initialization and audit ablations in the language model.
(a,b) With SFT demonstration proportions $G/H/N=0.4/0.4/0.2$,
GRPO increases verifier reward while reducing correctness.
Raw audit and PAC instead increase correctness and reduce hacks.
(c) With proportions $0.3/0.3/0.4$, both corrections again favor
correct responses, with PAC reaching higher correctness earlier.
(d,e) Correctness over training rounds for raw audit and PAC at
three audit probabilities $\rho$, using the initialization in (c).
PAC has higher mean correctness at round 5 at each probability.
(f) Final test correctness against audit count. Auditing fewer
responses retains high correctness with fewer labels, but PAC has
no final advantage over raw audit in these runs.
Curves show means over five seeds. Panels (a)--(e) use calibration
evaluations at rounds 0, 5, 10, and 20; arrows in (a)--(c) indicate
training direction. Panel (f) uses separate test prompts, with
points at $\rho \in \{0.25,0.5,1\}$ from left to right for each method.
}
\label{fig:llm-appendix}
\end{figure*}

\begin{table*}[t]
\centering
\caption{Final test outcomes after 20 rounds. Values are percentages, reported as mean $\pm$ one sample standard deviation across five training seeds. Each SFT row is one shared initialization. SFT evaluations are shared across runs, so no standard deviation
across training seeds is reported for these rows. N20 and N40 use
$G/H/N$ demonstration probabilities $40/40/20\%$ and $30/30/40\%$, respectively. Coverage ablations use N40.}
\label{tab:llm_appendix_outcomes}
\begin{tabular}{llcccc}
\toprule
Scenario & Method & Audit probability & $J_C=p_G$ & $J_R=p$ & $p_H$ \\
\midrule
N20 & SFT & --- & $31.2$ & $80.5$ & $49.2$ \\
N20 & GRPO & --- & $0.5 \pm 0.7$ & $98.8 \pm 0.4$ & $98.3 \pm 0.3$ \\
N20 & Raw audit & $1$ & $98.9 \pm 0.4$ & $99.2 \pm 0.0$ & $0.3 \pm 0.4$ \\
N20 & PAC & $1$ & $98.1 \pm 1.3$ & $99.4 \pm 0.3$ & $1.2 \pm 1.6$ \\
\midrule
N40 & SFT & --- & $30.5$ & $69.5$ & $39.1$ \\
N40 & GRPO & --- & $2.2 \pm 2.4$ & $99.4 \pm 0.7$ & $97.2 \pm 2.4$ \\
N40 & Raw audit & $1$ & $98.4 \pm 2.6$ & $99.7 \pm 0.7$ & $1.2 \pm 2.0$ \\
N40 & PAC & $1$ & $97.2 \pm 2.0$ & $98.9 \pm 1.7$ & $1.7 \pm 1.3$ \\
\midrule
N40 & Raw audit & $0.5$ & $97.2 \pm 2.0$ & $99.1 \pm 1.0$ & $1.9 \pm 2.3$ \\
N40 & PAC & $0.5$ & $93.9 \pm 4.6$ & $99.1 \pm 0.7$ & $5.2 \pm 4.4$ \\
N40 & Raw audit & $0.25$ & $96.1 \pm 1.1$ & $99.5 \pm 1.0$ & $3.4 \pm 1.8$ \\
N40 & PAC & $0.25$ & $95.2 \pm 2.1$ & $99.5 \pm 0.7$ & $4.4 \pm 2.6$ \\
\bottomrule
\end{tabular}
\end{table*}

\paragraph{Hyperparameters.} Table~\ref{tab:llm_appendix_hyperparameters}
collects the settings shared across training runs. Generation and
evaluation use the same temperature and token limit.

\begin{table}[t]
\centering
\caption{Settings for the language model experiments.}
\label{tab:llm_appendix_hyperparameters}
\begin{tabular}{ll}
\toprule
Setting & Value \\
\midrule
Model & Qwen2-0.5B \\
Input / output digits & 6 / 12 \\
Final SFT updates per scenario & 80 \\
Reinforcement learning seeds & 0--4 \\
Training rounds & 20 \\
Prompts per round / responses per prompt & 2 / 8 \\
Total training responses per run & 320 \\
Optimizer / learning rate & AdamW / $10^{-5}$ \\
Gradient norm cap / weight decay & 1 / 0 \\
KL coefficient & 0 \\
Advantage denominator offset & $10^{-4}$ \\
Temperature / maximum response tokens & 1 / 192 \\
Sampling & Full token distribution \\
Calibration rounds & 0, 5, 10, 20 \\
Calibration prompts / responses each & 8 / 4 \\
Test prompts / responses each & 32 / 4 \\
\bottomrule
\end{tabular}
\end{table}
\newpage

\end{document}